\documentclass{article}

\usepackage[utf8]{inputenc}
\usepackage[T1]{fontenc}

\PassOptionsToPackage{numbers,sort&compress}{natbib}
\usepackage[numbers,sort&compress]{natbib}

\usepackage{microtype}
\usepackage{graphicx}
\usepackage{subcaption}
\usepackage{booktabs}
\usepackage[
    colorlinks=true,
    linkcolor=pierCite,
    citecolor=pierCite,
    urlcolor=pierLink
]{hyperref}
\usepackage{url}
\usepackage[dvipsnames]{xcolor}
\usepackage{geometry}
\usepackage{titletoc}
\usepackage{array}
\usepackage{colortbl}
\usepackage{pifont}
\usepackage{textcomp}
\usepackage{authblk}
\usepackage{amsmath}
\usepackage{amssymb}
\usepackage{mathtools}
\usepackage{amsthm}
\usepackage{macros}
\usepackage{algorithm}
\usepackage{algpseudocode} 
\usepackage{tikz}
\usetikzlibrary{shapes, arrows.meta, positioning}
\usepackage{titletoc}
\newcommand\DoToC{%
  \startcontents
\hypersetup{colorlinks=true}
  \printcontents{}{1}{\subsection*{\textbf{Table of contents}}}
  \vskip3pt\vskip5pt
}
\usepackage[skip=0.5\baselineskip, indent=0pt]{parskip}

\definecolor{PastelRed}{HTML}{FDE2E1}
\definecolor{PastelGreen}{HTML}{E6F4EA}
\definecolor{PastelBlue}{HTML}{E9F0FB}
\definecolor{PastelYellow}{HTML}{faebc8}

\usepackage[capitalize,noabbrev]{cleveref}
\crefname{assumption}{assumption}{assumptions}
\Crefname{assumption}{Assumption}{Assumptions}
\crefname{appendix}{appendix}{appendices}
\Crefname{appendix}{Appendix}{Appendices}

\theoremstyle{plain}
\newtheorem{theorem}{Theorem}[section]
\newtheorem{proposition}{Proposition}[section]
\newtheorem{lemma}{Lemma}[section]
\newtheorem{corollary}{Corollary}[section]

\theoremstyle{definition}
\newtheorem{definition}{Definition}[section]
\newtheorem{assumption}{Assumption}[section]

\theoremstyle{remark}

\usepackage[textsize=tiny]{todonotes}

\title{How Useful is Causal Invariance for Domain Adaptation\\
in Finite-Sample Settings?}

\author[1]{Julia Kostin$^*$}
\author[2]{Kasra Jalaldoust$^*$}
\author[2]{Elias Bareinboim}
\author[3]{Samory Kpotufe}
\author[1]{Fanny Yang}

\affil[1]{Department of Computer Science, ETH Zurich}
\affil[2]{Causal Artificial Intelligence Lab, Columbia University}
\affil[3]{Department of Statistics, Columbia University}

\date{}

\begin{document}

\maketitle

\def\thefootnote{*}\footnotetext{Equal contribution.}

\begin{abstract}
Machine learning models often degrade when they are deployed on a target distribution that differs from the source distributions they were trained on.
Recent work in causality-based domain generalization has shown how shared causal structure between domains can induce invariant predictors, e.g., models on a subset of features which have stable risk across structured domain shifts. 
However, the extent to which such population-level causal invariances can lead to gains in finite-sample settings remains underexplored. In particular, in practice  we often have access to a few labeled target samples, a setting called supervised domain adaptation (sDA). In this paper, we explore when (full or partial) causal knowledge can provably improve supervised domain adaptation.

As a first step, we study linear regression, where full or partial causal knowledge specifies a collection of invariant or possibly invariant feature subsets, each yielding a source-trained candidate predictor. 
We derive matching upper and lower bounds showing that finite-sample gains are governed by the target-risk margins separating the candidates, together with the finite-source estimation error. 
When these margins are sufficiently large relative to $n_Q$, an adaptive aggregation procedure can match the best candidate predictor while avoiding negative transfer relative to target-only learning. On the other hand, when the margins are too small, no algorithm can reliably exploit the candidate collection to obtain faster finite-sample rates. We further connect these margins to structural shift magnitude in linear SCMs and validate the theory on real-world causal benchmarks.

\end{abstract}

\renewcommand{\thefootnote}{\arabic{footnote}}
\section{Introduction}

Causal invariance has garnered much recent attention in domain generalization as a way to model invariant information that might be transferable across prediction tasks involving different data distributions \citep{peters2016causal,rojas2018invariant,arjovsky2019invariant,bareinboim2014transportability}. 
For example, one may expect that the causal relationships between feature variables $X \doteq [X_{(1)}, \ldots, X_{(d)}]$ and output variable $Y$ hold across domains $P$ and $Q$, e.g. across geographic regions or patient populations. 
Then, it is possible that for certain subsets $I \subset  [d]$ of features of $X$ (especially, but not exclusively, the causes of $Y$), the regression functions of the form $h_I \doteq \E [Y \mid X_I]$ remain unchanged, i.e., $h_{I, P} = h_{I, Q} = h_I$.



As outlined in \Cref{sec:related-work}, existing works in causal domain generalization primarily focus on 1) identifying causal relationships and the 
resulting invariant models and 
2) characterizing the \emph{structural shift}
from $P$ to $Q$ under which such invariant models are guaranteed to have bounded, though not necessarily low, risk under $Q$.
However, despite much promising work on the subject, the extent to which such causal knowledge might be leveraged in certain finite-sample settings of interest (involving $P$ and $Q$ data) remains unclear, as we will soon describe. 



 In case an invariant feature set $I$ exists, with qualitative knowledge of causal relationships and structural shift, even \emph{with no data from $Q$}, we may perform well under $Q$ by estimating invariant $h_I$ from $P$ data alone \citep{peters2016causal,buhlmann2020invariance,Jalaldoust_Bareinboim_2024}. 
However, we typically do not have such causal knowledge in practice, 
only allowing us to identify a set of candidate invariant models.
Furthermore, not all of these models actually have low $Q$-risk and some could be too conservative and perform poorly under $Q$ compared to the best predictor $h_{Q}$ that uses all features.

Fortunately, in many settings, we have access not only to data from $P$, but also to a limited amount of target data from $Q$.
These samples might be useful for identifying which candidates to leverage among the invariant $h_I, I \in {\cal I}$ that are trainable on source data from $P$. 
In this work, we are interested in how useful (partial) causal knowledge is in such settings with finite samples from both $P$ and $Q$. For instance, 
suppose we have a collection $\cal I$ of candidate invariant feature subsets that correspond to the available causal knowledge. Further assume that one of the candidate sets in $\cal I$
has a very low population $Q$-risk $R_Q(h_I)$.
Is it always possible to \emph{leverage} $\cal I$ to achieve a better $Q$-risk than the baseline of using the $P$ and $Q$ data alone, ignoring the partial causal knowledge that yields $\cal I$?
The answer turns out to be mixed with both positive and negative results, and tightly depends on the interplay between certain \emph{risk margins} separating candidate invariants $I \in \cal I$ and the amount of samples from $P$ and $Q$. 

\begin{figure}
    \centering
    \includegraphics[width=0.7\linewidth]{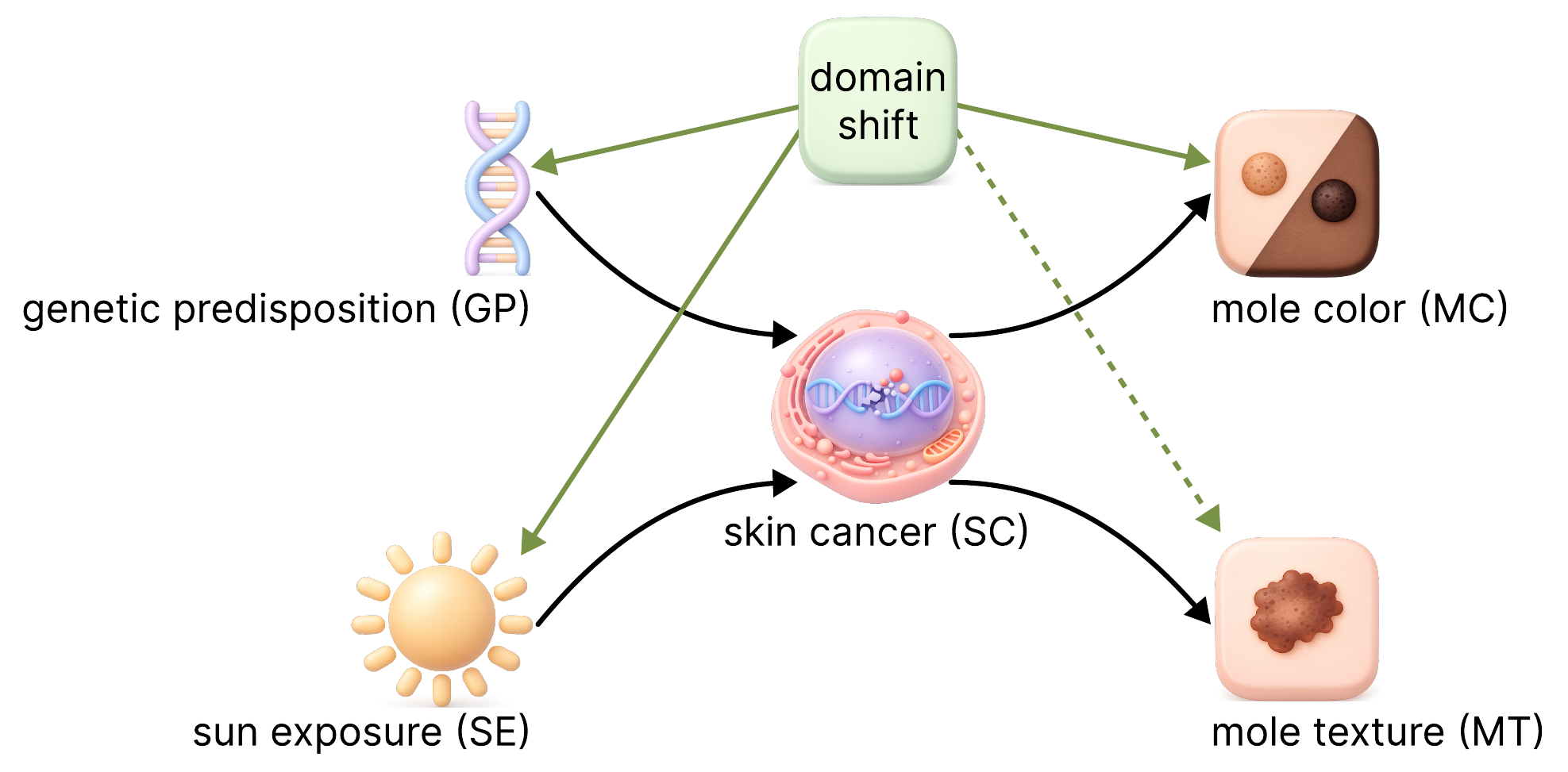}
\caption{A toy causal supervised adaptation problem. The graph depicts simplified relationships between the most common causes and symptoms of skin cancer. The light green shift node indicates which variables may be shifted between two patient subpopulations (domains) $P$ and $Q$, where the dashed edge indicates a symptom, which, unbeknownst to us, remains unshifted between $P$ and $Q$. Two feature sets result in invariant models: The set of \emph{causes}, or parents, $I_1 = \{GP, SE\}$, which results in a model ignoring symptoms of the disease, and the bigger set $I_2 = \{GP, SE, MT\}$, which includes an unshifted symptom. }
\label{fig:dag-skin-cancer}
\end{figure}

To provide quantitative answers to these questions, we consider a linear regression setting in $\R^d$ with $n_P$ samples from a source distribution $P$, and $n_Q$ samples from a target $Q$. The goal is to return a regression estimate $\hat h$ with low $L_2$ excess risk ${\cal E}_Q(\hat h) \doteq R_Q(\hat h) - R_Q(h^*_Q)$ w.r.t. 
an unknown best linear function $h^*_Q$ under $Q$. 
While our analysis makes no assumption on the relation between 
$n_P, n_Q$, we are mostly interested in regimes with $n_P \gg n_Q$, and where prior knowledge of causal invariances might be leveraged to harness the large data from $P$ to do well on $Q$. 
\emph{Thus, we aim to characterize situations when prior causal knowledge, even if imperfect, can be used to achieve an excess risk rate strictly faster than any learner without access to such causal knowledge. }

As alluded to so far, prior causal knowledge and remaining uncertainties can for example be captured using a collection $\cal I$ that consists of all 
candidate subsets $I$ that may remain invariant under possible shifts 
(cf.  \Cref{sec:collections} for a more formal discussion). 
For now, consider the concrete example in
Figure \ref{fig:dag-skin-cancer} for intuition: suppose we would like to predict $Y$ (skin cancer) using 4 groups of features $X = \{\text{GP, SE, MC, MT}\}$, and assume the depicted causal graph is known a priori along with the fact that $Y$ is not intervened on in the shift $P\to Q$. 
Depending on the amount of knowledge about the structural shift between $P and Q$, this collection may contain different sets. 
For instance, suppose we know that $P$ and $Q$ are populations with different skin tones, modeled as a shift arrow on MC in the graph. At the same time assume that we do not know whether MT is shifted, as indicated by the dashed edge on MT in \Cref{fig:dag-skin-cancer}.
In that case, $I_1 \doteq \{\text{GP, SE}\}$ and $I_2 \doteq \{\text{GP, SE, MT}\}$ are the possible invariant sets 
and the collection ${\cal I} = \{I_1, I_2\}$ reflects the uncertainty about the shift in MT. If the causal mechanism of MT shifts from source to target, then $I_1$ would be the only valid invariant subset inside the collection $\cal I$. However, if MT does not shift, $I_2$ is invariant as well and could in fact yield better $Q$ risk since it uses an informative, unshifted \emph{symptom}. 
This example illustrates how $\cal I$ can encode uncertainty in invariants owing to partial knowledge of causal mechanisms.  

We can now describe our main results in the context of a simplifying assumption: 
Suppose a learner has access to prior (partial) causal knowledge encoded by the collection ${\cal I}$ that includes an invariant subset $\goodset$ with $h_{\goodset} = h_Q$. 
Can \emph{any learner} $\hat h$ given $\cal I$ recognize that $\goodset$ is a good (the best!) subset and, that using $P$ data to learn $h_{\goodset}$ would lead to low risk under $Q$? We show that the answer depends not only on the amount of data available in both $P$ and $Q$, but crucially on \emph{margins} 
$\margin_I \doteq R_Q(h_I) - R_Q(h_{\goodset})$ -- that is, the difference of the $Q$-risks between candidates $I\in \cal I$ and the a priori best choice, such as $\goodset$ in this example. In particular, we show that larger risk margins make it easier to leverage $\cal I$.
Our main technical effort is in deriving sufficient and necessary conditions on the magnitude of such $\margin_I$'s w.r.t. sample sizes $n_P, n_Q$ to guarantee benefits from prior causal knowledge. Our results imply the following:

$\bullet$ As a positive result, we show that margins satisfying 
$\inf_{I \in {\cal I}} \margin_I \gtrsim \log |{\cal I}|/ n_Q$ 
are sufficient to successfully leverage causal knowledge at a given $\ntarget$, provided {(i)} $n_P$ itself is sufficiently large to properly estimate regression functions, and {(ii)} the shift $P_X$ to $Q_X$ is sufficiently controlled, so that errors under $P$ may be related to errors under $Q$. 
Under such conditions, knowledge of causal invariances 
enables \emph{few-shot learning} in the following sense: in favorable settings, instead of $\Omega(d/\epsilon)$, we require only $\Omega(d)$ samples to achieve $\epsilon$ excess risk.
We emphasize that in order to 
identify the best predictive model $h_{\goodset}$ 
with a sufficient margin order of $n_Q^{-1}$, we 
require careful model selection. In particular, as the technical reader may note, using uniform convergence arguments, naïve cross-validation may in general require a margin of order $n_Q^{-1/2}$  to succeed. 
Instead, the small margin requirement of order $n_Q^{-1}$ is achieved via iterative refinements of carefully constructed confidence sets centered at optimal aggregation estimators.

$\bullet$ 
Further, we show that this upper bound is tight. In particular, 
we show that even in the two-model case and with infinite source data, if the risk margin satisfies 
$\margin \lesssim \log |{\cal I}|/ n_Q$, no algorithm can reliably achieve the risk of the better candidate with an error of order less than $\margin$. 
Thus, the margin-driven sample complexity $n_Q \gtrsim \Omega(1/\margin)$ in our upper bound is necessary.
Hence, we cannot always leverage $\cal I$---even when it contains arbitrarily good invariants---to achieve a better rate than that of a learner which ignores $\cal I$ altogether. 

With these main results at hand, we return to our original question: can causal knowledge improve supervised domain adaptation when a small number of labeled samples from the target distribution ($Q$) is available? In linear structural causal models,  we first show that large structural shifts induce large target-risk margins between candidate models. Our subsequent guarantees echo a familiar message from the causality literature: causal knowledge is particularly useful under large shifts. For the first time however, we establish usefulness of even incomplete causal knowledge under such shifts, and show how labeled samples from a \emph{fixed} target distribution  can be used more efficiently when some causal knowledge is available.  Thus, our setting differs from the classical domain generalization perspective, where causal invariance is used to guarantee robust performance over a class of possible target distributions; here, labeled samples from a fixed target ($Q$) change the benchmark to \emph{target-specific} performance relative to alternatives such as target-only learning. We now discuss the relation to prior work in more detail.

\subsection{Related work}\label{sec:related-work}

\textbf{Causality-based Domain Generalization.} A long line of work in domain generalization (DG) leverages shared structure between source and target distributions to derive predictors that remain valid under unseen shifts. Early instances assume global covariate or label shift \citep{shimodaira2000improving,sugiyama2007covariate,storkey2008training,lipton2018detecting,garg2020unified}; more structural approaches posit that the data-generating process is governed by shared causal mechanisms \citep{glennan1996mechanisms,machamer2000thinking}, and exploit the resulting invariances either through distributional  ($P(Y \mid X_{\Sstar})$) or regression function ($\EP[Y \mid X_{\Sstar}]$) invariance, specified by an ``invariant'' feature subset $X_{\Sstar}$. The invariance $h_{\Sstar,Q} = h_{\Sstar,P}$ is a special case of \emph{causal transportability}, where an object of interest under the target distribution $Q$, e.g., $h_{\Sstar,Q}$, can be uniquely computed from the source distribution $P$. When qualitative knowledge of causal relationships and structural shifts is available, even \emph{with no data from $Q$}, one may transport $h_{\Sstar}$ from $P$ to $Q$ \cite{pearlbareinboim2011,correa2019statistical,correa2020general,Jalaldoust_Bareinboim_2024}. Invariant predictions are of particular interest \citep{peters2016causal,christiansen2021causal,heinze2018invariant,huang2020,buhlmann2020invariance,rojas2018invariant,arjovsky2019invariant,wang2022generalizing,wald2021on,muandet2013domain,li2018domain}, since they allow the model  $h_{\Sstar,Q}$ to be directly estimated from $P$ data. The strength of an identified invariant predictor is that it provides robust generalization guarantees against worst-case shifts. In practice, however, these methods are frequently outperformed by plain source-based empirical risk minimization (ERM) on realistic benchmarks \citep{gulrajani2020search,nastl2024do,vedantam2021an}, for two reasons: (i) the invariant feature set is typically \emph{not identifiable} from the available source distributions, and (ii) even when it is, restricting prediction to invariant features can yield \emph{weakly predictive} models on a specific target. Recent relaxations partially address these issues: \emph{partial transportability} \citep{jalaldoust2024partial} and \emph{worst-case robustness} \citep{kostin2024achievable} extend worst-case guarantees to partially identifiable settings, while \emph{bounded robustness} \citep{rothenhausler2021anchor,shen2023causalityoriented,kook2022distributionalanchor,jakobsen2022distributional} and the \emph{stable blanket} of \cite{pfister2021stabilizing} retain non-invariant features when shifts are assumed bounded or partially structured. Notably, relaxations outlined above still aim to be robust against a worst-case target distribution; albeit from a smaller family. In contrast, in the field of domain adaptation, which we describe below, the target distribution is viewed as fixed. If some target data is available, one has to compete with both source and target ERM, and robustness to worst-case target distributions may hurt generalization in presence of target data.

\textbf{Domain adaptation.} This  related paradigm investigates to which extent access to some data from the target domain can improve transfer from the source distribution. The field is conventionally split into unsupervised (uDA) \citep{redko2020survey,ben-david-2006,ben2010theory,cortes2010,blitzer2006,mansour2008} and supervised (sDA) \citep{blitzer2011da,chattopadhyay13,hanneke2013,saha2011,valueoftargetdata,hanneke2025adaptivesampleaggregationtransfer} regimes, with the latter additionally assuming a (typically small) labeled target sample, also referred to as \emph{transfer learning}. Classical theory characterizes transferability through a distributional discrepancy between source and target \citep{bendavid2003,ben-david-2006,blitzer2006,mansour2008,ben2010theory}, and shows that without further structure, adaptation is provably hard whenever this discrepancy is unrestricted \citep{david2010impossibility,david2012hardness,ben2014domain}. Such distance-based guarantees can be vacuous in regimes where the source-target shift is large but structurally constrained---precisely the regime in which causal invariances ought to help. Several methods inject structural assumptions into DA \citep{gong2016domain,heinze2021conditional,chen2021domain,wu2023prominent}, but they predominantly target the \emph{unsupervised} setting, require strict identifiability, and have been observed to suffer from \emph{label flipping} when these assumptions fail \citep{zhao2019learning,johansson2019support,wu2023prominent}. Empirically, access to a handful of labeled target samples has been observed to substantially improve causality-based DA over pooled-source models \citep{chen2021domain,londschien2025domain}, yet---to the best of our knowledge---no prior work provides finite-sample sDA bounds that quantify when, and by how much, partial causal knowledge accelerates adaptation with target samples. 

\textbf{Model selection aggregation.} Our procedure builds on the \emph{model selection aggregation} literature \citep{tsybakov2003optimal,rigollet2012kullback}, which takes a fixed  dictionary of candidate predictors and an i.i.d. sample from a single distribution, and outputs an aggregate whose risk is close to that of the best dictionary element. Crucially, \emph{improper} aggregation procedures, such as exponential weights aggregation \citep{yang1999model,audibert2007progressive,juditsky2008learning}, Q-aggregation \citep{lecue2014optimal}, and the STAR estimator \citep{audibert2007progressive,audibert2007proof}, attain the fast rate $\log |\setI|/n$, in contrast to the slow rate $\sqrt{\log |\setI|/n}$ of na\"ive empirical risk minimization over the dictionary. The fast $\log |\setI|/n_Q$ rate of aggregation  allows us to convert risk margins between candidate models into the fast target sample complexity bounds of our main results. Standard aggregation results, however, neither address transfer between distributions nor permit the dictionary to depend on the data; we adapt them to a transfer setting in which candidates are first trained on source data and then aggregated on target. We discuss model aggregation more in \Cref{apx:model-aggregation}.

\section{Setting}

In this section we briefly introduce our supervised domain adaptation setting and notation, and give a short primer on causal graphs and shift mechanisms that the experienced reader is invited to skip. In \Cref{sec:collections} we discuss candidate set collections $\cal I$ that may be constructed using different degrees of causal knowledge. 

\subsection{Supervised domain adaptation setting}
\label{sec:setup-and-notation}

Given a  feature space $\mathcal{X} \subseteq \R^d$ and a label space $\mathcal{Y} \subseteq \R$, we denote by $D$ any joint probability distribution over $\mathcal{X} \times \mathcal{Y}$ (also referred to as a \emph{domain}). We aim to find a hypothesis with a low average loss on a certain data domain. A loss   $\ell: \hypo \times \domainX \times \domainY \to \Rpos$ takes a hypothesis $h: \domainX \to \domainY$ and a (feature, label)-pair $(x,y)$, and returns a non-negative penalty value. 
Throughout this paper, we focus on regression with
the \emph{squared loss} $\ell(h,(x,y)) = (h(x)-y)^2$. For any joint data distribution $\prob$ and hypothesis $h$, we define the \emph{risk} of $h$ as $$\riskD (h) = \E_{(x,y)\sim \prob}[\ell(h,(x,y))] = \ED[\ell(h,(x,y))].$$  Of special interest is the excess risk $$\excessD(h; \hypo) := \riskD(h) - \inf_{h' \in \hypo} \riskD(h'),$$ which expresses how much worse the hypothesis $h$ is compared to the best choice in $\hypo$. We also write $\excessD(h)$ when $\hypo$ is clear from the context. Given a dataset $\mathcal{D}_D = \{(x_i, y_i)\}_{i=1}^n$ of $n$ i.i.d. samples from $D$, we define the empirical risk $\riskempD (h) \coloneqq \frac{1}{n} \sum_{i=1}^n \ell(h,(x_i,y_i))$ and the empirical excess risk $\excessempD(h; \hypo) \coloneqq \riskempD(h) - \min_{h' \in \hypo} \riskempD(h')$. 


The learning setting we consider in this paper is supervised domain adaptation, sometimes also referred to as transfer learning:  we have access to a finite dataset $\datasource = \{(x_i, y_i)\}_{i=1}^{\nsource}$ containing $\nsource$ i.i.d. samples from the \emph{source} distribution $\probsource$ and a (typically small) labeled dataset $\datatarget = \{(x_i, y_i)\}_{i=1}^{\ntarget}$ from the \emph{target} distribution $\probtarget$.  
Whenever multiple source distributions $\{\probsource_e\}_{e \in \mathcal{E}}$ are available, we denote by $\probsource$ their mixture $\sum_{e} w_e \probsource_e$, where $w_e > 0$ is the probability of sampling from source $\probsource_e$. 
The goal is to construct a procedure which depends on $\datasource$ and $\datatarget$, and outputs a hypothesis $\hemp \in \hypo$ with a \emph{low target excess risk} $\excessQ(\hemp)$.

Our analysis focuses on 
the class of (bounded) linear functions $h: \domainX \to \domainY$ as $\hypolinB \coloneqq \{ x \mapsto \langle w, x \rangle \mid w \in \R^d, \| w \|_2 \leq B  \}$ 
where $B \in [0, \infty)$. We write $\mathcal{H} = \hypolinB$ when clear from the context. For an index set $I \subseteq [d] = \{1,...,d\}$, denote by $\projI: \mathcal{X} \to \mathcal{X}$ the projection onto the features $X_I$. We define the restricted class  $\hypoI = \hypo \circ \projI \subseteq \hypo $, that contains all linear functions only depending on features $X_I$. We denote any \emph{feature-subset} population and empirical risk minimizer as
\begin{align}
\label{eq:subsetminimizers}
    \hID \in \arg\min_{h \in \hypoI} \riskD(h) \quad \text{and} \quad \hempID \in \arg\min_{h \in \hypoI} \riskempD(h).
\end{align}
\begin{definition} \label{def:invariant-subset} 
For two distributions $D,D'$, we call a subset of indices $I \subseteq [d]$ and the corresponding feature set $X_I$ (as well as its corresponding model $\hIP$) $(D,D')$-invariant,
if the following holds:
\begin{equation*}
    D(Y\mid X_{I}) = D'(Y\mid X_{I}).
\end{equation*}
Further we have $\E_D[Y \mid X_I] =  \E_{D'}[Y \mid X_I]$ and under the realizability assumption that $\E_D[Y \mid X_I]\in \hypoI$, we also have $\hID = \hIDprime$. 
\end{definition}

If one abstracts away how the candidate models are obtained from source data,
supervised domain adaptation with a finite candidate set is closely related to
the model selection aggregation literature. In its classical form, model
selection aggregation considers a finite dictionary
\(\mathcal H = \{h_m : m \in [M]\}\) of predictors
\(h_m : \mathcal X \to \mathbb R\), together with an i.i.d. labeled sample
\(\mathcal D = \{(X_i,Y_i)\}_{i=1}^n\) from a distribution \(D\).
The goal is to output an aggregate predictor in the convex hull
\(h_{\mathrm{agg}} \in \operatorname{Conv}(\mathcal H)\) whose risk is close to that of the best dictionary element. Our algorithm uses such
aggregation procedures as subroutines, and our generalization bounds rely on
well-known fast-rate guarantees from that literature. In particular, we use
\emph{deviation-optimal} aggregation procedures, which satisfy, with probability
at least \(1-\alpha\),
\begin{equation*}
    \riskQ(\hagg) \leq \min_{m \in [M]} \riskQ(h_m) + C \frac{\log (M/\alpha)}{n}.
\end{equation*}
We review several deviation-optimal aggregation procedures and their guarantees
in \Cref{apx:model-aggregation}.

\subsection{Causality Primer}\label{sec:new-causal}
This paper considers the setting where data might come from different domains. Whether data from some domains can be leveraged for prediction in another depends crucially on how the shifts between domains are modeled.
Causality-based methods assume that domains share certain causal mechanisms, often formalized via a structural causal model (SCM). An SCM $\SCM$ over the observable variables $V_1,V_2,...,V_p$ is defined using a joint distribution over unobserved variables $\mathbf{U}$ and a set of \emph{structural equations} or \emph{causal mechanisms},
\begin{equation*}
    V_i \gets f^{D}_{i}(\mathbf{V}_{\Pa({V_i})},\mathbf{U}_{i}) \quad i = 1,...,p,
\end{equation*}
 where $\Pa({V_i}) \subseteq [d+1]$ are called the \emph{parents} or direct causes of $V_i$, and $\mathbf{U}_{i} \subseteq \mathbf{U}$ is a subset of unobserved variables that affect $V_i$.
For an SCM, its \emph{causal graph} $\graph$ is a graph with vertices $\mathbf{V}$, a directed edge from $V_i$ to $V_j$ whenever $i \in \Pa(V_j)$, and a bi-directed edge between $V_i$ and $V_{j}$ whenever they are \emph{confounded}, i.e., the noise vectors $\mathbf{U}_i$ and $ \mathbf{U}_j$ are dependent. 
We assume that the causal graph is acyclic \cite{pearl2009causality}. Further, we say that a distribution $D$ over $V \coloneqq (X_1,...,X_d,Y)$
has a causal structure if there exists an SCM $\SCM_{D}$ over $V$
characterized by causal mechanisms $\{f^D_1,...,f^D_d,f^D_Y\}$ and a joint distribution $D_{\mathbf{U}}$ over the noise variables, such that $D$ is induced by the SCM $\SCM_{D}$.

When we are dealing with different domains with distributions $D$ and $D'$ that we believe to follow the structural causal models $\SCM_D$ and $\SCM_{D'}$ respectively, one way to model their difference is to assume \emph{structural shifts}, or \emph{causal mechanism shifts}: that is, for some variables $V_i \in \{X_1,...,X_d,Y\}$, their assignment function and/or their noise distribution differ across the domains, i.e. $f^D_i \neq f^{D'}_i$ and/or the marginal distributions over the noise vectors $D_{\mathbf{U}_i} \neq D'_{\mathbf{U}_i}$. 
The \emph{domain discrepancy}, or \emph{shift set} $\shiftset{D}{D'} \subseteq [d+1]$ 
consists of all variables that experience structural shift 
across the domains. 
One prominent example of a shift set $\shiftset{D}{D'}$ in the literature are the set of variables that are \emph{intervened} on between $D,D'$ \cite{pearl2009causality}, while the structural equations of all other variables are shared across $D, D'$. 
When both domains 
only differ in terms of such structural shifts,
one can represent both domains and the structural shift between them via a so-called \emph{selection diagram}: it is constructed by augmenting the causal graph with a square node $\squarenode{D}{D'}$ that points to all shifted variables $V_i$ with $i \in \shiftset{D}{D'}$. See \Cref{fig:main-causal-figure} for the formal selection diagram that captures the example in \Cref{fig:dag-skin-cancer} described in the introduction.

In the causal invariance literature, finding invariant (sometimes also referred to as stable)  
feature sets $I$ (cf. \Cref{def:invariant-subset}) is often considered the ultimate goal in the presence of unknown structural shifts. The rationale is that the corresponding feature-subset minimizers in \Cref{eq:subsetminimizers} have controlled risk for any admissible target distribution. However, in practice,  we often cannot guarantee full knowledge of the graph and structural shift, for example due to lack of heterogeneity in the source distributions. Instead, we sometimes can at most identify a few edges in the causal graph and only have partial knowledge about shift set. For instance, we could know that the structural equations of certain variables are definitely invariant between source and the target, while for some other variables they have possibly shifted.
As a consequence, we may only be able to identify a collection of many possible invariant sets (only few of which are truly invariant for a particular target). Such a collection then reflects the remaining uncertainty about the causal graph and structural shift. 
\begin{figure*}[t]
    \centering
    \includegraphics[width=\textwidth]{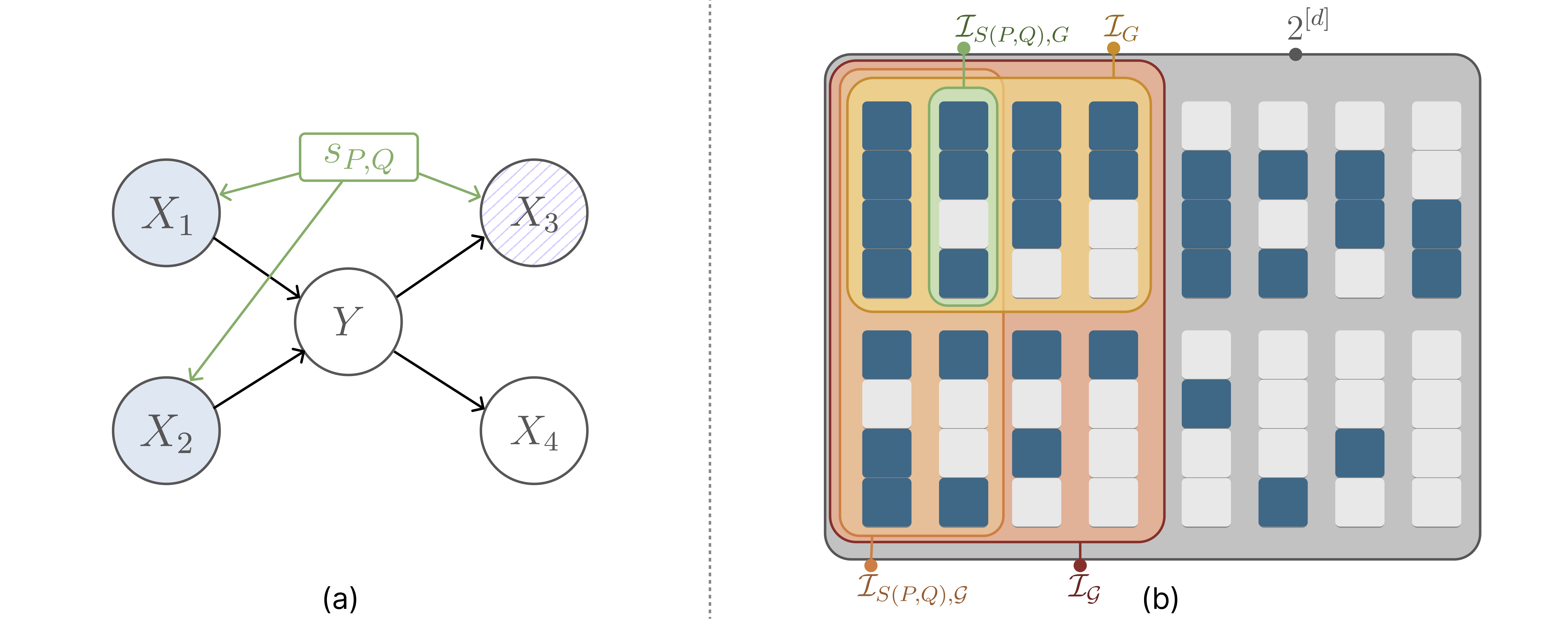}
    \caption{
    (Generally unknown) underlying causal structure of the data shared across source and target domains $P$ and $Q$. The parents of $Y$ are indicated in blue. The node $\squarenode{P}{Q}$ points to variables whose causal mechanisms shift  from $P$ to $Q$. The striped node $X_3$ is the shifted child of $Y$. \textbf{(b)} Corresponding sufficient invariance collections for varying degrees of causal knowledge. Here, $\shiftset{P}{Q} = \{1,2,3\}$. The green singleton collection $\mathcal{I}_{\shiftset{P}{Q},G}$ corresponds to full knowledge of the causal graph $G$ and the shift set  $\shiftset{P}{Q} = \{1,2,3\}$. The yellow collection $\mathcal{I}_G$ corresponds to knowledge of the graph but not the shift. The orange collection $\set_{\shiftset{P}{Q}, \mathcal{G}}$ corresponds to knowledge of the shift set and partial knowledge of the graph, where $\graphs = \{\graph': 1 \in \pa(Y) \text{ in } \graph'\}$ (cf. \Cref{sec:toy-example}) is discoverable using two source domains where only $X_1$ shifts between them. Finally, the red collection $\set_{\graph}$ corresponds to partial knowledge of graph $\mathcal{G}$ and no knowledge of the shift. The grey collection contains all feature subsets.
    } 

    \label{fig:main-causal-figure}
\end{figure*}

\textit{}


\subsection{Collections of feature sets}\label{sec:collections}


One broad objective could be to find a collection of predictors (trained on the source distribution $P$) that contains at least one predictor with bounded $Q$-risk. In this section we show how to construct such a candidate collection that reflects causal knowledge and in  \Cref{sec:main-results}, we show how this collection can be used for supervised domain adaptation. Our main results in \Cref{sec:main-results} however apply to \emph{any} collection of feature-restricted source models, whether or not they follow the specific constructions discussed here and throughout the causal inference literature. 

Consider a collection of feature subsets $\set \subseteq 2^{[d]}$, and the corresponding collection of feature-restricted source predictors $h_{I,P}$. 
Recall that under the realizability assumption, it holds that $h_{I,P} = h_{I,Q}$ for any $(P,Q)$-invariant $I\in \set$.
The $(P,Q)$-invariance of feature-restricted models allows us to train a model on the source and safely use it for prediction in the target. However, no amount of source data alone can determine whether a feature subset $I \subseteq [d]$ is $(P,Q)$-invariant, and therefore, access to 
knowledge of the causal graph and structural shift is necessary.

Given the causal graph $\graph$ and the shift set $\shiftset{P}{Q} \subseteq [d]$, we could directly 
find $(P,Q)$-invariant models: For a feature set $I\subseteq [d]$, if the label $Y$ and the shift node $\squarenode{P}{Q}$ are \emph{d-separated} given $X_I$ (denoted as $Y \dsep \squarenode{P}{Q}  \mid X_I \text{ in } \graph$) then $I$ is $(P,Q)$-invariant\footnote{The converse is not true in general, however, under
the additional assumption of causal faithfulness \cite{spirtes2001causation} 
we can obtain a direct correspondence between d-separation and invariance. 
}.
However, the causal graph $\graph$ is rarely fully known in advance. Instead, given the source $P$, one may still be able to \emph{partially} discover the causal graph and obtain a set $\graphs$ of possible candidate graphs: for example, under the causal faithfulness assumption, we can use the source data $P$ to obtain the Markov Equivalence Class (MEC) of causal diagrams represented by a PAG/CPDAG \cite{spirtes2000causation}. A complementary line of work \cite{wang2026confidence} quantifies the statistical uncertainty in discovery, returning valid confidence sets of causal orderings that can be translated into an uncertainty set of graphs $\graphs$.

We can now formally define a collection of most predictive invariant sets that depends on the degree of knowledge about the causal graph and the shift set.

\begin{definition} \label{def:sufficient-invariance-collection} We define the \emph{sufficient invariance collection} as the collection of \emph{maximal} feature subsets that could be $(P,Q)$-invariant under an admissible $Q$, and for different degrees of causal knowledge:
\begin{itemize}
\item Given full causal knowledge $\graph$ and $\shiftset{P}{Q}$, we have
\begin{equation*}
    \suffcollection = \set_{\shiftset{P}{Q},\graph} \coloneqq \big\{ I \subseteq [d] \text{ maximal such that } Y \dsep \squarenode{P}{Q}  \mid X_I \text{ in } \graph\big\},
\end{equation*}
where $I$ is maximal when $\nexists$ $I'$ with $I \subsetneq I'$ with $Y \dsep \squarenode{P}{Q}  \mid X_{I'}$.
\item When the causal graph is only partially known, i.e., $\graph \in \graphs$ where $\graphs$ is a set of possible causal graphs, we have
$$\suffcollection = \set_{\shiftset{P}{Q},\graphs} \coloneqq  \bigcup_{\graph \in \graphs} \set_{\shiftset{P}{Q},\graph} .$$
\item When the shift set is partially known, e.g., one has access to $\mathcal{S} \subseteq 2^{[d]}$ such that $\shiftset{P}{Q} \in \mathcal{S}$, we define,

 $$ \suffcollection= \set_{\mathcal{S},\graph}  \coloneqq  \bigcup_{S \in \mathcal{S}} \set_{S,\graph} .$$ 
 \item Finally, when both the causal graph and the shift set are partially known,  encoded by $\graphs \ni G,\mathcal{S}\ni S(P,Q)$, we use
 $$ \suffcollection= \set_{\mathcal{S},\graphs}  \coloneqq \bigcup_{S \in \mathcal{S}} \set_{S,\graphs}  = \bigcup_{\graph \in \graphs} \set_{\mathcal{S},\graph} .$$ 
When no knowledge of the shift set is given at all, we write $\set_{\graphs} \coloneqq \set_{2^{[d]},\graphs}$ and $\set_G \coloneqq \set_{2^{[d]},G} $ as a shorthand. 
\end{itemize}
\end{definition}
For an illustration of sufficient invariance collections, cf. \Cref{fig:main-causal-figure}. 
In the absence of confounding  (i.e., the Markovian setting), knowledge of the causal graph and the shift set is enough to obtain a single $(P,Q)$-invariant model with lowest $Q$-risk so that $\suffcollection = \{I\}$ with $I = [d] \setminus \bigcup_{X \in \Ch(Y) \cap \shiftset{P}{Q} } \De(X)$. This set has been referred to as the \emph{stable blanket} in \cite{pfister2021stabilizing} and the \emph{maximum invariant set} in \cite{gu2024causality}. In non-Markovian cases, however, even the full-knowledge  sufficient invariance collection $\set_{\shiftset{P}{Q},\graph}$ may contain as many as $2^{d/4}$ feature subsets \cite{Jalaldoust_Bareinboim_2024}.

From the domain adaptation perspective, a desirable property of a collection $\set$ of feature subsets is to contain $Q$-predictive $(P,Q)$-invariant models, 
as we may then
achieve a low $Q$-risk 
when we aggregate the collection of predictors  $\set$ 
using the $Q$ samples. The sufficient invariance collection defined in \Cref{def:sufficient-invariance-collection} satisfies this property, that is,
\begin{equation}
\label{eq:wantedproperty}
    \min_{I \in \suffcollection} \riskQ (h_{I,P}) \leq \min_{(P,Q)\text{-invariant } I \subset [d]} \riskQ (h_{I,P})
\end{equation}
for any degree of causal knowledge.



\paragraph{Construction of $\suffcollection$.} In the case of partial causal knowledge, 
we may bypass the exhaustive iteration of finding $\set_{\shiftset{P}{Q},\graph}$ over all graphs $\graph \in \graphs$.
For example, when we have a MEC, for every feature subset $I \subseteq [d]$, one can test the membership $\set_{\shiftset{P}{Q},\graphs}$ using \emph{possible} d-separation graphical criteria in the PAG/CPDAG that represents the MEC, cf., \cite{zhang2008completeness,perkovic2018complete}.
In the presence of heterogeneous source data distributions $\{ {P^e} \}_{e \in \mathcal{E}}$, where $\mathcal{E}$ denotes the set of source domains, we may use recent work (e.g., \cite{peters2016causal}) to identify some parents of $Y$ --  another way to induce a set $\graphs$ of possible graphs, as we will discuss next.  Multi-source data can also be used to obtain 
Markov equivalence classes \citep{li2023causal,hauser2012characterization,yang2018characterizing}, which yield sufficient invariance collections akin to \Cref{def:sufficient-invariance-collection}. 

Besides the causal graph, the shift set $\shiftset{P}{Q} \subseteq [d]$ 
may also be unknown. 
In the absence of any knowledge about $S(P,Q)$, we may simply choose $\mathcal{S} = 2^{[d]} \ni S(P,Q)$, resulting in $\suffcollection = \set_\graphs$.
Note that adding a parent to a feature subset that d-separates $Y$ from $s_{P,Q}$ preserves d-separation. Therefore, due to maximality in \Cref{def:sufficient-invariance-collection} all definite parents of $Y$ in $\graphs$
must be included in every $I \in \set_{ \graphs}$, that is
\begin{align} \label{eq:parents-in-I}
    X_i \in \Pa(Y) \text{ for all } \graph \in \graphs &\implies i \in \bigcap_{I \in \set_{ \graphs}} I.
\end{align}
\paragraph{Alternative collections using multiple environments.}\label{sec:multiple-environments} Now note that \Cref{eq:parents-in-I} can in fact be interpreted in the inverse direction: one can find parents of $Y$ by finding sufficient invariance collections and considering their intersection. In fact, this approach is taken in Invariant Causal Prediction (ICP), which  outputs an intersection of source-invariant subsets to gain (partial) knowledge of the causal parents and obtain a set $\graphs$. The key difference is that the construction of that collection does not need explicit causal knowledge to begin with but is directly based on the invariance using the different domains.  We now formalize these statements.
More concretely, consider the setting where data from multiple source distributions $\{ {P^e} \}_{e \in \mathcal{E}}$ is available, and let $P = \sum_{e \in \mathcal{E}} w_e P^e$ denote the pooled source distribution.
Suppose for every source index $e \in \mathcal{E}$, the distribution $P^e$ is generated by its own SCM over $X,Y$, and the structural equation of $Y$ remains unchanged across all source and target SCMs.
Now for every $I \subseteq [d]$, one can verify whether the invariance $h_{I,P^e} \overset{?}{=} h_{I,P}$ holds for all $e \in \mathcal{E}$, and this defines the following collection.

\begin{definition} \label{def:source-invariance-ICP} For a set of source domains $\env$, we define the \emph{source-invariance collection} of feature subsets as
\begin{equation*}
    \invariantset(\env) \coloneqq \{I \subseteq [d] \text{ such that } I \text{ is } (P^e,P^{e'})\text{-invariant } \forall e,e' \in \mathcal{E} \}.
\end{equation*}
where we usually omit the explicit dependence on $\env$.
\end{definition}

Note that the source-invariance collection itself may not contain the most $Q$-predictive $(P,Q)$-invariant model (cf. \Cref{sec:toy-example}), and therefore, does not satisfy \Cref{eq:wantedproperty} and hence differs significantly from $\suffcollection$. However, it 
contains at least one $(P,Q)$-invariant model; so if that model happens to be $Q$-predictive, then we may gain in the supervised domain adaptation setting. 
In \Cref{sec:optimistic-collection} we elaborate more on collections that have  desirable properties for supervised domain adaptation. As mentioned earlier, a major advantage compared to   
sufficient invariance collection is that the $(P^e,P^{e'})$-invariances are testable using source domains alone,
without \emph{any} explicit knowledge about the SCM a priori.

At the same time, one \emph{can} also use $\invariantset$ to obtain partial knowledge --- encoded via $\graphs$ --- about the underlying graph, and, with that, build the corresponding sufficient invariance collection.
In particular, akin to the converse of \Cref{eq:parents-in-I},
source-invariance collection can be used to reason about the possible parents of $Y$, that in turn gives rise to a set $\graphs$ of possible graphs.

Invariant Causal Prediction (ICP) \cite{peters2016causal}  
takes the intersection of all feature subsets in the collection $\invariantset$, which, under certain conditions, is guaranteed to contain definite parents of $Y$:
\begin{equation*}
    I_{\mathrm{ICP}} \coloneqq \bigcap_{I \in \invariantset} I \subseteq \pa(Y) \text{ in all } \graph \in \graphs.
\end{equation*}
With access to \emph{sufficiently diverse} source domains, ICP is guaranteed to 
perfectly identify
the parent set $\Pa(Y)$. 
In practice, the set $I_{\mathrm{ICP}}$ is often empty, or contains only a very small subset of the parents $\Pa(Y)$, e.g. 
when the source domains are not diverse enough 
or if $Y$ is confounded. In contrast, under the same assumptions, the collection $\invariantset$ is \emph{guaranteed} to contain a $(P,Q)$-invariant model even if the source heterogeneity is insufficient, making it a desirable object for adaptation.

\section{Main results}\label{sec:main-results}
In the previous section, we have introduced the concept of a collection of candidate invariants, and, as an example, discussed so-called sufficient invariance collections (cf. \Cref{def:sufficient-invariance-collection}). The remaining question is whether we can leverage \emph{any} such   
collection and a (small) target sample $\dataQ$
to obtain a model with low target risk; in particular, whether there exists a procedure depending on $\dataP, \dataQ$ and $\set$ which results in much lower target risk than the retraining-on-target baseline $\hempQ$. If so, we would like to quantify how many target samples are sufficient and necessary to benefit from knowledge of a collection $\set$, and whether there are collections which enable especially fast target adaptation. The rest of this section is agnostic to the "causal nature" of the candidate collection $\cal I$. In particular, the following results are valid regardless of whether $\cal I$ contains a feature set with low target risk, or a $(P,Q)$-invariant feature set. We discuss how the degree of causal knowledge in $\cal I$ affects the resulting adaptation guarantees in linear SCMs in \Cref{sec:benefits-of-causal-invariance}.
\smallskip

Denote by $\hypoinit \coloneqq \{ \hempIP: I \in \set \}$ the collection of models resulting from training a feature-subset model $\hempIP$ on $\datasource$ for each candidate feature set $I \in \set$.  As a first thought, we might want to take the straightforward approach and output the best model given $\dataQ$, i.e., compute the empirical risk minimizer $\herm (\hypoinit) \in \arg\min_{I\in\cal I} \risktargetemp(\hempIP)$ over the finite collection of the candidate models $\hempIP$.
This naive strategy has multiple drawbacks: first, empirical risk minimization over finite classes only guarantees a \emph{slow rate} of the form $\sqrt{\log(| \set |) / \ntarget}$.  Thus, the corresponding excess target risk guarantee for the naive approach would be 
\begin{equation*}
    \excesstarget(\herm(\hypoinit)) \leq \min_{I \in \set} \excesstarget(\hempIP) + C \sqrt{\frac{\log | \set |}{\ntarget}}.
\end{equation*}
In particular, if the collection is large, i.e. $| \set | = \Theta(2^d)$, such an estimator is generally outperformed by the target linear model $\hempQ$, which achieves the \emph{fast rate} of linear regression $\excessQ(\hempQ) \leq \Creg d/\ntarget$ \cite{hsu2012random} (cf. \Cref{lem:convergence-linear-regression}).
Moreover, if all candidate models in $\hypoinit$ have high target excess risk (for instance, in case of a large $Y$-shift, or bad estimation of $\hIP$ caused by small $\nsource$ or large covariate shift), $\herm$ is bound to have a high target risk even for very large $\ntarget$, where it is outperformed by the target baseline $\hempQ$. These caveats imply that the potential disadvantage of $\herm$ compared to retraining on target data, also called \emph{negative transfer}, is unavoidable. In the following sections, we show that it is possible to avoid negative transfer and achieve no worse than the best risk $\min_{I \in \set} \excesstarget(\hempIP)$ in the collection given potentially a small amount of target samples, even if the collection $\set$ is large. Further, we show that a similar amount of target samples is \emph{necessary} to benefit from a collection $\set$.   First, we outline a two-step adaptive procedure which achieves our target risk guarantees. 

\vspace{-6pt}
\subsection{Adaptive procedure} 

\begin{algorithm}[t]
  \caption{Two-step Adaptation via Iterative Localized Aggregation}
  \label{alg:iterative-step-agg}
  \begin{algorithmic}[1]

    \Require Source data $\dataP = \{ (x_i, y_i)\}_{i\in [\nsource]}$, target data
    $\dataQ = \{ (x_i, y_i)\}_{i\in [\ntarget]}$ s.t. $| \dataQ | \geq 2$ split into $(\datatargetone,\datatargettwo)$, collection of feature subsets
    $\set \subseteq 2^{[d]}$, constants $C_1, C_2 > 0$, confidence level $0 < \alpha < 1$.

    \State \textbf{Step 0: Split the data and compute models.}
    \State For each $I \in \set$, compute the feature-subset model $\hempIP$ using source data.
    \State Define $\hypoinit \gets \{ \hempIP: I \in \set \}$.
    \State Split target data into two equal parts, $\datatargetone$ and $\datatargettwo$.
    \State Compute the target ERM $\hempQ$ on $\datatargetone$.

    \State \textbf{Step 1: Guard against negative transfer.}
    \State On $\datatargetone$, retain candidates within confidence band of $\hempQ$:
\State $\hypoacc^0 \gets \left\{\, \hempIP :\,
\|\hempIP - \hempQ\|_{\hat{\Sigma}_Q}^2 \le C_1 \tfrac{d + \log(1/\alpha)}{\ntarget} \,\right\}$.
\If{$\hypoacc^0 =\emptyset$}
\State \Return $\houtput =\hempQ$
\EndIf
    \State \textbf{Step 2: Aggregate, refine and iterate.}
    \State Iteratively reject models worse than their aggregate:
    \State Set $i \gets 0$
    \While{true}
      \State On $\datatargettwo$, compute model selection aggregate $\hagg(\hypoacci)$. 
      \State Retain candidates within confidence band of $\hagg(\hypoacci)$:
\State $\hypoacciplusone \gets \{ \hempIP \in \hypoacci : \|\hempIP - \hagg(\hypoacci)\|_{\hat{\Sigma}_Q}^2 \le C_2 \frac{\log |\hypoacci| + \log(1/\alpha)}{\ntarget} \}$.
      \If{$\hypoacciplusone = \emptyset$}
        \State \Return $\hadapt 
         \in \arg\min_{\hempIP \in \hypoacci} \|\hempIP - \hagg(\hypoacci) \|_{\hat{\Sigma}_Q}^2$
      \ElsIf{$\hypoacciplusone = \hypoacci$}
        \State \Return $\hadapt = \hagg(\hypoacci)$.
      \EndIf
      \State $i \gets i + 1$.
    \EndWhile
  \end{algorithmic}
\end{algorithm}
\vspace{-0pt}

We outline and discuss \Cref{alg:iterative-step-agg}, a theoretical procedure for supervised domain adaptation under partial knowledge of causal invariances in form of a collection $\set \subseteq 2^{[d]}$ of candidate feature sets.

\textbf{Input. } We are given datasets $\dataP$ and $\dataQ$ from the source and target distributions, together with a collection of feature subsets $\set \subseteq 2^{[d]}$ which induces a collection of feature-restricted linear models $\hypoinit := \{\hempIP: I \in \set \}$\footnote{In the following, we assume that $\set$ is always augmented by $[d]$, i.e. our collection always includes the full-feature source model.}. The algorithm receives a confidence level $0 < \alpha < 1$ and constants $C_1, C_2 > 0$. We split the target dataset into two equal folds, $\datatargetone$ and $\datatargettwo$.
We state the steps of the procedure informally:
\paragraph{Step 1. Guard against negative transfer.} In the first step, we use the first fold of the target data $\datatargetone$ to retain all models which are guaranteed to have (population) target risk \emph{at most} of the same order as the target ERM. The goal of Step 1 is to \emph{avoid negative transfer}, i.e. outputting a model with a worse target risk compared to retraining. We refer to the retained set as $\hypoacc^0$. If $\hypoacc^0$ is empty, $\hempQ$ is returned.
\paragraph{Step 2. Aggregate, refine and iterate.} The second step can be performed once or iteratively. On $\datatargettwo$, we compute a model selection aggregate $\hagg(\hypoacci)$ over the current retain set $\hypoacc^i$ (cf. \Cref{apx:model-aggregation} for a discussion of different aggregation methods and their guarantees). We then reject any model whose $\Sigmatargetemp$-weighted distance to the aggregate is not within the confidence band $\frac{\log |\hypoacci| }{\ntarget}$, which corresponds to the convergence rate of the aggregate.
Finally, we update the retain set $\hypoacc^{i}$ to exclude the rejected models. Iterative refinement stops when either no more models are rejected or the retain set is empty. 

In the following sections, we provide target generalization guarantees for \Cref{alg:iterative-step-agg} and show how it allows for fast adaptation under favorable target risk margin conditions in the candidate collection $\hypoinit$.


\subsection{Supervised adaptation guarantees}\label{sec:upper-bound}
Our adaptation guarantees rely on some 
distributional assumptions on the source and target domains that we discuss below. 
\subsubsection{Distributional assumptions}\label{sec:distributional-assumptions}
\begin{assumption}[Boundedness]\label{asm:tail-of-X-and-Y}
    For $\prob \in \{\source, \target\}$, the covariates $X$ are bounded with $\| X\|_2 \leq B_X$ a.s., and $|Y| \leq \sigma_Y$.  Additionally, we consider the bounded linear hypothesis class $\hypolinB$ with $B < \infty$, implying bounded restricted hypothesis classes $\hypoI$ for any feature subset $I \subseteq [d]$. 
\end{assumption}
We note that \Cref{asm:tail-of-X-and-Y} is widely regarded necessary to establish fast deviation-optimal guarantees for model selection aggregation of form $\frac{\log |\hypo| }{n}$. Model selection aggregation is part of  \Cref{alg:iterative-step-agg} (cf. \Cref{alg:iterative-step-agg}, Step 2). Consequently, generalizations of aggregation guarantees to sub-Gaussian and heavy-tailed noise designs immediately establish our result in these settings. We provide more details on extensions beyond bounded designs in \Cref{sec:limitations} and \Cref{apx:model-aggregation}. 
\begin{assumption}[Conditioning of covariances]\label{asm:conditioning-of-covariances}
    There exists $\lambda_{\min} > 0$ such that for $\prob \in \{\source, \target\}$, the matrix $\Sigma_D \coloneqq \ED[X X^\top]$ is positive definite with $\lambda_{\min}(\ED[X X^\top]) \geq \lambda_{\min}$.  
\end{assumption}
In particular, \Cref{asm:conditioning-of-covariances} implies that the quantity $\lambda_{\max} \coloneqq \lambda_{\max}(\Sigma_P^{-1}\Sigma_Q)$, which is a measure of covariate shift between $\probsource$ and $\probtarget$, is finite. 
\subsubsection{Upper bound}\label{sec:upper-bound-subsection}
Denote by $\bestinset \in \argmin_{I\in \setI} \excessQ(\hIP)$ any best candidate feature set from $\setI$ on $Q$ in population, and by $\bestrisk \coloneqq \excessQ(\hbestP)$ its target excess risk (similarly, $\epsilon_I \coloneqq \excessQ(\hIP)$ for other candidates). We express our target risk guarantees in terms of the following key quantities. For any $I \in \setI$, a key quantity are the \emph{target risk margins}
\begin{align*}
    \gapIbest \coloneqq \epsilon_I - \bestrisk .
\end{align*}
Further, the $\gapIbest$-estimation error due to having only $\nsource$ source samples 
is captured by the residual 
\begin{align}\label{eq:rnp-definition}
    \maxmarginerror \coloneqq \lambdamax \Creg \frac{d + \log( | \hypoinit | /\alpha)}{\nsource} + 2\sqrt{ \left(\sup_{I \in \setI} \epsilon_I \right) \lambdamax \Creg \frac{d + \log( | \hypoinit | /\alpha)}{\nsource}},
\end{align}
where $\Creg$ is the linear regression concentration constant as in \Cref{lem:convergence-linear-regression}. We note that $\maxmarginerror = \mathcal{O}(\sqrt{d/\nsource})$, in particular, $\maxmarginerror \to 0$ as $\nsource \to \infty$.
Next, by Step 1 of \Cref{alg:iterative-step-agg}, our upper bounds inherit  the fast rate of the excess risk of model selection aggregation for strongly convex losses $\erroragg(\ntarget) \coloneqq \Cagg \frac{\log | \hypoacczero | + \log(1/\alpha)}{\ntarget}$, where $\Cagg$ is the aggregation concentration constant (see \Cref{apx:model-aggregation}) and $| \hypoacczero | \leq  |\hypoinit| \leq 2^d$ is the size of the accepted collection in Step 1 of \Cref{alg:iterative-step-agg}. 
In particular, this rate appears in the largest margin error incurred at  $(\nsource,\ntarget)$:
\begin{align}\label{eq:deltamax}
\gapmax(\nsource,\ntarget)\coloneqq\sup\bigl\{\gapIbest:\gapIbest\le c_0(\bestrisk+\erroragg(\ntarget)+\maxmarginerror)\bigr\}.
\end{align}


We can now state our main result.

\begin{theorem}\label{thm:main-result}
Let $\hypoinit = \{ \hempIP: I \in \set \}$ be a collection of feature-subset linear models indexed by a collection $\set$. Let $0 < \alpha < 1$. Let $\hadapt$ be the output of \Cref{alg:iterative-step-agg}. Assume that $\ntarget \geq C(B_X, \lambda_{\min},d,\alpha)$ as in \Cref{lem:matrix-concentration}, $C_1 \geq 4 \Creg$ and $C_2 \geq 3\Cagg$. 
    Then, under \Cref{asm:tail-of-X-and-Y,asm:conditioning-of-covariances}, it holds with probability at least $1-\alpha$:
{\small
\begin{align*}\label{eq:main-equation}
    \excesstarget(\hadapt)
    &\le \min\Bigl\{
    \underbrace{C\frac{d + \log(1/\alpha)}{n_Q}\vphantom{\excessQ(\hbestP)}}_{\targetrateterm},
    \underbrace{\bestrisk \vphantom{\frac{d}{n_Q}}}_{\bestriskterm}\!+\!\underbrace{\gapmax(\nsource,\ntarget)+3\maxmarginerror\vphantom{\frac{d}{n_Q}}}_{\marginerrorterm}
    \Bigr\},
\end{align*}
}
where $C, c_0 > 0$ are universal constants.
\end{theorem}
We prove \Cref{thm:main-result} in \Cref{sec:actual-proof-of-theorem}.
The upper bound in \Cref{thm:main-result} is a minimum of two rates: $\targetrateterm$ (\textcolor{Bittersweet}{target rate}), the convergence rate of the target linear model, and $\bestriskterm + \marginerrorterm$ (\textcolor{ForestGreen}{best risk} + \textcolor{RoyalBlue}{margin error}), the target excess risk of the best candidate model plus a margin error term $\marginerrorterm$ that converges to zero  as $n_P, n_Q\to\infty$. $\marginerrorterm$ corresponds to the largest risk margin not rejectable by the procedure.  
The $\targetrateterm$  term guarantees \emph{no negative transfer}: the generalization error is never worse than the rate of the linear model which uses target data only and ignores source data.
The second term in the minimum, $\bestriskterm + \marginerrorterm$,  
represents the possibly much lower excess risk one could obtain 
by utilizing source data and partial causal structure as encoded in the  
candidate set collection $\set$. 
In particular, if $\marginerrorterm=0$, we achieve the "best of both worlds" -- the generalization guarantee is both never worse than that of the target model, and potentially as low as the lowest possible excess risk among candidate models from the collection $\set$ trained using source data only. In what follows, we refer to this as the "optimal" gain.

For the following discussion, we assume that the excess risk of the best candidate model is small, $\bestrisk \lesssim  \log | \hypoacczero | / \ntarget$. We ask: in which sample size regimes and margin constellations can we successfully leverage access to a collection $\cal I$ -- and its induced best target risk $\bestrisk$? In other words, in which settings does the bound get close to the "optimal" gain $\min \{ \targetrateterm, \bestriskterm \}$? 
We observe that if the 
source data is sufficiently  abundant such that $\maxmarginerror \lesssim  \log | \hypoacczero | / \ntarget $, we achieve excess risk of at most $\bestrisk$ up to the \emph{fast decreasing margin error}
\begin{equation}\label{eq:fast-rate}
    \gapmax(\nsource, \ntarget) + 3 \maxmarginerror \lesssim \log | \hypoacczero | / \ntarget. 
\end{equation}
Note that the resulting rate depends only on the size $| \hypoacczero |$ of the accepted collection, which can be much smaller than the initial number of candidates $| \hypoinit |$ that appears in the term $\maxmarginerror$. In particular, if most models are "outperformed by the target baseline", in the sense that $\excessQ(\hIP) \gtrsim \frac{d}{\ntarget}$, then even for exponential-sized collections $ | \hypoinit | = 2^d$, it may be that $ | \hypoacczero | = \mathcal{O}(1)$, resulting in the rate $\mathcal{O}(1/\ntarget)$.
\smallskip

We note that \Cref{thm:main-result} requires a minimum number of target samples $\ntarget \geq C(B_X, \lambda_{\min},d,\alpha)$, where $C(B_X, \lambda_{\min},d,\alpha) = \mathcal{O}\left(\frac{B_X^2}{\lambda_{\min}} \log( \frac{d}{\alpha} ) \right)$. In particular, \Cref{asm:conditioning-of-covariances} implies that $\ntarget \gtrsim d \log \frac{d}{\alpha}$ target samples are required for our generalization bound to hold. Although the minimum number of samples scales linearly in $d$, it does not depend on the excess risk which the bound potentially achieves: while the standard linear regression estimator requires $\ntarget \gtrsim \frac{d}{\epsilon}$ samples to achieve an excess risk of at most $\epsilon$, in case of favorable margin structure, an arbitrarily low excess risk $\bestrisk$ can be achieved via \Cref{thm:main-result} given an amount of target samples which does not depend on $\bestrisk$. In particular, in such settings, our bound can enable \emph{few-shot learning}: achievement of much lower excess risk than the target estimator given the same amount of samples.

\subsubsection{Corollaries}
For the rest of the discussion, we assume $\nsource \to \infty$ for simplicity. Under the conditions of small excess risk $\bestrisk$, the amount $\ntarget \gtrsim \max \{ \frac{ \log | \hypoacczero |}{\inf_{I \in \set} \gapIbest}, C(B_X, \lambda_{\min},d,\alpha) \} $ of target samples suffice to achieve the excess risk $\bestrisk$ of the best candidate model: In this case, we have $\gapmax(\nsource,\ntarget) = 0$ so that the convergence 
rate $\gapmax(\nsource, \ntarget) + 3 \maxmarginerror=0$. 
The following corollary directly follows from the theorem and illustrates how the margin determines the sample complexity in the two-model case. 
\begin{corollary}[Two models]\label{cor:two-models}
    Let $\hypoinit = \{\hbestP, \hsecondbestP\}$. 
    Recall the risk margin $\Delta_I = \epsilon_I - \bestrisk$. There exist  constants $0<c_0<1$, $C_0 > 0$ such that if  $\bestrisk \leq c_0 \Delta_I$, or $\bestrisk \leq C_0 \erroragg(\ntarget)$, and $\ntarget \gtrsim \max \{\frac{\log(1/\alpha)}{\Delta_I}, C(B_X, \lambda_{\min},d,\alpha) \}$, it holds with probability at least $1-\alpha$ for the output $\hadapt$ of \Cref{alg:iterative-step-agg} that
    \begin{align*}
        \excesstarget(\hadapt) \leq \min \left\{C\frac{d + \log(1/\alpha)}{n_Q}, \bestrisk \right\}. 
    \end{align*}
\end{corollary}
In other words, given just $\ntarget \gtrsim \frac{1}{\Delta_I}$ target samples, it is possible to pick either the best  candidate model or the target estimator. For instance, if the two-model collection consists of the source model $\hP$ and one identified invariant model $\hIP$, \Cref{cor:two-models} allows us to pick the "best of both worlds", as opposed to the structure-agnostic transfer bound $\excesstarget(\hadapt) \leq \min \left\{C\frac{d + \log(1/\alpha)}{n_Q}, \excessQ(\hP)  \right\}$ in \cite{hanneke2025adaptivesampleaggregationtransfer} involving just the source model. As we illustrate in \Cref{sec:benefits-of-causal-invariance}, this allows gains w.r.t. the source model in large-shift scenarios, as well as w.r.t. the invariant model in small-shift scenarios. 
\smallskip

We can easily extend the spirit of the two model case to multiple models. In the following corollary, we illustrate an  advantageous case for adaptation in which a group of "$\tau$-optimal" models is separated from a group of much worse models by a large margin $\taugap$:

\begin{corollary}[Adaptation with low- and high-risk models]\label{cor:stable-unstable-models}
    For $\tau \geq \bestrisk$, let $\taugap \coloneqq \inf_{\epsilon_I > \tau} \epsilon_I - \sup_{\epsilon_{I'} \leq \tau} \epsilon_{I'}   $. There exists $0 < c_0 < 1$, such that if $\tau \leq c_0 \taugap$ and $\ntarget \gtrsim \max \{ \frac{\log | \hypoacczero | + \log(1/\alpha)}{\taugap}, C(B_X, \lambda_{\min},d,\alpha) \}$, with probability at least $1-\alpha$ it holds that
        \begin{align*}
        \excesstarget(\hadapt) \leq \min \left\{C\frac{d + \log(1/\alpha)}{n_Q}, \tau \right\}. 
    \end{align*}
\end{corollary}
In other words, even if there is a large group of candidate models with target risk at most $\tau \ll \frac{d}{\ntarget}$, $\ntarget \gtrsim \frac{\log |\hypoacczero|}{\taugap}$ target samples suffice to achieve risk $\tau$, and large margins $\taugap$ between groups of "good" and "bad" models enable fast adaptation even if the best model cannot be identified (cf. \Cref{fig:tau-risk}). We show how this situation can occur under shared causal structure of the data and strong structural shifts in \Cref{sec:supervised-adapt-toy-example}.
\begin{figure}[h!]
    \centering
    \includegraphics[width=0.8\linewidth]{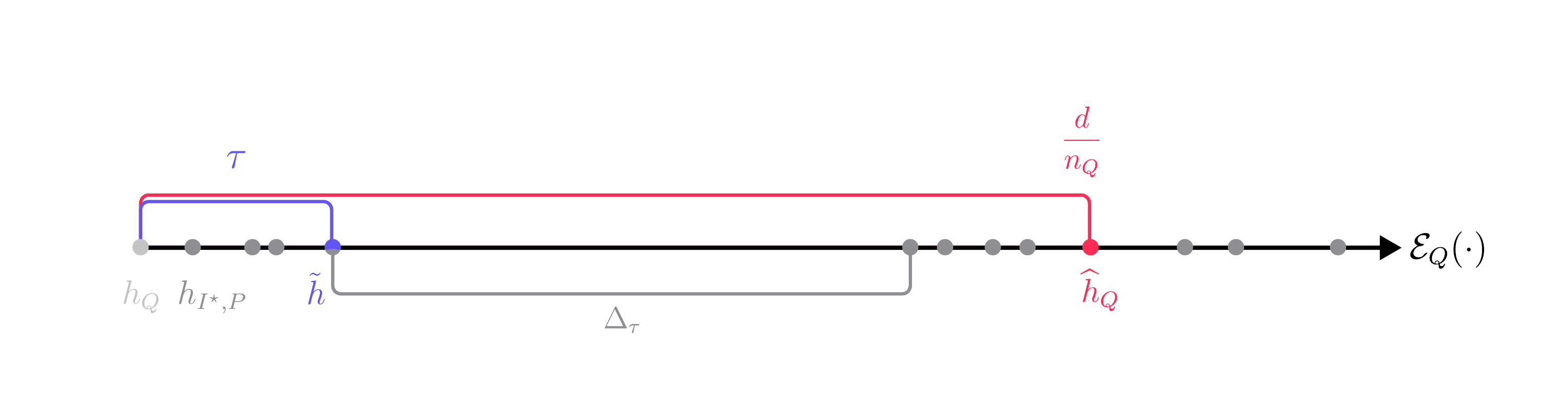}
    \caption{Guarantees for our procedure $\hadapt$ (shown in purple) in the case of \Cref{cor:stable-unstable-models}. A small excess risk $\tau$ can be achieved given $\ntarget \gtrsim \log | \hypoacczero| / \taugap$ samples, provided $\taugap$ is sufficiently large. In particular, $\hadapt$ does not have to select the best model $\hbestP$ to achieve the guarantee.  }
    \label{fig:tau-risk}
\end{figure}
\subsubsection{Comparison with naïve model selection and aggregation}
We now discuss how some natural approaches, which utilize a collection of models for adaptation, generally result in slower rates than \Cref{thm:main-result}. 
As we have suggested earlier, one straightforward way of utilizing the candidate model collection $\hypoinit$ is to simply perform ERM over these models using target data. Even if we include a target model (trained on a separate target data split) as a candidate in our collection, we merely obtain 
\begin{equation*}
    \excessQ(\herm(\hypoinit)) \leq \min \left\{\Creg \frac{d + \log(1/\alpha)}{n_Q}, \bestrisk \right\}  + C \sqrt{\frac{\log |\set|}{\ntarget}}.
\end{equation*}
The \emph{slow rate} $\sqrt{\frac{\log |\set|}{\ntarget}}$ of ERM over models does not prevent negative transfer, depends on the – potentially very large - initial number $|\set|$ of candidate models, and is worse than our guarantee in most adaptation regimes (cf. \Cref{fig:excess-risk-plot}). A more advanced baseline---Step 1 of \Cref{alg:iterative-step-agg} followed by ERM over the accepted models---results in margin error term of order $\sqrt{\frac{\log |\hypoacczero|}{\ntarget}}$, a slow rate compared to the margin error  incurred in \Cref{thm:main-result}.   
\begin{figure}[h!]
    \centering
    \includegraphics[width=0.8\linewidth]{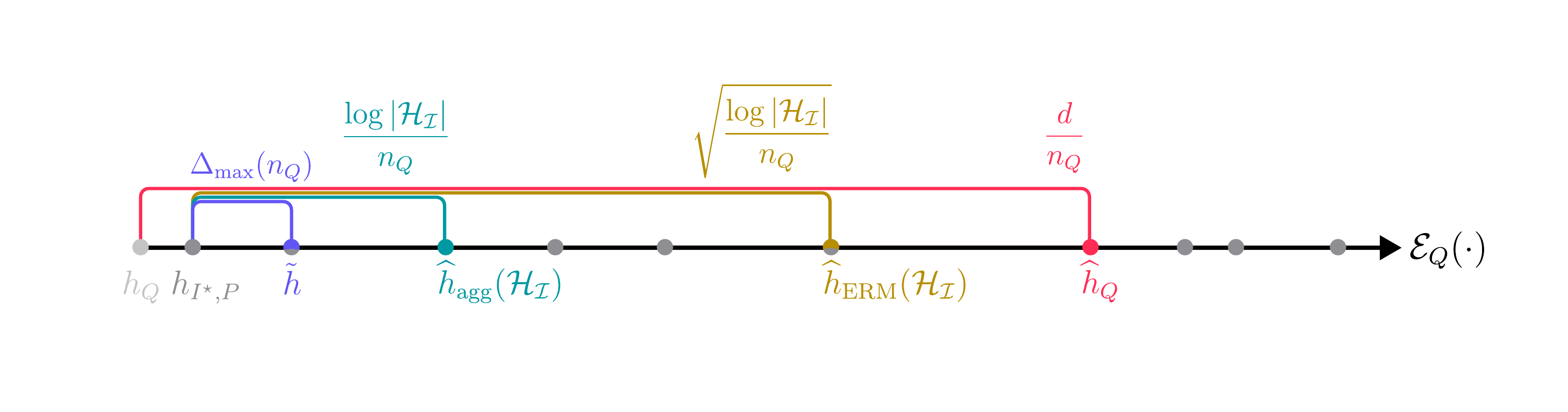}
    \caption{Guarantees for our procedure $\hadapt$ (shown in purple), model aggregation $\hagg(\hypoinit)$ (teal) and ERM over models $\herm(\hypoinit)$ (yellow) are shown with respect to the excess risk of the best candidate model $\hbestP$. Excess risk guarantee for the target linear model $\hempQ$ is shown in red. Grey dots denote candidate models $\hIP \in \hypoinit$. The figure illustrates the scenario $\nsource \to \infty$, $\bestrisk \lesssim \log |\hypoacczero|/\ntarget$.  }
    \label{fig:excess-risk-plot}
\end{figure}

Another line of well-developed literature \cite{tsybakov2003optimal,lecue2014optimal,dai2012deviation,rigollet2012kullback} would suggest directly leveraging the candidate collection
by more efficient model aggregation (cf. \Cref{apx:model-aggregation})
that can achieve an excess risk bound
\begin{equation*}
    \excessQ(h_{\mathrm{agg}}) \leq \excessQ(\hbestP) + \Cagg \frac{\log |\hypoinit|}{\ntarget}.
\end{equation*}
Again, the initial number of models $|\hypoinit|$ can be much larger than $| \hypoacczero |$ which enters our rate via $\erroragg(\ntarget)$.  Further, our margin error $\gapmax(\nsource,\ntarget)$ is upper bounded by the aggregation rate in case of small $\bestrisk$. Additionally, it can in fact be zero, and thus much smaller than the aggregation rate (cf. \Cref{fig:excess-risk-plot}) in case the risk margins $\Delta_I$ to the other candidates  are sufficiently large (cf. \Cref{cor:two-models} and \Cref{cor:stable-unstable-models}).

\subsection{Lower bounds}

In the following lower bounds, we show necessity of our margin-driven sample complexity $\ntarget \gtrsim 1/ \Delta$ in the two model case $| \set | = 2$. For brevity, we assume $\nsource = \infty$ and note that the lower bounds trivially hold for $\nsource < \infty$. For a fixed best excess target risk $\epsilon_1 \geq 0$ and margin $\Delta \geq 0$, consider the family $\Sigma(\epsilon_1,\Delta, \set)$ of pairs of distributions $(P,Q)$ on $\R^d \times \R$ for which there exists a collection $\set = \{I_1,I_2\}$ s.t. 
\begin{align*}
 \min \{ \excessQ(h_{I_1,P}),\excessQ(\h_{I_2,P})\} &\leq \epsilon_1; \\
 \max \{ \excessQ(h_{I_1,P}),\excessQ(\h_{I_2,P})\} &\leq \epsilon_1 + \Delta.
\end{align*}
In other words, we consider all pairs of distributions $(P,Q)$ for which there exists a feature-set collection $\set = \{I_1,I_2\}$ such that the best candidate model trained on $P$ has target excess risk at most $\epsilon_1$, and the second-best candidate model has target excess risk at most $\epsilon_1 + \Delta$. 

First, we outline a lower bound for all selector algorithms, i.e., all algorithms which can output either $h_{I_1,P}$ or $h_{I_2,P}$. We implicitly assume $ \max \{ \epsilon_1, \epsilon_1 + \Delta \} \lesssim \frac{d}{\ntarget}$, as we are interested in the optimality of selection of the candidate models, as opposed to interpolation between the candidates and the target model. For latter, we refer to the lower bound in \cite{hanneke2025adaptivesampleaggregationtransfer}.
\begin{theorem}[Selection lower bound]\label{thm:lower-bound1} 
Let $d \geq 3$. 
 Consider the family $\Sigma(\epsilon_1,\Delta)$ of pairs of distributions $(P,Q)$ on $\R^d \times \R$. 
Given datasets $\dataP$, $\dataQ$, 
consider any selector algorithm $\mathcal{A}: \dataP \times \dataQ \times \{I_1,I_2\} \to \{h_{I_1,P}, h_{I_2,P} \}$.
There exist constants $c_0, c_1 > 0$ such that if  $\ntarget \leq c_0 /\Delta$, 
\begin{align*}
    \inf_{\mathcal{A}} \sup_{\Sigma} Q(\excessQ(\mathcal{A}) \geq \epsilon_1 + \Delta) \geq c_1. 
\end{align*}
\end{theorem}
We prove \Cref{thm:lower-bound1} in \Cref{apx:proof-of-lower-bound}. \Cref{thm:lower-bound1} implies tightness of our upper bound in the target sample size $\ntarget$ for selector algorithms, as well as necessity of the margin-driven target sample complexity $\ntarget \gtrsim 1/\Delta$. In words, any selector algorithm, if given $\ntarget \lesssim 1/\Delta$ samples, necessarily incurs the target risk $\epsilon_1 + \Delta$ of the worst candidate model. In particular, our lower bound outlines a regime in which there is indeed \emph{no benefit} of causal knowledge in the  form of the collection $\set$ compared to unstructured sDA which picks a single model from the collection (source model): the resulting guarantees are trivially upper bounded by the risk of the worst model in the collection.  

Next, we state our information-theoretic lower bound, which holds for any algorithm not limited to selectors.
\begin{theorem}[Information-theoretic lower bound]\label{thm:lower-bound2}
    Consider the family $\Sigma(\epsilon_1,\Delta)$ and any learning algorithm $\mathcal{A}$ with domain $\dataP \times \dataQ \times \{I_1,I_2\}$. There exist constants $c_1, c_0 > 0$ such that if $\ntarget \leq c_0/\Delta$, 
\begin{align*}
    \inf_{\mathcal{A}} \sup_{\Sigma} Q(\excessQ(\mathcal{A}) \geq \frac{1}{8} \Delta) \geq c_1.
\end{align*}
\end{theorem}
We prove \Cref{thm:lower-bound2} in \Cref{apx:proof-of-lower-bound2}. \Cref{thm:lower-bound2} shows that no learning algorithm, even improper, can utilize the candidate models for adaptation beyond the risk margin error $\Delta$ if $\ntarget \lesssim 1/\Delta$. We conclude that our upper bound is tight in the regime $\epsilon_1 \lesssim \Delta$ and $|\hypoacczero| = \mathcal{O}(1)$. 

In total, our lower bounds \Cref{thm:lower-bound1} and \Cref{thm:lower-bound2} show optimality of the $1/\Delta$ margin dependence of our generalization upper bound. Further, they illustrate an important scenario in which we indeed \emph{cannot benefit} from partial knowledge of causal invariances: whenever the margins $\Delta$ are not sufficiently large with respect to the target sample size $\ntarget$, or, equivalently, not enough target samples are observed to distinguish predictive and non-predictive candidates, an excess risk price of order $\Delta$ has to be paid. 
Still, our results leave open the question whether the $\log |\hypoacczero|$ dependence on the number of accepted candidates is necessary or can be further reduced, for instance, in large-margin scenarios such as \Cref{cor:stable-unstable-models}. We leave this as an open question for future work.  

\section{Benefits of causal invariance for sDA}\label{sec:benefits-of-causal-invariance}


In this section, we  instantiate and interpret the bounds in \Cref{sec:main-results} in concrete causal examples, in particular, in the case of full and partial causal knowledge. We specifically focus on risk margins and their relation to structural shift magnitude, as discussed in causal \emph{finite robustness} literature \cite{rothenhausler2021anchor,shen2023causalityoriented}.
\subsection{Toy causal example} \label{sec:toy-example}\label{sec:supervised-adapt-toy-example}
Recall the skin cancer example illustrated in \Cref{fig:dag-skin-cancer}.
We can formalize it as a regression problem with features $X = (X_1,X_2,X_3,X_4) \in \R^4$ and label $Y \in \R$. Assume that the source and target distributions $P,Q$ are generated by the following linear structural causal models (SCMs) indexed by $D \in \{P,Q\}$:
\begin{align}
\label{eq:runningexample}
  U_1,&U_2,U_3,U_4,U_Y \overset{i.i.d.}{\sim} \Unif([-\sqrt{3},\sqrt{3}])   \nonumber\\
  X_1 &= U_1 + \mu_{D,1}, \quad
  X_2 = U_2 + \mu_{D,2}, \\
  Y   &= X_1 + X_2 + U_Y,\quad
  X_3 = Y + U_3 + \mu_{D,3}, \nonumber\\
  X_4 &= Y + U_4 + \mu_{D,4}.\nonumber
\end{align}
Evident from the structural equations above, the feature $X_i$ \emph{shifts} between domains $D,D'$ if $\mu_{D,i} \neq \mu_{D',i}$. Suppose $\mu_{Q,4} = \mu_{P,4}$, while $\mu_{Q,i} - \mu_{P,i} = t > 0$ for $i\in \{1,2,3\}$, i.e., the causal mechanism of $X_4$ is shared across source and target domains, whereas $X_1, X_2, X_3$ each undergo a structural shift of magnitude $t$. In this sense, the shift set is $\shiftset{P}{Q} = \{1,2,3\}$, 
and the corresponding selection diagram is shown in \Cref{fig:main-causal-figure} (a). Depending on the shift magnitude $t$, different feature-subset models 
achieve best risk on $Q$: for instance, if $t = 0$, we have trivially $P=Q$ and $\hP$ corresponds to the best predictor.
More concretely, the ``causal'' predictor $h_{\{1,2\},P}$ on the parents $\Pa(Y) = \{1,2\}$ of the label $Y$ is $(P,Q)$-invariant,
and thus – for the linear SCMs in \Cref{eq:runningexample} – it has \emph{invariant risk} $\riskQ(h_{\{1,2\},P}) = \riskP(h_{\{1,2\},P}) = 1$ regardless of the target shift strength $t$, and even regardless of the shift set $\shiftset{P}{Q}$. However, since $\shiftset{P}{Q} = \{1,2,3\}$, there exists another $(P,Q)$-invariant feature set $\{1,2,4\}$, which allows us to achieve an even better target risk $\riskQ(h_{\{1,2,4\},P}) = 1/2$ regardless of the shift strength $t$. In contrast, the full-feature model $\hP$ attains potentially unbounded target risk $\riskQ(\hP) = 1/3 + t^2/9$.
At the same time, for small shifts $t  \leq \sqrt{3/2}$ 
, the full-feature model outperforms both invariant models (cf. \Cref{fig:simple-example-plot}).


\paragraph{Comparison with prior work.}
As discussed in \Cref{sec:multiple-environments}, prior work in causality-based DG often relies on invariance tests on multi-source data.
Suppose in this example, we have data from two source domains, $P^1$ and $P^2$ that are generated by the template \Cref{eq:runningexample}, with $\mu_{P^1,1} - \mu_{P^2,1} > 0$, and $\mu_{P^1,i} = \mu_{P^2,i}$ for $i\in \{2,3,4\}$. This means that the cross-source shift set is $\shiftset{P^1}{P^2} = \{1\}$ and the source-target shift set is $\shiftset{P}{Q} = \{1,2,3\}$. 
By testing invariance across the sources, we will obtain the source-invariance collection $\invariantset = \{I \subseteq [d]: \{1 \} \subseteq I\}$, and consequently $I_{\mathrm{ICP}} = \{1\}$ that corresponds to the model $h_{\{1\},P}$ with unbounded target risk (cf. \Cref{fig:simple-example-plot}), i.e., the subset of parents of $Y$ discovered by ICP is not $(P,Q)$-invariant. Additionally, other multi-source domain generalization methods such as IRM \cite{arjovsky2019invariant}, VREx \cite{krueger2021out} and groupDRO \cite{sagawa2019distributionally} all coincide with the source model $\hP$ which also has unbounded target risk (cf. \Cref{fig:simple-example-plot}). In contrast, in our example, the collection $\invariantset$ satisfies \Cref{eq:wantedproperty} regardless of the magnitude of the shift $t > 0$, since it contains the most $Q$-predictive $(P,Q)$-invariant model $h_{\{1,2,4\},P }$.  

\begin{figure}[t]
    \centering
    \begin{subfigure}[t]{0.32\linewidth}
        \centering
        \includegraphics[width=\linewidth]{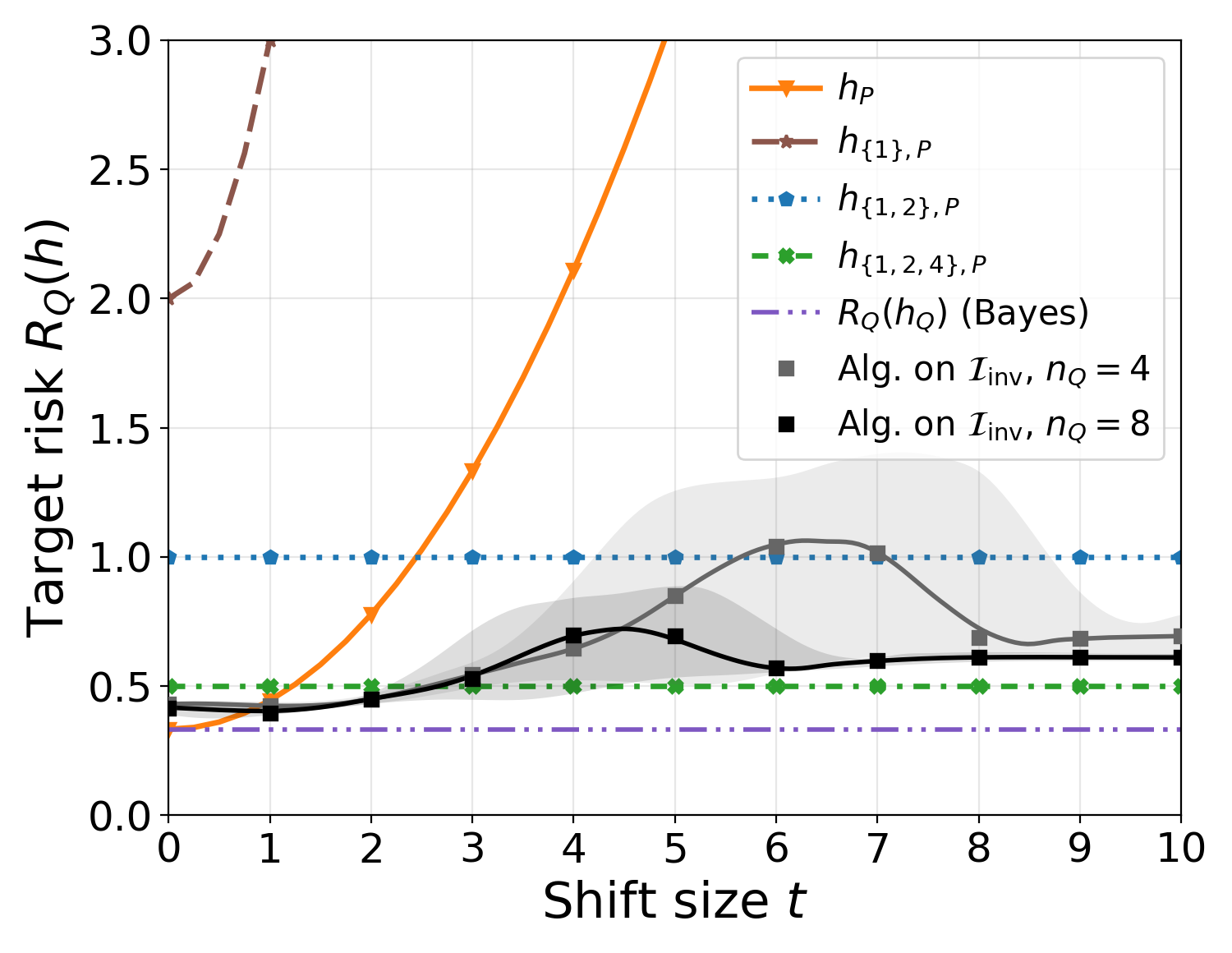}
        \caption{Source subset models and $\hadapt$ for $\invariantset$.}
        \label{fig:simple-example-plot-a}
    \end{subfigure}\hfill
    \begin{subfigure}[t]{0.32\linewidth}
        \centering
        \includegraphics[width=\linewidth]{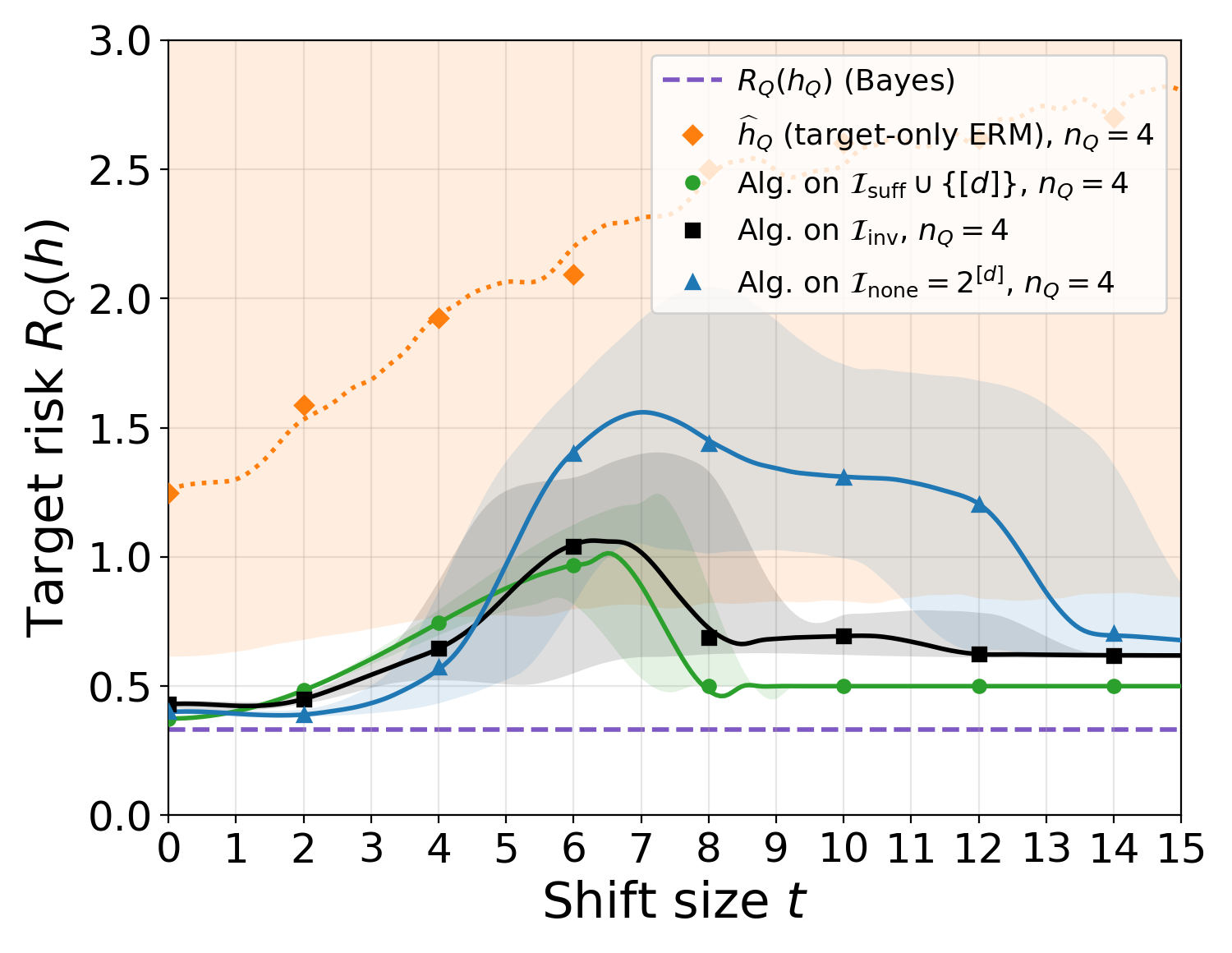}
        \caption{$\hadapt$ compared with $\hempQ$ for $\ntarget=4$.}
        \label{fig:simple-example-plot-b}
    \end{subfigure}\hfill
    \begin{subfigure}[t]{0.32\linewidth}
        \centering
        \includegraphics[width=\linewidth]{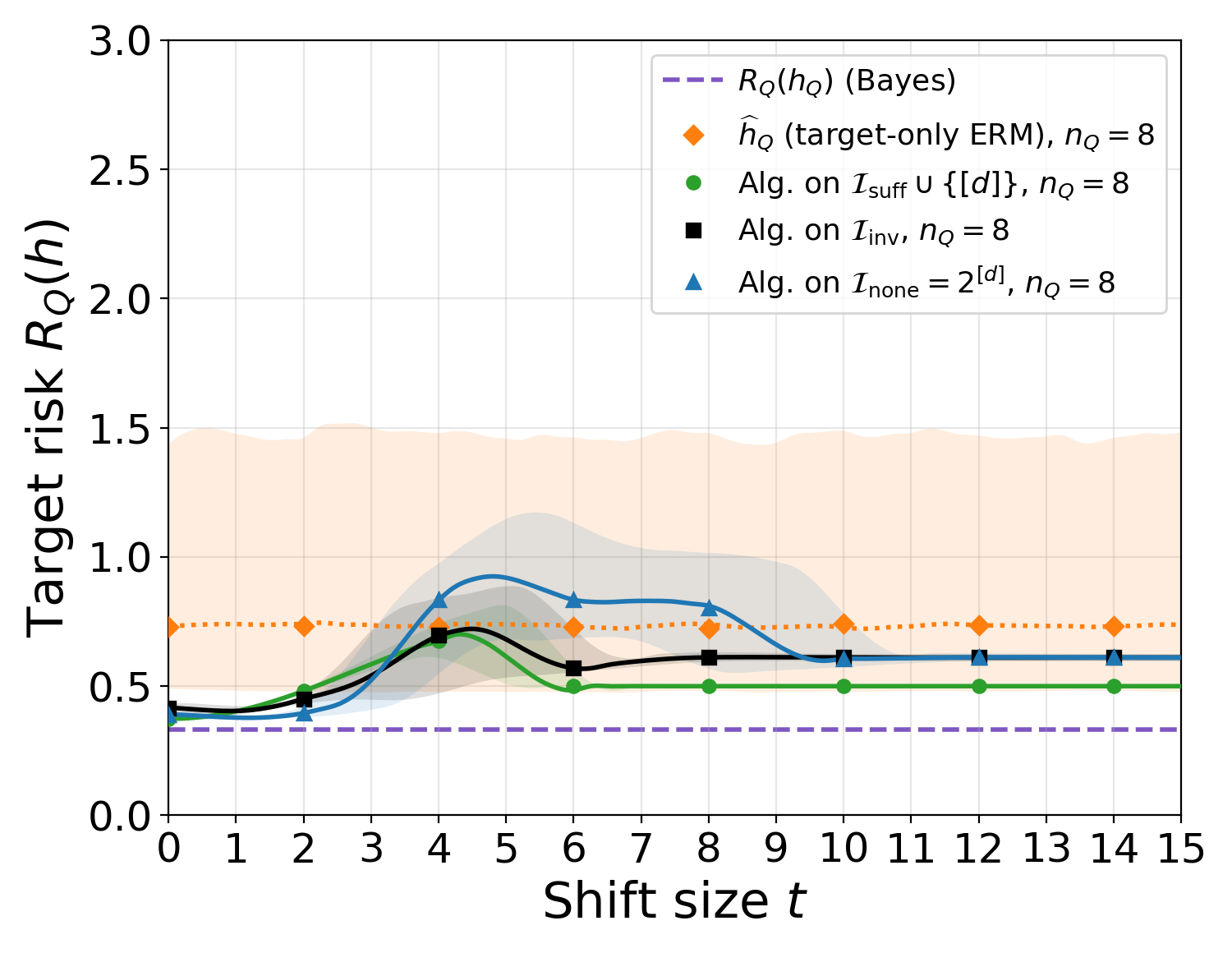}
        \caption{$\hadapt$ compared with $\hempQ$ for $\ntarget=8$.}
        \label{fig:simple-example-plot-c}
    \end{subfigure}
    \caption{Target risk for the toy example in \Cref{eq:runningexample} as a function of the shift size $t$. \textbf{(a)} Population target risks of the candidate feature-subset models trained on $P$ ($\nsource=1000$): $\hP$ (orange), $h_{\{1\},P}$ (brown), $h_{\{1,2\},P}$ (blue), $h_{\{1,2,4\},P}$ (green); dashed purple line: Bayes risk $\riskQ(h_Q) = 1/3$. Overlaid in black are two instances of \Cref{alg:iterative-step-agg} on $\invariantset$ at $\ntarget \in \{4, 8\}$ (light shade for $\ntarget=4$, dark shade for $\ntarget=8$). \textbf{(b)} $\hadapt$ at $\ntarget=4$ across three candidate collections: $\suffcollection \cup \{[d]\}$ (green), $\invariantset$ (black), $\set_{\mathrm{none}} = 2^{[d]}$ (blue); the target-only baseline $\hempQ$ at $\ntarget=4$ in orange (dashed). \textbf{(c)} Same as (b), at $\ntarget=8$. Each curve is the smoothed median across $3000$ runs; the shaded band is the $16$th--$84$th percentile of $\riskQ(\widehat{h}_r)$ over runs.}
    \label{fig:simple-example-plot}
\end{figure}
We now discuss how this picture changes when we are given a limited amount of target data from $Q$, and how the bound in \Cref{thm:main-result} applies to our toy example.


\textbf{Two-model collection.}
We first consider the simplest case in which both the causal graph and the shift set $\shiftset{P}{Q} = \{1,2,3\}$ are known, but the shift magnitude $t$ is not.
Under full causal knowledge, the sufficient invariance collection (cf.\ \Cref{sec:collections}) reduces to the singleton $\suffcollection = \{\{1,2,4\}\}$; we augment it with the full feature set $[d]$ as a baseline, giving the two-element collection $\suffcollection \cup \{[d]\} = \{\{1,2,4\}, [d]\}$ and the corresponding two-model dictionary $\hypoinit = \{h_{\{1,2,4\},P}, \hP\}$. Recall from above that, for any such target $Q$,
\begin{equation*}
    \riskQ(h_{\{1,2,4\},P}) = 1/2, \qquad \riskQ(\hP) = \frac{1}{3} + \frac{t^2}{9}.
\end{equation*}
First, consider the \emph{small shift} case $t \leq \sqrt{3/2}$, in which case $\hP$ is the best model. 
For the SCM in \Cref{eq:runningexample}, additive shifts leave the residual structure intact, so $R_Q(h_Q) = 1/3$ at every $t$ and the two-model excess risks are $\excessQ(\hP) = t^2/9$ and $\excessQ(h_{\{1,2,4\},P}) = 1/6$. Then, the risk margin is $\Delta = 1/6 - t^2/9$ and for $\ntarget \gtrsim 1/\Delta$, the procedure in \Cref{alg:iterative-step-agg} correctly selects the full-feature source model and achieves the best target excess risk $\excessQ(\hadapt) = \excessQ(\hP)$ as long as $\excessQ(\hP) = t^2/9 \lesssim 1/\ntarget$ (cf. \Cref{cor:two-models}). In particular, for very small shifts $t$, the procedure achieves near-zero excess risk given a constant $\ntarget$. Next, consider the \emph{large shift} case $t \geq  \sqrt{3/2}$, resulting in the invariant model $ h_{\{1,2,4\},P}$ as the best model. The risk margin is $\Delta = t^2/9 - 1/6 = \Theta(t^2)$ and the excess target risk of the invariant model is the constant $\excessQ( h_{\{1,2,4\},P}) = 1/6$. Thus, in the target sample regime $1/t^2 \lesssim \ntarget$, our procedure achieves the low target risk of the best invariant model $\excessQ(\hadapt) = \excessQ( h_{\{1,2,4\},P}) = 1/6$. This shows how large shifts $t$ result in large risk margins and thus enable fast adaptation -- excess risk capped at a constant given only $\ntarget \gtrsim 1/t^2$ target samples, while structure-agnostic procedures incur the unbounded source-model risk $\riskQ(\hP) = 1/3 + t^2/9$.
\smallskip


\textbf{Adaptation using multiple source domains.} Returning to the multi-source setting introduced above, suppose we now additionally observe a few target samples from $Q$. Recall that cross-source invariance testing yields only the source-invariance collection $\invariantset = \{ I : \{1\} \subseteq I\}$, since it cannot distinguish which features beyond $X_1$ are stable from source to target. Within $\invariantset$, the two $(P,Q)$-invariant models $h_{\{1,2\},P}$ and $h_{\{1,2,4\},P}$ have the stable target risks $1$ and $\tfrac{1}{2}$, while every remaining candidate has target risk scaling as $\Theta(t^2)$. For large shifts $t$, this separation between ``stable-risk'' and ``unstable-risk'' models places us in the regime of \Cref{cor:stable-unstable-models} with the  ``stable-risk threshold'' $\tau \leq \tfrac{2}{3}$ and margin $\Delta_\tau = \Theta(t^2)$. For the adaptive procedure, \Cref{cor:stable-unstable-models} then implies:  
$$ \text{If }  \ntarget \gtrsim \frac{1}{t^2}, \quad \excessQ(\hadapt) \leq \frac{2}{3}.$$
Whether the procedure can further refine ``an invariant'' to ``the best invariant'' model  --- i.e., improve from the parents-based model $h_{\{1,2\},P}$ to the true best $(P,Q)$-invariant model $h_{\{1,2,4\},P}$ --- depends on the risk margin $1/2$ between the two, which does not increase with the shift strength $t$. 

We illustrate these regimes empirically in \Cref{fig:simple-example-plot}. \Cref{fig:simple-example-plot} (a) shows two instances of \Cref{alg:iterative-step-agg} on the cross-source invariance collection $\invariantset$ at $\ntarget \in \{4, 8\}$ together with the four source candidate models: at small shifts, the algorithm selects $\hP$; at large shifts, the margin grows as $\Theta(t^2)$ and the procedure switches to the maximal-invariant $h_{\{1,2,4\},P}$.
The two $\ntarget$ values differ in the intermediate-shift regime: $\ntarget=4$ target samples do not suffice to iteratively prune the eight candidates in $\invariantset$; at $\ntarget=8$ the band is much tighter and the excess risk is markedly smaller. In \Cref{fig:simple-example-plot} (b) and (c), the collection is varied across $\suffcollection \cup \{[d]\}$, $\invariantset$ and $\set_{\mathrm{none}} = 2^{[d]}$, at $\ntarget=4$ (panel (b)) and $\ntarget=8$ (panel (c)). Three regimes are visible. \emph{(i) Small shifts.} For both $\ntarget$ values the three algorithm curves nearly coincide and track $\riskQ(\hP) = 1/3 + t^2/9$: in every collection, $\hP$ has the lowest target risk and is well separated from the invariants by the constant margin $\Delta = 1/6 - t^2/9$, so the procedure falls back to the source model. \emph{(ii) Medium shifts}. The risks of the non-invariant candidates scale as $\Theta(t^2)$ (cf.\ \Cref{cor:linear-scm}); at $\ntarget=4$ such candidates cannot be reliably pruned, and the algorithm climbs above the invariant baseline, more so the larger the collection, with a hump roughly ordered by collection size.
At $\ntarget=8$ the bands tighten and the three curves move toward the invariant baseline. \emph{(iii) Large shifts.} The risk margin $\Theta(t^2)$ is now large enough that even $\set_{\mathrm{none}}$'s iterative filter rejects $\hP$ and the other non-invariant models at both $\ntarget$ values, and all three curves converge to $\riskQ(h_{\{1,2,4\},P}) = 1/2$.
%

In the next section, we generalize our observations made on the toy SCM example to all linear SCMs under additive shifts.




\subsection{Supervised adaptation bound for linear SCMs under additive shifts}
In this section, we show that our toy example illustrates a phenomenon which holds for a general type of shared causal structure, namely when the data $(X,Y)$ are generated by a linear Markovian structural causal model and $P$  and $Q$ differ via an \emph{additive shift}. More concretely, we fix a DAG over the variables $V = (X,Y) \in \R^d$. Let $B \in \R^{(d+1) \times (d+1)}$ be a weighted adjacency matrix such that $B_{ij} = 0$ whenever $i \notin \pa(V_j)$ and $U \in \R^{d+1}$ a vector of independent unobserved variables with $\E_P[U] = \E_Q[U] = 0$, $\Var_P(U) = \Var_Q(U) = \Sigma \coloneqq \diag(\sigma_1^2,...,\sigma_d^2,\sigma_Y^2)$. The SCM is given by
\begin{equation}\label{eq:linear-scm}
    V = B V + U + t \addshift_D,
\end{equation}
where w.l.o.g. $\addshift_P = 0$ $P$-a.s., and $\addshift_Q$ is a random variable satisfying $U \perp \addshift_Q$, $\E_Q[\addshift_Q] = \mu $ with $\| \mu \|_2 = 1$ and $\Var_Q(\addshift_Q) = H = \diag(\eta)$. Hence, $\E_Q[t \addshift_Q] = t \mu$ and $\Var_Q(t \addshift_Q) = t^2 H$. In other words, for $t > 0$, $t \addshift_Q$ is a (random) additive shift with varying \emph{shift strength} $t$. If $H = 0$, it reduces to a deterministic mean shift as illustrated in our toy example. Let $\shiftset{P}{Q} \subset [d]$ be the shift set from $P$ to $Q$. Observe that in $\mu,\eta$ the non-zero indices are a subset of $\shiftset{P}{Q}$. As before, we augment the causal graph with the shift node $\squarenode{P}{Q}$ pointing to nodes in $\shiftset{P}{Q}$ (cf. \Cref{sec:new-causal}). The following supervised adaptation guarantee illustrates how the magnitude of additive shifts in linear SCMs influences the risk margin structure between the candidate models:
\begin{corollary}[Linear SCMs under additive shifts]\label{cor:linear-scm}
    Assume that $P$ and $Q$ follow the data-generating process \eqref{eq:linear-scm} and $\nsource = \infty$.
    Consider a collection $\set \subseteq 2^{[d]}$ of feature sets and the corresponding collection $\hypoinit$ of feature-subset models $\hIP$. Define the sub-collection $\setsep \coloneqq \{I \in \set: Y \perp_d \squarenode{P}{Q} \mid X_I \}$, which are $(P,Q)$-invariant feature sets in $\set$.
    Then, for all $(B,\Sigma,\mu)$ except a measure-zero set, it holds for the target risks of the candidate models:
    \begin{align*}
        \riskQ(\hIP) &= \riskP(\hIP) \quad &&\text{if } I \in \setsep; \\
        \riskQ(\hIP) &= \riskP(\hIP) + \Theta(t^2) \quad &&\text{otherwise.}
    \end{align*}
In particular, for $\ntarget \gtrsim \max \{ \frac{\log |\set| }{t^2}, C(B_X, \lambda_{\min},d) \}$, we have 
\begin{equation*}
    \riskQ(\hadapt) \leq \sup_{I \in \setsep} \riskP(\hIP).
\end{equation*}
\end{corollary}
We prove \Cref{cor:linear-scm} in \Cref{apx:proof-of-linear-scm}.  \Cref{cor:linear-scm} instantiates the setting of low- and high-risk models in \Cref{cor:stable-unstable-models} with $\tau = \sup_{I \in \setsep} \riskP(\hIP)$ and $\taugap = \Theta(t^2)$. 
In particular, we observe how large additive shifts separate a  collection $\set$ into "stable" and "unstable" models with either invariant or large risk. Target adaptation is then governed by the worst-case target risk among $(P,Q)$-invariant models in the collection $\set$. This illustrates the influence of increasing causal knowledge on the speed of adaptation: knowing both the graph and the shift set enables us to have as little as one $(P,Q)$-invariant model in the collection, resulting in small target risk in \Cref{cor:linear-scm}. In contrast, if the causal graph and/or the shift set are not known, sufficient invariance collections can include many more $(P,Q)$-invariant models, resulting in more pessimistic risk guarantees even under large additive shifts.

\section{Experiments}\label{sec:experiments}
In this section, we outline real-world experiments in which we utilize varying collections of candidate feature sets for supervised domain adaptation on new target domains.

\begin{figure}[t]
  \centering

  \begin{subfigure}[t]{0.48\textwidth}
    \centering
    \includegraphics[width=\linewidth]{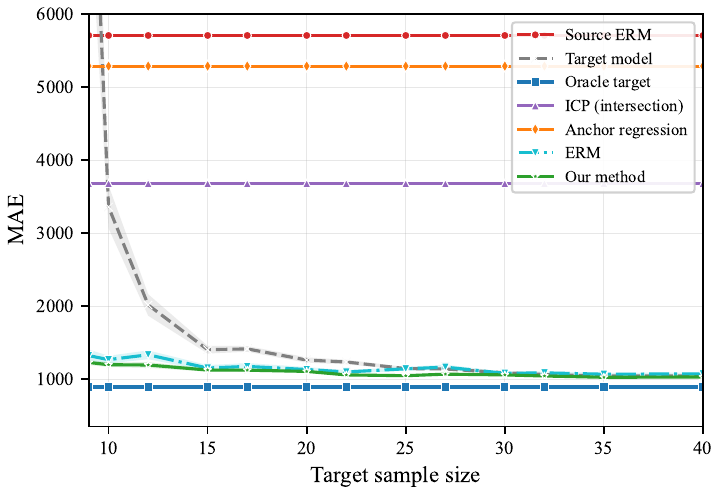}
    \caption{Causal Chambers experiment.}
    \label{fig:cc-exp}
  \end{subfigure}
  \hfill
  \begin{subfigure}[t]{0.48\textwidth}
    \centering
    \includegraphics[width=\linewidth]{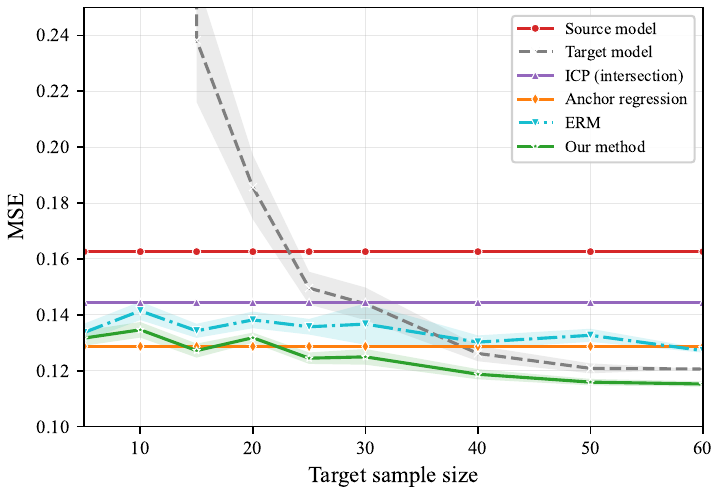}
    \caption{Gene expression experiment.}
    \label{fig:gene-exp}
  \end{subfigure}

  \caption{Target error (MAE and MSE, respectively) of the source and target model, causal DG methods, naive ERM and our procedure as a function of labeled target sample size $\ntarget$. Results are averaged over $50$ and $20$ runs, respectively. Shaded regions indicate standard error across runs.}
  \label{fig:main-experiments}
\end{figure}

\subsection{Causal Chamber data}\label{sec:exp-causal-chambers}
We consider the  \emph{light tunnel} collection of datasets from the Causal Chamber benchmark \cite{gamella2024causal}. The light tunnel dataset contains sensor recordings in a real-world physical system, and is comprised of one observational domain as well as multiple domains arising through interventions on concrete parts of the light tunnel. We take $3$ available domains as source domains and a previously unseen interventional domain as a target domain. We consider the collection of candidate models $\{\hempIP: I \in \invariantset \}$ satisfying invariances across multiple domains as described in \Cref{sec:collections}. We sample a few datapoints from the target domain and evaluate the target risk of some DG baselines (such as ICP intersection \cite{peters2016causal} and anchor regression \cite{rothenhausler2021anchor}), naive DA baselines, as well as our adaptive procedure. The results are shown in \Cref{fig:main-experiments} (a). Full details of the light tunnel experiment and implementation of our algorithm can be found in \Cref{apx:cc-linear-regression-details}. 
\subsection{Gene expression data}
We consider a subset of the RPE1 single-cell gene expression dataset \cite{replogle2022mapping}, which contains expression levels of the $10$ most highly expressed genes in the circuit. We consider a randomly selected gene as the target variable $Y$. As source domains, we pick the observational domain and $3$ randomly selected interventional domains. Since the considered gene subnetwork is a highly confounded system with many unobserved variables, we find that no candidate sets are accepted by ICP. Therefore, we resort to the collection $\set_{none}  = 2^{[d]}$ that corresponds to no causal knowledge of the shift set or the graph, and is trivially guaranteed to contain the most $Q$-predictive invariant model (cf. \Cref{eq:wantedproperty}).  We fit the corresponding $512$ linear models on source data. We evaluate multiple DG and DA baselines as well as our procedure given an increasing amount $\ntarget$ of available labeled target samples. Our results are shown in \Cref{fig:main-experiments} (b). Full details of the gene expression experiment can be found in \Cref{apx:gene-experiments}. 
\smallskip \\

\noindent Finally, in \Cref{apx:cc-linear-probes}, we explore a generalization of our causal domain adaptation setting beyond linear feature-subset models. Using the Causal Chamber image dataset \cite{gamella2024causal}, we construct a collection of candidate linear probes on pretrained CLIP embeddings. We describe the results and details of our image regression experiment in \Cref{apx:cc-linear-probes}.

\section{Outlook}{\label{sec:outlook}}
In this work, we studied linear supervised domain adaptation under partial knowledge of the shared causal structure. We showed 1) how access to a dictionary of potentially invariant models can improve supervised adaptation; 2) access to some target data can weaken the strict assumption-reliance of causality-based domain generalization methods. We put the notion of causal invariance in a finite-sample setting and demonstrated how the finite-sample effects of both source and target data cannot be ignored when utilizing causal invariance for generalization. 
We discuss promising future directions and potential extensions of our work in \Cref{sec:limitations}.

\section{Acknowledgments}
JK was supported by the SNF grant number 204439. We thank Tobias Wegel and Linus Kuehne for helpful discussions. This work is a result of collaboration  which started while JK, KJ, SK and FY  were visiting the Simons Institute for the Theory of Computing. 

\newpage
\bibliography{references}

@article{wang2026confidence,
  title={Confidence sets for causal orderings},
  author={Wang, Y Samuel and Kolar, Mladen and Drton, Mathias},
  journal={Journal of the American Statistical Association},
  volume={121},
  number={553},
  pages={690--703},
  year={2026},
  publisher={Taylor \& Francis}
}

@article{shimodaira2000improving,
  title={Improving predictive inference under covariate shift by weighting the log-likelihood function},
  author={Shimodaira, Hidetoshi},
  journal={Journal of Statistical Planning and Inference},
  volume={90},
  number={2},
  pages={227--244},
  year={2000},
  publisher={Elsevier}
}

@article{sugiyama2007covariate,
  title={Covariate shift adaptation by importance weighted cross validation.},
  author={Sugiyama, Masashi and Krauledat, Matthias and M{\"u}ller, Klaus-Robert},
  journal={Journal of Machine Learning Research},
  volume={8},
  number={5},
  year={2007}
}

@inproceedings{
nastl2024do,
title={Do causal predictors generalize better to new domains?},
author={Vivian Yvonne Nastl and Moritz Hardt},
booktitle={The Thirty-eighth Annual Conference on Neural Information Processing Systems},
year={2024},
}

@article{tropp2015introduction,
  title={An introduction to matrix concentration inequalities},
  author={Tropp, Joel A and others},
  journal={Foundations and Trends{\textregistered} in Machine Learning},
  volume={8},
  number={1-2},
  pages={1--230},
  year={2015},
  publisher={Now Publishers, Inc.}
}

@article{chen2021domain,
  title={Domain adaptation under structural causal models},
  author={Chen, Yuansi and B{\"u}hlmann, Peter},
  journal={Journal of Machine Learning Research},
  volume={22},
  number={261},
  pages={1--80},
  year={2021}
}

@article{christiansen2021causal,
  title={A causal framework for distribution generalization},
  author={Christiansen, Rune and Pfister, Niklas and Jakobsen, Martin Emil and Gnecco, Nicola and Peters, Jonas},
  journal={IEEE Transactions on Pattern Analysis and Machine Intelligence},
  volume={44},
  number={10},
  pages={6614--6630},
  year={2021},
  publisher={IEEE}
}

@article{wu2023prominent,
  title={Prominent Roles of Conditionally Invariant Components in Domain Adaptation: Theory and Algorithms},
  author={Wu, Keru and Chen, Yuansi and Ha, Wooseok and Yu, Bin},
  journal={arXiv preprint arXiv:2309.10301},
  year={2023}
}

@article{heinze2021conditional,
  title={Conditional variance penalties and domain shift robustness},
  author={Heinze-Deml, Christina and Meinshausen, Nicolai},
  journal={Machine Learning},
  volume={110},
  number={2},
  pages={303--348},
  year={2021},
  publisher={Springer}
}

@inproceedings{gong2016domain,
  title={Domain adaptation with conditional transferable components},
  author={Gong, Mingming and Zhang, Kun and Liu, Tongliang and Tao, Dacheng and Glymour, Clark and Sch{\"o}lkopf, Bernhard},
  booktitle={International Conference on Machine Learning},
  pages={2839--2848},
  year={2016},
  organization={PMLR}
}

@inproceedings{yang2018characterizing,
  title={Characterizing and learning equivalence classes of causal DAGs under interventions},
  author={Yang, Karren and Katcoff, Abigail and Uhler, Caroline},
  booktitle={International Conference on Machine Learning},
  pages={5541--5550},
  year={2018},
  organization={PMLR}
}

@book{pearl2009causality,
author = {Pearl, Judea},
title = {Causality: Models, Reasoning and Inference},
year = {2009},
isbn = {052189560X},
publisher = {Cambridge University Press},
address = {USA},
edition = {2nd}
}

@book{spirtes2001causation,
  title={Causation, prediction, and search},
  author={Spirtes, Peter and Glymour, Clark and Scheines, Richard},
  year={2001},
  publisher={MIT press}
}

@inproceedings{
gulrajani2020search,
title={In Search of Lost Domain Generalization},
author={Ishaan Gulrajani and David Lopez-Paz},
booktitle={International Conference on Learning Representations},
year={2021},
}

@inproceedings{muandet2013domain,
  title={Domain generalization via invariant feature representation},
  author={Muandet, Krikamol and Balduzzi, David and Sch{\"o}lkopf, Bernhard},
  booktitle={International conference on Machine Learning},
  pages={10--18},
  year={2013},
  organization={PMLR}
}

@article{heinze2018invariant,
  title={Invariant causal prediction for nonlinear models},
  author={Heinze-Deml, Christina and Peters, Jonas and Meinshausen, Nicolai},
  journal={Journal of Causal Inference},
  volume={6},
  number={2},
  pages={20170016},
  year={2018},
  publisher={De Gruyter}
}

@article{rothenhausler2021anchor,
  title={Anchor regression: Heterogeneous data meet causality},
  author={Rothenh{\"a}usler, Dominik and Meinshausen, Nicolai and B{\"u}hlmann, Peter and Peters, Jonas},
  journal={Journal of the Royal Statistical Society Series B: Statistical Methodology},
  volume={83},
  number={2},
  pages={215--246},
  year={2021},
  publisher={Oxford University Press}
}

@article{pfister2019invariant,
  title={Invariant causal prediction for sequential data},
  author={Pfister, Niklas and B{\"u}hlmann, Peter and Peters, Jonas},
  journal={Journal of the American Statistical Association},
  volume={114},
  number={527},
  pages={1264--1276},
  year={2019},
  publisher={Taylor \& Francis}
}

@article{peters2016causal,
  title={Causal inference by using invariant prediction: identification and confidence intervals},
  author={Peters, Jonas and B{\"u}hlmann, Peter and Meinshausen, Nicolai},
  journal={Journal of the Royal Statistical Society Series B: Statistical Methodology},
  volume={78},
  number={5},
  pages={947--1012},
  year={2016},
  publisher={Oxford University Press}
}

@article{redko2020survey,
  title={A survey on domain adaptation theory: learning bounds and theoretical guarantees},
  author={Redko, Ievgen and Morvant, Emilie and Habrard, Amaury and Sebban, Marc and Bennani, Youn{\`e}s},
  journal={arXiv preprint arXiv:2004.11829},
  year={2020}
}

@article{rojas2018invariant,
  author    = {Rojas-Carulla, Mateo and Sch{\"o}lkopf, Bernhard and Turner, Richard and Peters, Jonas},
  journal   = {The Journal of Machine Learning Research},
  number    = {1},
  pages     = {1309--1342},
  publisher = {JMLR. org},
  title     = {Invariant models for causal transfer learning},
  volume    = {19},
  year      = {2018}
}

@inproceedings{ahuja2020invariant,
  title={Invariant risk minimization games},
  author={Ahuja, Kartik and Shanmugam, Karthikeyan and Varshney, Kush and Dhurandhar, Amit},
  booktitle={International Conference on Machine Learning},
  pages={145--155},
  year={2020},
  organization={PMLR}
}

@inproceedings{krueger2021out,
  title={Out-of-distribution generalization via risk extrapolation ({RE}x)},
  author={Krueger, David and Caballero, Ethan and Jacobsen, Joern-Henrik and Zhang, Amy and Binas, Jonathan and Zhang, Dinghuai and Le Priol, Remi and Courville, Aaron},
  booktitle={International Conference on Machine Learning},
  pages={5815--5826},
  year={2021},
  organization={PMLR}
}

@inproceedings{
sagawa2019distributionally,
title={Distributionally Robust Neural Networks for Group Shifts: On the Importance of Regularization for Worst-Case Generalization},
author={Shiori Sagawa* and Pang Wei Koh* and Tatsunori B. Hashimoto and Percy Liang},
booktitle={International Conference on Learning Representations},
year={2020},
url={https://openreview.net/forum?id=ryxGuJrFvS}
}

@article{gnecco2023boosted,
  title={Boosted control functions},
  author={Gnecco, Nicola and Peters, Jonas and Engelke, Sebastian and Pfister, Niklas},
  journal={arXiv preprint arXiv:2310.05805},
  year={2023}
}

@article{buhlmann2020invariance,
  author    = {B{\"u}hlmann, Peter},
  journal   = {Statistical Science},
  number    = {3},
  pages     = {404--426},
  publisher = {Institute of Mathematical Statistics},
  title     = {Invariance, causality and robustness},
  volume    = {35},
  year      = {2020}
}

@article{jakobsen2022distributional,
  author    = {Jakobsen, Martin Emil and Peters, Jonas},
  journal   = {The Econometrics Journal},
  number    = {2},
  pages     = {404--432},
  publisher = {Oxford University Press},
  title     = {Distributional robustness of {K}-class estimators and the {PULSE}},
  volume    = {25},
  year      = {2022}
}

@article{kook2022distributionalanchor,
   title={Distributional anchor regression},
   volume={32},
   ISSN={1573-1375},
   number={3},
   journal={Statistics and Computing},
   publisher={Springer Science and Business Media LLC},
   author={Kook, Lucas and Sick, Beate and Bühlmann, Peter},
   year={2022},
   month=may }

@article{shen2023causalityoriented,
  title={Causality-oriented robustness: exploiting general additive interventions},
  author={Shen, Xinwei and B{\"u}hlmann, Peter and Taeb, Armeen},
  journal={arXiv preprint arXiv:2307.10299},
  year={2023}
}

@inproceedings{
vedantam2021an,
title={An Empirical Investigation of Domain Generalization with Empirical Risk Minimizers},
author={Shanmukha Ramakrishna Vedantam and David Lopez-Paz and David J. Schwab},
booktitle={Advances in Neural Information Processing Systems},
editor={A. Beygelzimer and Y. Dauphin and P. Liang and J. Wortman Vaughan},
year={2021},
}

@article{ben2010theory,
  title={A theory of learning from different domains},
  author={Ben-David, Shai and Blitzer, John and Crammer, Koby and Kulesza, Alex and Pereira, Fernando and Vaughan, Jennifer Wortman},
  journal={Machine learning},
  volume={79},
  pages={151--175},
  year={2010},
  publisher={Springer}
}

@inproceedings{lipton2018detecting,
  title={Detecting and correcting for label shift with black box predictors},
  author={Lipton, Zachary and Wang, Yu-Xiang and Smola, Alexander},
  booktitle={International conference on machine learning},
  pages={3122--3130},
  year={2018},
  organization={PMLR}
}

@article{garg2020unified,
  title={A unified view of label shift estimation},
  author={Garg, Saurabh and Wu, Yifan and Balakrishnan, Sivaraman and Lipton, Zachary},
  journal={Advances in Neural Information Processing Systems},
  year={2020}
}

@article{replogle2022mapping,
  title={Mapping information-rich genotype-phenotype landscapes with genome-scale {P}erturb-seq},
  author={Replogle, Joseph M and Saunders, Reuben A and Pogson, Angela N and Hussmann, Jeffrey A and Lenail, Alexander and Guna, Alina and Mascibroda, Lauren and Wagner, Eric J and Adelman, Karen and Lithwick-Yanai, Gila and others},
  journal={Cell},
  volume={185},
  number={14},
  pages={2559--2575},
  year={2022},
  publisher={Elsevier}
}

@article{gu2024causality,
  title={Causality pursuit from heterogeneous environments via neural adversarial invariance learning},
  author={Gu, Yihong and Fang, Cong and B{\"u}hlmann, Peter and Fan, Jianqing},
  journal={arXiv preprint arXiv:2405.04715},
  year={2024}
}

@inproceedings{jalaldoust2024partial,
  title={Partial transportability for domain generalization},
  author={Jalaldoust, Kasra and Bellot, Alexis and Bareinboim, Elias},
  booktitle={The Thirty-eighth Annual Conference on Neural Information Processing Systems},
  year={2024}
}

@article{kostin2024achievable,
  title={Achievable distributional robustness when the robust risk is only partially identified},
  author={Kostin, Julia and Gnecco, Nicola and Yang, Fanny},
  journal={Advances in Neural Information Processing Systems},
  volume={37},
  pages={83915--83950},
  year={2024}
}

@book{spirtes2000causation,
  title={Causation, prediction, and search},
  author={Spirtes, Peter and Glymour, Clark N and Scheines, Richard},
  year={2000},
  publisher={MIT press}
}

@article{hauser2012characterization,
  title={Characterization and greedy learning of interventional Markov equivalence classes of directed acyclic graphs},
  author={Hauser, Alain and B{\"u}hlmann, Peter},
  journal={The Journal of Machine Learning Research},
  volume={13},
  number={1},
  pages={2409--2464},
  year={2012},
  publisher={JMLR. org}
}

@article{zhang2008completeness,
  title={On the completeness of orientation rules for causal discovery in the presence of latent confounders and selection bias},
  author={Zhang, Jiji},
  journal={Artificial Intelligence},
  volume={172},
  number={16-17},
  pages={1873--1896},
  year={2008},
  publisher={Elsevier}
}

@article{li2023causal,
  title={Causal discovery from observational and interventional data across multiple environments},
  author={Li, Adam and Jaber, Amin and Bareinboim, Elias},
  journal={Advances in Neural Information Processing Systems},
  volume={36},
  pages={16942--16956},
  year={2023}
}

@article{perkovic2018complete,
  title={Complete graphical characterization and construction of adjustment sets in Markov equivalence classes of ancestral graphs},
  author={Perkovi{\'c}, Emilija and Textor, Johannes and Kalisch, Markus and Maathuis, Marloes H},
  journal={Journal of Machine Learning Research},
  volume={18},
  number={220},
  pages={1--62},
  year={2018}
}

@inproceedings{ben-david-2006,
 author = {Ben-David, Shai and Blitzer, John and Crammer, Koby and Pereira, Fernando},
 booktitle = {Advances in Neural Information Processing Systems},
 editor = {B. Sch\"{o}lkopf and J. Platt and T. Hoffman},
 pages = {},
 publisher = {MIT Press},
 title = {Analysis of Representations for Domain Adaptation},
   OPTurl = {https://proceedings.neurips.cc/paper_files/paper/2006/file/b1b0432ceafb0ce714426e9114852ac7-Paper.pdf},
 volume = {19},
 year = {2006}
}

@inproceedings{bendavid2003,
  title={Exploiting Task Relatedness for Multiple Task Learning},
  author={Ben-David, Shai and Schuller, Reba},
  booktitle={Proceedings of the 16th Annual Conference on Learning Theory (COLT)},
  pages={567--580},
  year={2003}
}

@inproceedings{blitzer2006,
  title={Domain Adaptation with Structural Correspondence Learning},
  author={Blitzer, John and McDonald, Ryan and Pereira, Fernando},
  booktitle={Proceedings of the 2006 Conference on Empirical Methods in Natural Language Processing},
  pages={120--128},
  year={2006}
}

@inproceedings{mansour2008,
 author = {Mansour, Yishay and Mohri, Mehryar and Rostamizadeh, Afshin},
 booktitle = {Advances in Neural Information Processing Systems},
 pages = {},
 publisher = {Curran Associates, Inc.},
 title = {Domain Adaptation with Multiple Sources},
 volume = {21},
 year = {2008}
}

@misc{hanneke2025adaptivesampleaggregationtransfer,
      title={Adaptive Sample Aggregation In Transfer Learning}, 
      author={Steve Hanneke and Samory Kpotufe},
      year={2025},
      eprint={2408.16189},
      archivePrefix={arXiv},
      primaryClass={stat.ML},
      url={https://arxiv.org/abs/2408.16189}, 
}

@inbook{valueoftargetdata,
author = {Hanneke, Steve and Kpotufe, Samory},
title = {On the value of target data in transfer learning},
year = {2019},
publisher = {Curran Associates Inc.},
address = {Red Hook, NY, USA},
booktitle = {Proceedings of the 33rd International Conference on Neural Information Processing Systems},
articleno = {885},
numpages = {11}
}

@article{Jalaldoust_Bareinboim_2024, title={Transportable Representations for Domain Generalization}, author={Jalaldoust, Kasra and Bareinboim, Elias}, volume={38}, number={11}, journal={Proceedings of the AAAI Conference on Artificial Intelligence}, year={2024}, month={Mar.}, pages={12790-12800} }

@article{glennan1996mechanisms,
  title={Mechanisms and the nature of causation},
  author={Glennan, Stuart S},
  journal={Erkenntnis},
  volume={44},
  number={1},
  pages={49--71},
  year={1996},
  publisher={Springer}
}

@article{machamer2000thinking,
  title={Thinking about mechanisms},
  author={Machamer, Peter and Darden, Lindley and Craver, Carl F},
  journal={Philosophy of science},
  volume={67},
  number={1},
  pages={1--25},
  year={2000},
  publisher={Cambridge University Press}
}

@article{arjovsky2019invariant,
  title={Invariant risk minimization},
  author={Arjovsky, Martin and Bottou, L{\'e}on and Gulrajani, Ishaan and Lopez-Paz, David},
  journal={arXiv preprint arXiv:1907.02893},
  year={2019}
}

@inproceedings{li2018domain,
  title={Domain generalization via conditional invariant representations},
  author={Li, Ya and Gong, Mingming and Tian, Xinmei and Liu, Tongliang and Tao, Dacheng},
  booktitle={Proceedings of the AAAI conference on artificial intelligence},
  volume={32},
  year={2018}
}

@article{wang2022generalizing,
  title={Generalizing to unseen domains: A survey on domain generalization},
  author={Wang, Jindong and Lan, Cuiling and Liu, Chang and Ouyang, Yidong and Qin, Tao and Lu, Wang and Chen, Yiqiang and Zeng, Wenjun and Yu, Philip},
  journal={IEEE Transactions on Knowledge and Data Engineering},
  year={2022},
  publisher={IEEE}
}

@inproceedings{
wald2021on,
title={On Calibration and Out-of-Domain Generalization},
author={Yoav Wald and Amir Feder and Daniel Greenfeld and Uri Shalit},
booktitle={Advances in Neural Information Processing Systems},
editor={A. Beygelzimer and Y. Dauphin and P. Liang and J. Wortman Vaughan},
year={2021},
    OPTurl={https://openreview.net/forum?id=XWYJ25-yTRS}
}

@article{storkey2008training,
  title={When training and test sets are different: characterizing learning transfer},
  author={Storkey, Amos},
  year={2008}
}

@article{bareinboim2014transportability,
  title={Transportability from multiple environments with limited experiments: Completeness results},
  author={Bareinboim, Elias and Pearl, Judea},
  journal={Advances in neural information processing systems},
  volume={27},
  year={2014}
}

@inproceedings{hsu2012random,
  title={Random design analysis of ridge regression},
  author={Hsu, Daniel and Kakade, Sham M and Zhang, Tong},
  booktitle={Conference on learning theory},
  pages={9--1},
  year={2012},
  organization={JMLR Workshop and Conference Proceedings}
}

@article{correa2020general,
  title={General transportability of soft interventions: Completeness results},
  author={Correa, Juan and Bareinboim, Elias},
  journal={Advances in Neural Information Processing Systems},
  volume={33},
  pages={10902--10912},
  year={2020}
}

@inproceedings{correa2019statistical,
  title={From Statistical Transportability to Estimating the Effect of Stochastic Interventions.},
  author={Correa, Juan D and Bareinboim, Elias},
  booktitle={IJCAI},
  pages={1661--1667},
  year={2019}
}

@inproceedings{blitzer2011da,
  title={Domain Adaptation with Coupled Subspaces},
  author={Blitzer, John and Kakade, Sham and Foster, Dean},
  booktitle={Proceedings of the Fourteenth International Conference on Artificial Intelligence and Statistics},
  series={Proceedings of Machine Learning Research},
  volume={15},
  pages={173--181},
  year={2011},
  publisher={PMLR}
}

@article{huang2020,
  author       = {Huang, Biwei and Zhang, Kun and Sch\"{o}lkopf, Bernhard},
  title        = {Invariant Causal Prediction for Nonlinear Models},
  journal      = {Journal of Causal Inference},
  volume       = {8},
  number       = {1}, 
  pages        = {350--367},
  year         = {2020},
  publisher    = {De Gruyter}
}

@inproceedings{johansson2019support,
  title={Support and invertibility in domain-invariant representations},
  author={Johansson, Fredrik D and Sontag, David and Ranganath, Rajesh},
  booktitle={The 22nd International Conference on Artificial Intelligence and Statistics},
  pages={527--536},
  year={2019},
  organization={PMLR}
}

@inproceedings{zhao2019learning,
  title={On learning invariant representations for domain adaptation},
  author={Zhao, Han and Des Combes, Remi Tachet and Zhang, Kun and Gordon, Geoffrey},
  booktitle={International conference on machine learning},
  pages={7523--7532},
  year={2019},
  organization={PMLR}
}

@INPROCEEDINGS{pearlbareinboim2011,
  author={Pearl, Judea and Bareinboim, Elias},
  booktitle={2011 IEEE 11th International Conference on Data Mining Workshops}, 
  title={Transportability of Causal and Statistical Relations: A Formal Approach}, 
  year={2011},
  volume={},
  number={},
  pages={540-547},
  doi={10.1109/ICDMW.2011.169}}

@article{gamella2024causal,
  author={Gamella, Juan L. and Peters, Jonas and B{\"u}hlmann, Peter},
  title={Causal chambers as a real-world physical testbed for {AI} methodology},
  journal={Nature Machine Intelligence},
  doi={10.1038/s42256-024-00964-x},
  year={2025},
}

@inproceedings{cortes2010,
 author = {Cortes, Corinna and Mansour, Yishay and Mohri, Mehryar},
 booktitle = {Advances in Neural Information Processing Systems},
 editor = {J. Lafferty and C. Williams and J. Shawe-Taylor and R. Zemel and A. Culotta},
 pages = {},
 publisher = {Curran Associates, Inc.},
 title = {Learning Bounds for Importance Weighting},
 url = {https://proceedings.neurips.cc/paper_files/paper/2010/file/59c33016884a62116be975a9bb8257e3-Paper.pdf},
 volume = {23},
 year = {2010}
}

@InProceedings{chattopadhyay13,
  title = 	 {Joint Transfer and Batch-mode Active Learning},
  author = 	 {Chattopadhyay, Rita and Fan, Wei and Davidson, Ian and Panchanathan, Sethuraman and Ye, Jieping},
  booktitle = 	 {Proceedings of the 30th International Conference on Machine Learning},
  pages = 	 {253--261},
  year = 	 {2013},
  editor = 	 {Dasgupta, Sanjoy and McAllester, David},
  volume = 	 {28},
  number =       {3},
  series = 	 {Proceedings of Machine Learning Research},
  address = 	 {Atlanta, Georgia, USA},
  month = 	 {17--19 Jun},
  publisher =    {PMLR},
  url = 	 {https://proceedings.mlr.press/v28/chattopadhyay13.html},
}

@article{hanneke2013,
author = {Yang, Liu and Hanneke, Steve and Carbonell, Jaime},
year = {2013},
month = {02},
pages = {},
title = {A Theory of Transfer Learning with Applications to Active Learning},
volume = {90},
journal = {Machine Learning},
doi = {10.1007/s10994-012-5310-y}
}

@InProceedings{saha2011,
author="Saha, Avishek
and Rai, Piyush
and Daum{\'e}, Hal
and Venkatasubramanian, Suresh
and DuVall, Scott L.",
editor="Gunopulos, Dimitrios
and Hofmann, Thomas
and Malerba, Donato
and Vazirgiannis, Michalis",
title="Active Supervised Domain Adaptation",
booktitle="Machine Learning and Knowledge Discovery in Databases",
year="2011",
publisher="Springer Berlin Heidelberg",
address="Berlin, Heidelberg",
pages="97--112",
isbn="978-3-642-23808-6"
}

@inproceedings{david2010impossibility,
  title={Impossibility theorems for domain adaptation},
  author={Ben-David, Shai and Lu, Tyler and Luu, Teresa and P{\'a}l, D{\'a}vid},
  booktitle={Proceedings of the Thirteenth International Conference on Artificial Intelligence and Statistics},
  pages={129--136},
  year={2010},
  organization={JMLR Workshop and Conference Proceedings}
}

@InProceedings{david2012hardness,
author="Ben-David, Shai
and Urner, Ruth",
editor="Bshouty, Nader H.
and Stoltz, Gilles
and Vayatis, Nicolas
and Zeugmann, Thomas",
title="On the Hardness of Domain Adaptation and the Utility of Unlabeled Target Samples",
booktitle="Algorithmic Learning Theory",
year="2012",
publisher="Springer Berlin Heidelberg",
address="Berlin, Heidelberg",
pages="139--153",
isbn="978-3-642-34106-9"
}

@article{ben2014domain,
  title={Domain adaptation--can quantity compensate for quality?},
  author={Ben-David, Shai and Urner, Ruth},
  journal={Annals of Mathematics and Artificial Intelligence},
  volume={70},
  number={3},
  pages={185--202},
  year={2014},
  publisher={Springer}
}

@article{londschien2025domain,
  title={Domain Generalization and Adaptation in Intensive Care with Anchor Regression},
  author={Londschien, Malte and Burger, Manuel and R{\"a}tsch, Gunnar and B{\"u}hlmann, Peter},
  journal={arXiv preprint arXiv:2507.21783},
  year={2025}
}

@inproceedings{tsybakov2003optimal,
  title={Optimal rates of aggregation},
  author={Tsybakov, Alexandre B},
  booktitle={Learning Theory and Kernel Machines: 16th Annual Conference on Learning Theory and 7th Kernel Workshop, COLT/Kernel 2003, Washington, DC, USA, August 24-27, 2003. Proceedings},
  pages={303--313},
  year={2003},
  organization={Springer}
}

@article{rigollet2012kullback,
  title={Kullback-Leibler aggregation and misspecified generalized linear models},
  author={Rigollet, Philippe},
  journal={The Annals of Statistics},
  pages={639--665},
  year={2012},
  publisher={JSTOR}
}

@article{audibert2007progressive,
  title={Progressive mixture rules are deviation suboptimal},
  author={Audibert, Jean-Yves},
  journal={Advances in Neural Information Processing Systems},
  volume={20},
  year={2007}
}

@article{yang1999model,
  title={Model selection for nonparametric regression},
  author={Yang, Yuhong},
  journal={Statistica Sinica},
  pages={475--499},
  year={1999},
  publisher={JSTOR}
}

@article{juditsky2008learning,
  title={Learning by mirror averaging},
  author={Juditsky, Anatoli and Rigollet, Philippe and Tsybakov, Alexandre B},
  year={2008}
}

@article{dai2012deviation,
  title={Deviation optimal learning using greedy Q-aggregation},
  author={Dai, Dong and Rigollet, Philippe and Zhang, Tong},
  year={2012}
}

@article{gaiffas2011hyper,
  title={Hyper-sparse optimal aggregation},
  author={Ga{\^\i}ffas, St{\'e}phane and Lecu{\'e}, Guillaume},
  journal={The Journal of Machine Learning Research},
  volume={12},
  pages={1813--1833},
  year={2011},
  publisher={JMLR. org}
}

@article{pfister2021stabilizing,
  title={Stabilizing variable selection and regression},
  author={Pfister, Niklas and Williams, Evan G and Peters, Jonas and Aebersold, Ruedi and B{\"u}hlmann, Peter},
  journal={The Annals of Applied Statistics},
  volume={15},
  number={3},
  pages={1220--1246},
  year={2021},
  publisher={JSTOR}
}

@article{lecue2014optimal,
  title={Optimal learning with Q-aggregation},
  author={Lecu{\'e}, Guillaume and Rigollet, Philippe},
  year={2014}
}

@article{audibert2007proof,
  title={Proof of the optimality of the empirical star algorithm},
  author={Audibert, Jean-Yves},
  journal={Technical note},
  year={2007}
}
\bibliographystyle{unsrt}

\clearpage
\appendix
\clearpage 

\hypersetup{
  linkcolor={pierCite},
  citecolor={pierCite},
  urlcolor={pierCite}
}

\section*{Appendix}
The following sections provide deferred discussions, proofs and experimental details.
\DoToC
\clearpage

\hypersetup{
  linkcolor={pierLink},
  citecolor={pierCite},
  urlcolor={pierCite}
}

\crefalias{section}{appendix}      
\crefalias{subsection}{subappendix}

\clearpage
\section{An Idealistic Construction of Feature subset Collections}\label{sec:optimistic-collection}

Suppose we have (partial) knowledge of the underlying SCMs over $(X,Y)$ in source $P$ and target $Q$ -- for instance, a causal graph, and some partial knowledge of the structural shift between $P$ and $Q$ (for instance, the shift set $\shiftset{P}{Q}$). This knowledge specifies a family of distributions $\mathcal{Q}$ containing all admissible target distributions $Q$. Recall that throughout this work we consider no structural shift on the label, i.e., $Y \not\in \shiftset{P}{Q}$. We will consider the following \emph{degrees of causal knowledge} (cf. \Cref{fig:main-causal-figure} (b)):
\begin{align*}
  \Qnone &\coloneqq \{Q: \text{ SCM } \SCM_Q \text{ differs from } \SCM_P \text{ in unknown } \shiftset{P}{Q} \not\ni Y \}  \quad \text{no causal knowledge}; \\
  \Qgraphs &\coloneqq \{Q: \text{ SCM $\SCM_Q$ induces } \graph \in \graphs, \text{ unknown } \shiftset{P}{Q} \not\ni Y\}  \quad \text{partial knowledge of graph}; \\
  \Qgraph &\coloneqq \{Q: \text{ SCM $\SCM_Q$ induces } \graph, \text{ unknown } \shiftset{P}{Q} \not\ni Y\}  \quad \text{full knowledge of the graph}; \\
  \Qgraphshift &\coloneqq \{Q: \text{ SCM $\SCM_Q$ induces } \graph, \text{ known } \shiftset{P}{Q} \not\ni Y\}  \quad \text{full causal knowledge}.
\end{align*}

In our example in \Cref{sec:toy-example}, we have observed that even under full causal knowledge $\Qgraphshift$, we could only identify a collection of \emph{candidate} best models $\{h_{\{1,2,4\},P}, \hP \}$ of which either could be preferred for a concrete $Q \in \Qgraphshift$.  Motivated by this observation, for a family $\mathcal{Q}$ of admissible target distributions, we might wish to  construct what we will call an  "\emph{optimistic}" collection:
\begin{equation*}
    \optset \text{ optimistic w.r.t. $\mathcal{Q}$ if } \quad \forall Q \in \mathcal{Q}: \argmin_{I \subseteq [d]} \riskQ(\hIP) \cap \optset \neq \emptyset. 
\end{equation*}
Notice that the sufficient invariance collection (\Cref{def:sufficient-invariance-collection}) are desirable specifically because they are optimistic collection. In words, $\optset$ is guaranteed to contain some best-performing feature subset model for any admissible target domain $Q$. In particular, the best model from such an optimistic collection is guaranteed to do \emph{at least as well as } any $(P,Q)$-invariant model and the full-feature source model.
An obvious collection that is optimistic w.r.t. any family of target distributions $\mathcal{Q}$ is $2^{[d]}$, containing all of the feature subsets. However, using the available causal knowledge, we might want to construct a smaller, non-trivial collection which is still optimistic w.r.t. $\mathcal{Q}$. 
For instance, in the example in \Cref{sec:toy-example}, if we only know $Y \not\in \shiftset{P}{Q}$, i.e., $Q \in \Qnone$, then the full collection $\optset = 2^{[4]}$ is the only option. A more refined knowledge is having a subset of the parents, i.e., $\graph \in \graphs = \{\graph: \text{ s.t } X_1 \in \Pa(Y) \}$; for every $Q \in \Qgraphs$ at least one of the best performing feature subsets must include $X_1$. If we further know the graph $G$ in \Cref{fig:main-causal-figure} (without the $s$-node), then we can deduce that $X_1,X_2$ belong to at least one of the best performing feature subsets for every $Q \in \Qgraph$. If we additionally know $\shiftset{P}{Q} = \{X_1,X_2,X_3\}$, then we can make a stronger claim that that every feature subset of every optimistic set contains $X_1,X_2,X_4$. Next, we propose a procedure to iteratively remove certain feature subsets from any optimistic collection while preserving it as an optimistic collection.\\

\begin{proposition} \label{prop:optsetmin-membership}
    Suppose we have full causal knowledge, i.e.,  $Q \in \Qgraphshift$. Also, assume causal faithfulness. For a subset $I \subseteq [d]$, consider the following: There exist disjoint subsets $I',I''\subseteq [d]$ where $I'\cup I'' \neq \emptyset$, and,
    \begin{enumerate}
    \setlength{\itemsep}{1pt}
        \item $Y$ is d-separated from $X_{I'}$ given $X_{I}$ in $\graph_{\overline{X_{I'}}}$.
        \item $X_{I''}$ is d-separated from $\shiftset{P}{Q}$ given $X_{I\cup I'},Y$;
        \item $X_{I''}$ is not d-separated from $Y$ given $X_{I\cup I'}$.
    \end{enumerate}
    If the feature set $I$ satisfies (1,2,3), then for every optimistic collection $\optset$, the collection $\optset \setminus \{I\}$ is also optimistic w.r.t. $\Qgraphshift$.
\end{proposition}

We note that this result is not a \emph{complete} characterization of \emph{minimal} optimistic collections; by iterative application of the rule one can arrive at an  irreducible optimistic collection, and still, that collection might contain a strict optimistic collection. Thus, an idealistic goal is to construct a minimal optimistic collection,
\begin{equation*}
    \optsetmin \text{ minimal optimistic w.r.t. }\mathcal{Q} \text{ if }\quad  \optsetmin  \in \argmin_{\optset \text{ optimistic w.r.t. } \mathcal{Q}} | \optset |.
\end{equation*}

A parametric search over the space of SCMs may allow us to verify whether a given collection $\optset$ is minimal and optimistic w.r.t.\ $\Qgraphshift$ (i.e., given the full causal knowledge): One may consider the parametric SCM pairs $(\SCM_P,\SCM_Q)$, and constrain them to generate the source and target distributions $P,Q$ while inducing the selection diagram $\graph,\shiftset{P}{Q}$, and then verify if every feature subset $I \in \optset$ is the most $Q$-predictive $(P,Q)$-invariant model for \emph{some} plausible pair $(\SCM_P,\SCM_Q)$. For discrete variables, there exists such a canonical parameterization of SCMs, cf., \cite{jalaldoust2024partial}, but a canonical parametric form for real-valued SCMs remains an open research question. With additional assumptions it may be possible to find a minimal optimistic collection $\optsetmin$: For instance, in linear structural models under additive shifts, deciding whether a set $I$ belongs to $\optsetmin$ w.r.t. $\Qgraphshift$ reduces to a quadratic feasibility problem in the shift parameter, and finding $\optsetmin$ corresponds to solving $2^d$ such feasibility problems.

Beyond computational intractability, the next result outlines another issue with the goal of finding minimal optimistic collections $\optsetmin$ for this linear setting: \\

\begin{proposition}\label{prop:impossibility-for-optimistic}
    \begin{enumerate}
        \item There exists a distribution $P$ induced by a linear Gaussian SCM $\SCM_P$ on $(X,Y)$, and a shift set $\shiftset{P}{Q}$, for which, given full causal knowledge $\Qgraphshift$, we have for any $\optset$: $|\optset| = 2^d$.
        \item There exists a distribution $P$ induced by a linear Gaussian SCM $\SCM_P$  with $|\Pa(Y)| = d-1$ and $|\shiftset{P}{Q}| = 1$, for which, given  $\Qgraphshift$, we have for any $\optset$: $|\optset| \geq 2^{d-1} +1$.
    \end{enumerate}
\end{proposition}
We prove \Cref{prop:impossibility-for-optimistic} in \Cref{sec:proof-of-impossibility-for-optimistic}. In particular, our result implies that even in benign settings with full causal knowledge where most features are causal, the mechanism shift is extremely sparse and known, $\optset$ generally cannot be reduced beyond exponential size.



\section{Future directions}\label{sec:limitations}
Our work makes the first step towards establishing scenarios under which  (partial) causal invariances enable fast supervised domain adaptation. Our setting already accounts for finite-sample effects in both source and target data, as well as model misspecification, both of which generalize the typical population-level assumptions in causal domain generalization. However, our theoretical results are, so far, limited to the linear regression setting and rely on further regularity assumptions such as boundedness of the noise and the function class. Below, we outline the main future directions of our work as well as potential challenges which might arise while extending our results.

\begin{enumerate}
   \item \textbf{Extension beyond the linear regression setting.}
Our analysis relies on linearity in multiple places. First, Step 1 of \Cref{alg:iterative-step-agg} makes use of so-called \emph{weak confidence sets} first introduced in \cite{hanneke2025adaptivesampleaggregationtransfer}. For a given hypothesis class $\mathcal{H}$ and any source model $\hempIP$,  guarantees of style $\min \{ \excessQ(\hempQ), \excessQ(\hempIP)\}$ can be easily given if one can compute the weak confidence sets $\hat{\hypo}_P$, $\hat{\hypo}_Q$ and output an element of their intersection. \cite{hanneke2025adaptivesampleaggregationtransfer} compute weak confidence sets for classification. However, weak confidence sets for linear regression enjoy convenient geometric properties and are easy to compute, both theoretically and empirically. An exciting direction for future work is thus to compute such confidence sets for other parametric classes, enabling guarantees for Step 1 of \Cref{alg:iterative-step-agg}. 
\item \textbf{Extension beyond bounded noise. } The boundedness assumption \Cref{asm:tail-of-X-and-Y} is required to establish high-probability guarantees for model selection aggregation (cf. \Cref{apx:model-aggregation}). \cite{gaiffas2011hyper} provide extended guarantees for model selection aggregation for sub-exponential and sub-Gaussian noise. The resulting rate $\erroragg = \frac{(\log(|\set|) \log(1/\alpha) \log n}{n}$ has an additional multiplicative $\log n$ dependency. Our results immediately transfer to the setting of bounded linear function class and sub-Gaussian noise via this guarantee, however, it is unclear whether guarantees without an additional multiplicative $\log n$ dependency are possible for unbounded, but "well-behaved" variables. In particular, it remains an important open question to establish optimal high-probability guarantees for model selection aggregation under fully sub-Gaussian designs. 
\item \textbf{Other adaptive procedures. } Although \Cref{alg:iterative-step-agg} achieves a fast adaptation guarantee, it is an interesting open question whether the same guarantee can be established for a procedure based on an end-to-end loss function. Further, \Cref{alg:iterative-step-agg} requires as input two constants $C_1,C_2$ which have to be fixed a priori or estimated from (limited) target data. It remains an important open question, especially for practitioners, whether a hyperparameter-free adaptation method for supervised adaptation under partial causal knowledge could be developed, and whether such a method could match our fast convergence guarantees theoretically.

    \item \textbf{Assumption of partially known causal mechanisms.}  
    Our assumptions go beyond the common assumption of full identifiability in causality-oriented robustness frameworks \citep{peters2016causal,heinze2018invariant,pfister2019invariant}. Unlike prior work that either assumes identifiability of the invariant representation or restricts distribution shift directions (e.g., \citep{rothenhausler2021anchor,gnecco2023boosted}), we do not impose such restrictions. 
    In \Cref{sec:new-causal}, we discuss how sets of candidate invariant models can be obtained from structure learning or multi-environment approaches.
    However, in practice, especially under confounding or model misspecification, these methods may fail to recover the correct set $\mathcal{I}$ of plausible invariant models. The performance of our adaptive procedure may thus degrade if this set is poorly estimated. This highlights the need for further research on \emph{partial-identifiability-aware} methods that output multiple candidate invariant models from multi-environment data, only committing to a single choice when full identifiability is given.
     \item \textbf{Extension to possibly invariant representations.}  
    A broad line of work (e.g., \citep{arjovsky2019invariant,ahuja2020invariant,wu2023prominent}) aims to learn a \emph{representation} $\phi(X)$ of the covariates such that the conditional distribution $P(Y \mid \phi(X))$, or the corresponding prediction risk, is invariant across environments. This setting generalizes ours, since restriction onto a subset of features is a special case of a representation map $\phi$. It is well known that in misspecified scenarios, such as insufficient distributional variability across training domains, there may exist multiple local minima of the invariance-penalized loss, potentially leading the method to select an incorrect, non-invariant representation. Due to the non-convex nature of these objectives, developing methods that output multiple invariant representations and allow interpolation between them using limited target data remains a challenging and important direction for future work.

\end{enumerate}
\clearpage
\section{Additional details on the real-world experiments}\label{apx:causal-chambers}
\subsection{The Causal Chamber dataset}

The Causal Chambers \cite{gamella2024causal} are physical testbeds designed to generate large-scale real-world datasets with known causal structure and controlled interventions. One of the chambers is a \emph{light tunnel}, which consists of a controllable RGB light source, two linear polarizers and multiple light sensors positioned before, between and after the polarizers. In the following, we describe two experiments on light tunnel data: 1) a linear regression experiment with light tunnel sensors as variables; 2) an image regression experiment in which the candidate models are varying linear probes based on a fixed CLIP representation. 
\subsubsection{Causal Chambers: linear regression experiment}\label{apx:cc-linear-regression-details}
\textbf{Data.} For the first experiment, we consider the $\mathrm{lt\_interventions\_standard\_v1}$ dataset consisting of sensor measurements under interventions. 
The light tunnel setup provides measurements of $41$ variables, out of which $32$ can be intervened on. Since the tunnel as a physical system is well-understood, the causal ground truth is available (cf. \Cref{fig:causal-chamber-details}). As demonstrated in \cite{gamella2024causal} (Task b1), the light tunnel can be used as a benchmark for out-of-distribution (OOD) generalization, since one has access to both observational data and data from various interventions. In the following, we follow the setting of Task b1 with some light modifications. We consider a subgraph of the light tunnel depicted in \Cref{fig:causal-chamber-details}. The subgraph contains the most important sensor variables. The omitted variables act as hidden confounders (depicted as bidirected dashed edges).

\begin{figure}[t]
    \centering
        \includegraphics[width=0.4\linewidth]{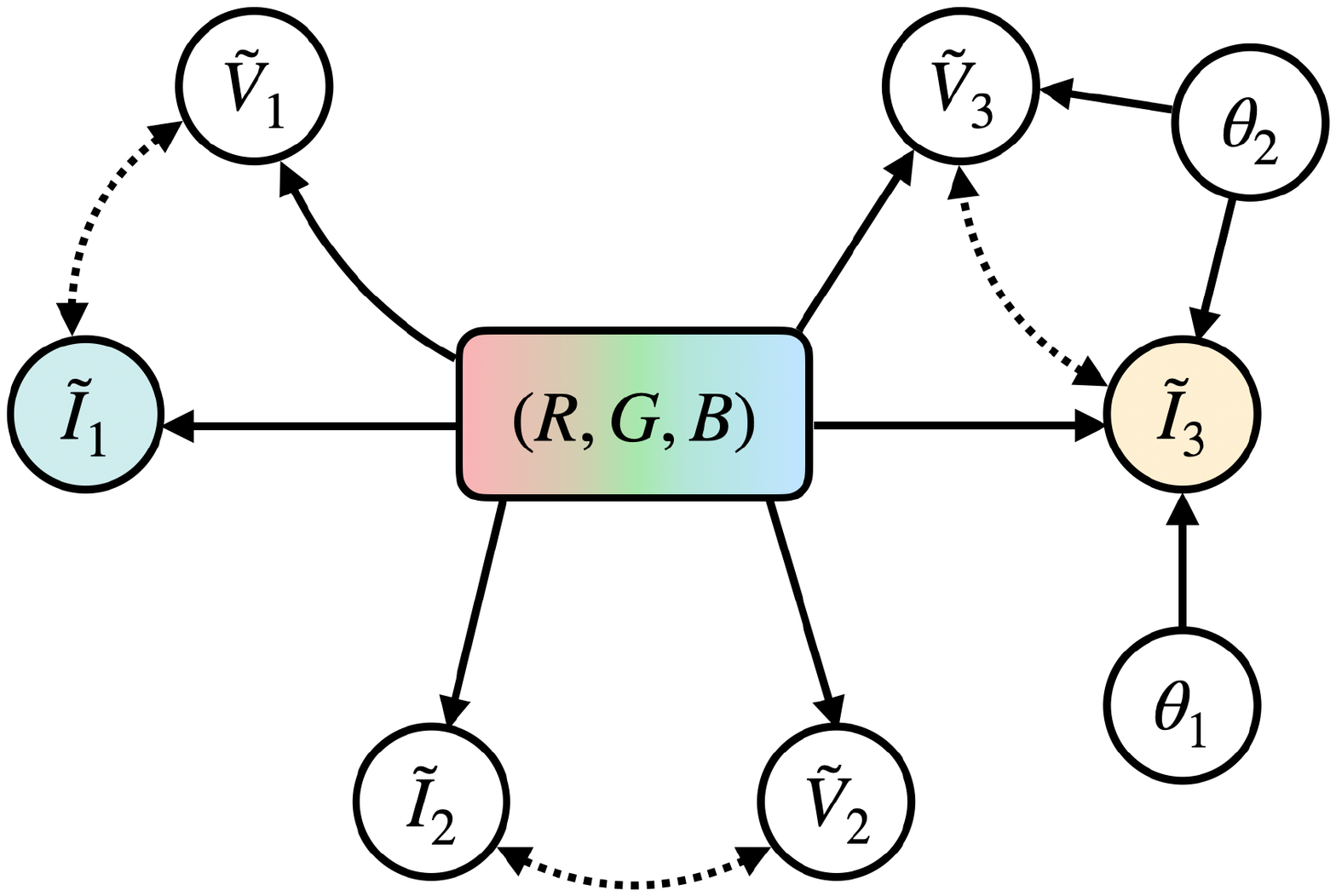}

    \caption{The causal graph of the light tunnel variables used in our experiment. }
    \label{fig:causal-chamber-details}
\end{figure}

\textbf{Multi-domain setup.} In our linear regression experiment, we consider $\tilde{I}_1$ as the target variable and exclude the polarizer angles $\theta_{1,2}$, resulting in $9$ variables in total. As multi-environment data for (partial) causal discovery, we take the reference (non-intervened) environment as well as two domains generated by interventions on unobserved variables $\theta_1$ and $T^I_2$, resulting in 3 domains in total with $10000$ samples each. We denote this collection as $\mathcal{E}_{\mathrm{train}}$. 

\textbf{Candidate sets.} We run invariant causal prediction (ICP) \cite{peters2016causal} on $\mathcal{E}_{\mathrm{train}}$. ICP is a hypothesis-testing based method which iterates over all subsets of features and for each subset $I$, tests the null hypothesis $H_{0,I}(\mathcal{E}_{\mathrm{train}})$: "there exists a linear relationship between $X_I$ and $Y$ with residuals invariant across all environments". We denote with $\invariantset$ the collection of all feature sets $I$ for which $H_{0,I}(\mathcal{E}_{\mathrm{train}})$ could not be rejected. This collection corresponds to the collection discussed in \Cref{sec:multiple-environments}. We compute the resulting collection of candidate models $\hypoinit = \{\hempIP: I \in \invariantset \}$.  

\textbf{(Domain generalization) baselines.} The output of ICP as a method is the intersection
$
    I_{\ICP} = \bigcap_{I \in \invariantset} I.
$
Thus, for ICP as a baseline, we fit a linear model trained on this subset of features and denote it by ICP (intersection). As another baseline, we fit the anchor regression estimator \cite{rothenhausler2021anchor} on $\mathcal{E}_{\mathrm{train}}$ (we pick the best regularization parameter $\lambda$ by cross-validation). The anchor regression estimator aims to be robust to bounded additive shifts under potential confounding between $X$ and $Y$, and, depending on the strength of the regularization parameter, interpolates between ERM and IV regression. We fit the source model on pooled source data. 

\textbf{Target-dependent baselines.} As the target distribution, we consider the intervention on $T^V_1$, which in the light tunnel dataset induces a strong distribution shift. We vary the available number $\ntarget$ of labeled target samples and compute the empirical target model $\hempQ$ based on the small sample as well as the oracle target model, which is computed on a large held-out target dataset. As a "naive" sDA baseline, we take ERM over the candidate model set. We split the target data and include the empirical target model as a candidate for the ERM baseline. 

\textbf{Our method.} For our procedure \Cref{alg:iterative-step-agg}, we use (inside the algorithm) exponential weights aggregation and pick a significance level $\alpha = 0.05$. The main challenge of \Cref{alg:iterative-step-agg} is picking the constants $C_1$ and $C_2$ which control the selection against the target model and refinement around the aggregator, respectively. It is well known from linear regression literature that the optimal value is $\Creg = \sigma_{\mathrm{res}}^2$ where $\sigma^2_{\mathrm{res}} = \Var(Y - \hQ(X))$ is the variance of the residuals of the optimal linear model. We use adaptive values for $C_1, C_2$: given a small target sample, we fit a linear model on one half of the data and estimate the variance of its residuals on the other half. We then compute the upper bound $\widehat{\sigma}_{\mathrm{res}}^2$ for the $95\%$-confidence interval for $\sigma_{\mathrm{res}}^2$ and average predictions on both folds. We then set $C_1, C_2 = \widehat{\sigma}_{\mathrm{res}}^2$. We note that \Cref{thm:main-result} sets a requirement for the relationship of $C_1,C_2$ and $\Creg, \Cagg$. This relationship is not guaranteed to hold with our heuristic choice of $C_1$ and $C_2$, but leads to the best performance in practice.   As with other target-dependent baselines, we add the empirical target model trained on another split to our set of candidates. However, we only do this if Step 1 of \Cref{alg:iterative-step-agg} is passed. Otherwise the full empirical target model is returned. 

\textbf{Results. } We evaluate all methods on $6000$ held-out target samples and plot the MAE (mean absolute error) as a function of the available target sample size $\ntarget$. The results are averaged over $50$ runs. The shaded regions depict standard error of the mean. We observe that the source model has the highest target MAE, followed by anchor regression and ICP (intersection). Our method achieves low target risk in the low-sample regime with advantage over the naive ERM in the larger sample regime. 

\subsubsection{Causal Chamber: image regression experiment with linear probes}\label{apx:cc-linear-probes}
\begin{figure}[t]
    \centering

    \begin{subfigure}[t]{0.7\linewidth}
        \centering
        \includegraphics[width=\linewidth,height=0.22\textheight,keepaspectratio]{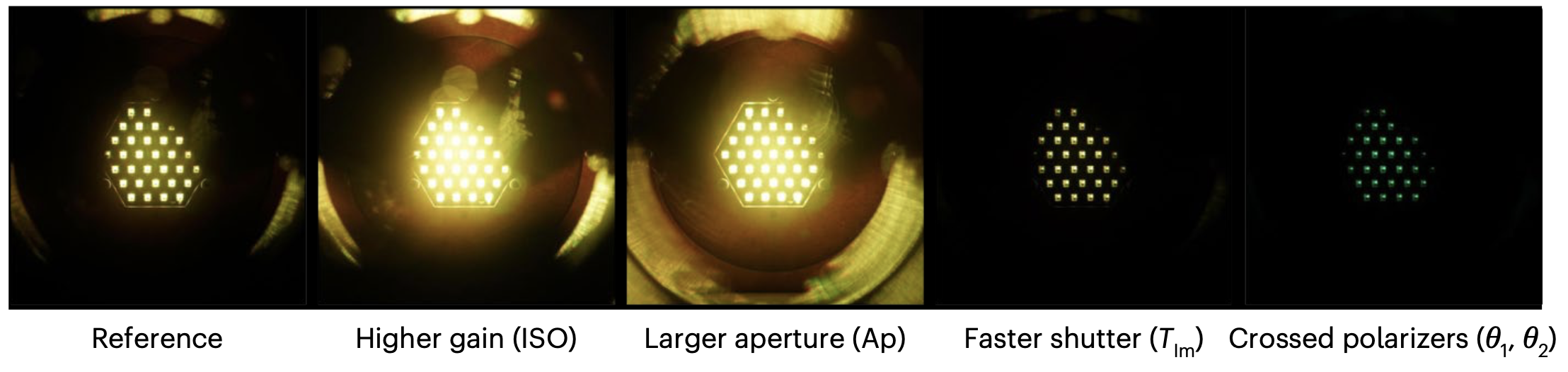}
        \caption{Light tunnel image data under interventions (Figure due to \cite{gamella2024causal}).}
        \label{fig:cc_light_tunnel_interventions}
    \end{subfigure}
    \begin{subfigure}[t]{0.4\linewidth}
        \centering
        \includegraphics[width=\linewidth]{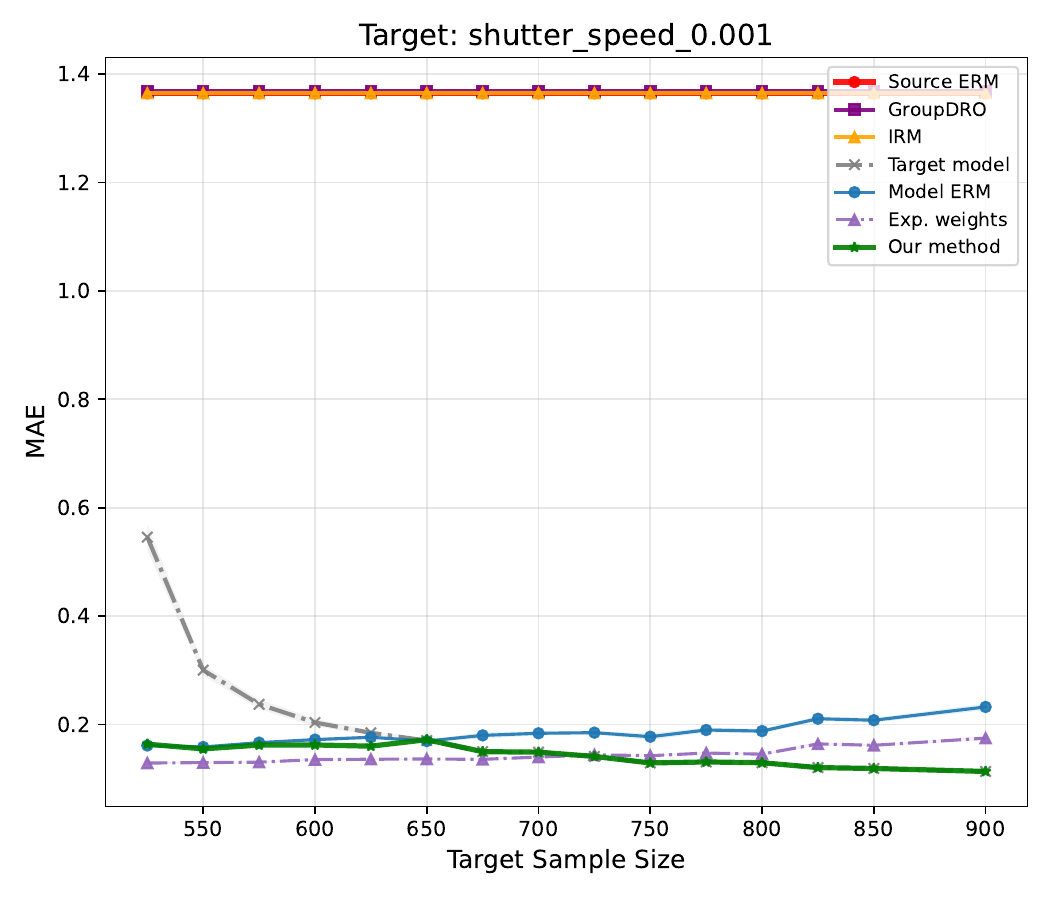}
        \caption{MAE of various methods as a function of $\ntarget$ for the image regression experiment.}
        \label{fig:cc_image_regression_mae}
    \end{subfigure}

    \vspace{0.8em}

    \caption{(a) Examples of light tunnel image data under various interventions on the camera and tunnel setup. (b) MAE (mean absolute error) of various domain generalization/adaptation methods as a function of $\ntarget$ for the \textbf{image regression} experiment with linear probe candidates based on a frozen CLIP backbone. Results are averaged over 5 runs. The shaded regions depict the standard error of the mean (SE).}
    \label{fig:cc_image_regression_with_lt}
\end{figure}

\textbf{Problem setting.} In this experiment, we consider an image regression problem under distribution shift. In particular, we consider \emph{more general} candidate models: 1) instead of pure linear regression, we compute linear probes based on a pretrained representation. 2) Since feature-subset models are not interpretable in this setting, we compute models satisfying \emph{subsets of invariances}, namely, for any subset $\mathcal{E} \subseteq \mathcal{E}_{\mathrm{train}}$ of training domains, we compute a model minimizing a certain risk measure on $\mathcal{E}$. We then perform supervised adaptation with the generated candidates.

\textbf{Data. } We consider the $\mathrm{lt\_color\_regression\_v1}$ dataset from the Causal Chambers benchmark \cite{gamella2024causal}. This dataset is collected from the light tunnel augmented with a camera, producing \textbf{images} from inside the tunnel. Along with the images, we have access to the ground-truth $(R,G,B)$ values of the LEDs before passing the tunnel. Our task is to predict the "brightness" value $\frac{1}{3}(R+G+B)$ from the light tunnel images. The dataset provides 10 environments, including the reference environment and 9 interventional environments with changes in the camera parameters (e.g., aperture and shutter speed), polarization filters and LED brightness. 

\textbf{Multi-domain setup.} As before, to ensure an under-identified scenario, as source data, we take data from three domains:  $\mathrm{reference}$, $\mathrm{aperture\_5.0}$, and $\mathrm{shutter\_speed\_0.002}$.

\textbf{Pretrained backbone. } We extract features from the image data using the frozen encoder of the pretrained CLIP ViT-B/32 model with "openai" weights. We encode the images and normalize the resulting $512$-dimensional representations. 

\textbf{Candidate set of linear probes.} To obtain the candidate models, for every non-empty subset $\mathcal{E} \subseteq \{0,1,2\}$ of the source environments, we compute the linear model on the CLIP features minimizing 1) the pooled loss (ERM objective); 2) the IRM objective \cite{arjovsky2019invariant}; 3) the groupDRO objective \cite{sagawa2019distributionally}. For IRM, we take hyperparameters $\mathrm{epochs} = 500$, $\mathrm{lr} = 10^{-3}$, $\lambda = 1.0$ and use Adam as optimizer. For groupDRO, we take $\mathrm{epochs} = 500$, $\mathrm{lr} = 10^{-3}$, $\eta = 0.1$. In total, we obtain $3 \times 7 = 21$ candidate models. 

\textbf{Baselines. } As baselines, we consider the groupDRO and IRM models trained on all sources, as well as the source ERM model. The target-dependent baselines include the target empirical model, exponential weights aggregation and naive ERM over models. 

\textbf{Our method. } We implement our procedure \Cref{alg:iterative-step-agg} as in \Cref{apx:cc-linear-regression-details}. 

\textbf{Results. }  We evaluate all methods on 6000 held-out target samples and plot the MAE (mean absolute error) as a function of the target sample size. The shaded regions represent standard error of the mean. We observe that source ERM, groupDRO and IRM all result in high target risk. In fact, due to insufficient diversity of environments and high expressivity of the pretrained features, all three algorithms result in very similar models (compare with \Cref{sec:toy-example}). In contrast, our method \Cref{alg:iterative-step-agg} is the only target-dependent procedure which successfully interpolates between the target MAE of the best candidate model and the target linear probe. Importantly, the naive ERM (with the target model included as a candidate) does not achieve such interpolation and exhibits negative transfer. 

All listed experiments were run on an Apple M2 Max GPU in under 1 hour.

\subsection{Experiments on gene expression data}\label{apx:gene-experiments}

\textbf{Source multi-environment setup and target data.} We analyze single-cell gene-expression data from a pooled CRISPRi Perturb-seq experiment in RPE1 cells \cite{replogle2022mapping}. The dataset includes $424$ domains, each induced by a different gene knockdown. We consider a collection of $10$ most highly expressed genes in the observational (non-targeting) environment. We collect $10$ environments corresponding to interventions (knockdowns) of the $10$ observed genes, resulting in a total of $11$ environments. We randomly pick one of the genes as the target variable $Y$ and use the rest as features $X$. To ensure stability of the causal mechanism of $Y$, we discard the environment in which $Y$ is knocked down, resulting in a pool of $10$ environments. We center each interventional environment by subtracting the $(X,Y)$ mean within the observational environment. As source data, we take the observational environment together with $3$ randomly selected knockdown environments. As the target environment, we pick the most challenging one, characterized by the biggest shift in $\E[XX^\top]$. 

\textbf{Candidate set. } Due to strong confounding induced by a large number of unobserved genes, and possibly nonlinear relationships in the data, $\ICP$ \cite{peters2016causal} fails to accept any feature subsets, even on subsets of environments. As candidate sets, we thus consider \emph{all feature subsets} over the $d$ predictors, resulting in $512$ candidates. This collection corresponds to $\suffcollection$ in \Cref{sec:new-causal}, and to \emph{no causal knowledge}: although we believe that the gene expression data has a fixed underlying causal graph, and that some causal mechanisms remain invariant, we cannot extract any information about its structure from multiple domains. For each candidate feature set, we fit a linear model $\hempIP$ on the source data pooled from all $4$ environments.

\textbf{Baselines. } As baselines, we include the source model trained on pooled data, the ICP model (which for the gene data results in a constant linear model), and anchor regression trained on source environments. As in previous experiments, as sDA baselines, we include ERM over candidate models + target model as well as exponential weights aggregation, for which we take the temperature $\beta = 2\hat{\sigma}_Y^2$ (in the notation of \Cref{apx:model-aggregation}, so that the weights are $\lambda_m \propto \exp(-\ntarget \riskemp(h_m)/\beta) = \exp(-\ntarget \riskemp(h_m)/(2\hat\sigma_Y^2))$). We implement our method as discussed in \Cref{apx:cc-linear-regression-details}. 

\textbf{Results. } We iterate over $\ntarget = \{5,10,20,25,30,40,50,60\}$ and average over $20$ runs. We evaluate the MSE of all methods on $62$ held-out target samples (due to the small number of samples in the target environment, the oracle target model cannot be reliably approximated). The shaded regions depict standard error of the mean. 

\clearpage
\section{Background on model selection aggregation}\label{apx:model-aggregation}
Given a finite dictionary of functions $\hypo = \{ h_1,...,h_M \}$ and a finite dataset $\{(X_i,Y_i)\}_{i=1}^n$, the goal of model selection (MS) aggregation is to construct a procedure whose risk is closest to the best function in the dictionary. We emphasize that despite the name \emph{model selection}, MS aggregation does not actually select a model from the dictionary. Instead, the goal of MS aggregation is to be as good as the best model in the dictionary. We consider the risk induced by the mean squared error (MSE). Formally, a good model aggregation procedure should satisfy
\begin{equation}\label{eq:agg-oracle-inequality}
    \risk(\hagg) \leq \min_{m \in [M]} \risk(h_m) + \Delta_{n,M}
\end{equation}
with the remainder $ \Delta_{n,M}$ as small as possible. In the classical line of work on MS aggregation \citep{rigollet2012kullback,tsybakov2003optimal} it has been established that the smallest possible order for $ \Delta_{n,M}$ is $\log M /n$. In contrast, any selector from the dictionary $\hypo$, including the straightforward choice of ERM over the dictionary, cannot satisfy \eqref{eq:agg-oracle-inequality} with a remainder smaller than $\sqrt{(\log M) / n}$. It has been concluded that any procedure satisfying \eqref{eq:agg-oracle-inequality} with the fast remainder term must in some sense "hedge", or "average" candidate models: $h_{\lambda} = \sum_{m=1}^M \lambda_m h_m$ with $\lambda \in \Lambda$ (the simplex). 

\textbf{Exponential weights aggregation.} One of the earliest and most practical MS aggregation estimators is based on \emph{exponential weights} \citep{yang1999model,audibert2007progressive,juditsky2008learning}. Given a prior $\pi \in \Lambda$ over models and a temperature parameter $\beta > 0$, the exponential weights are defined by
\begin{equation*}
    \lambda^{\mathrm{EXP}}_m \propto \pi_m \exp(-n \riskemp(h_m) / \beta).
\end{equation*}
The most straightforward choice of the prior is $\pi = (1/M,...,1/M)$ (uniform). Exponential weights achieve the optimal aggregation rate in expectation. However, \cite{dai2012deviation} have shown that exponential weights are generally suboptimal in the sense that \eqref{eq:agg-oracle-inequality} cannot hold with high probability, unless additional Bernstein assumptions hold. 

\textbf{Q-Aggregation.}\footnote{We remark that Q-aggregation does not refer to our target distribution $Q$ and is merely a notational coincidence.} Q-aggregation \cite{lecue2014optimal} is explicitly designed to be optimal in deviation and uses higher interpolation orders between models. For a prior $\pi$ and hyperparameters $\nu, \beta$, the weights of the Q-aggregator minimize the penalized Q-functional
\begin{equation*}
    \hat{\lambda} \in \argmin_{\lambda \in \Lambda} \left[ (1-\nu) \riskemp(h_{\lambda}) + \nu \sum_{m=1}^M \lambda_m \riskemp(h_m) - \frac{\beta}{n} \sum_{m=1}^M \lambda_m \log \pi_m  \right].
\end{equation*}
Assume that the label and the output space are uniformly bounded:
\begin{equation*}
    | Y | \leq b; \quad \max_{h \in \hypo} |h(X)| \leq b \text{ a.s. }
\end{equation*}
Under standard loss regularity conditions satisfied for squared loss, a temperature condition $\beta \geq C(b,\nu)$ and a uniform prior, we obtain optimal high-probability guarantees
\begin{equation*}
    \risk(h_{\hat{\lambda}}) \leq \min_{m} \risk(h_m) + 2 \beta \frac{\log M + \log(1/\alpha)}{n}.
\end{equation*}
We define $\Cagg(B) \coloneqq 2 \beta$.

\textbf{The STAR estimator.} The STAR estimator \cite{audibert2007progressive,audibert2007proof} is another deviation-optimal aggregation procedure. Notably, the algorithm is parameter-free: it considers the star-shaped union of line segments between $h_{\mathrm{ERM}}$ (the model with the lowest empirical risk in the dictionary) and all other models $h_m$. It then performs empirical risk minimization on this star-shaped set. Under standard boundedness assumptions $|Y| \leq B$ and $\|\h_m\|_{\infty} \leq B$, the STAR method enjoys the high-probability guarantee
\begin{equation*}
    \risk(h_{\mathrm{STAR}}) \leq \min_{m} \risk(h_m) + \Cagg(B) \frac{\log M + \log(1/\alpha)}{n}.
\end{equation*}

\textbf{Hyper-sparse optimal aggregation.} \cite{gaiffas2011hyper} introduce hyper-sparse aggregation and show that the minimal amount of models which have to be combined to achieve the optimal MS aggregation rate is two. Assume
\begin{equation*}
    Y = h^\star(X) + \epsilon,
\end{equation*}
with $\| \epsilon \|_{\psi_1} \leq B$, where $\| \epsilon \|_{\psi_1}$ denotes the sub-exponential norm. Assume that $\sup_{h \in \hypo} \| h - h^\star \| \leq B$. In particular, \cite{gaiffas2011hyper} allow for \emph{unbounded output} Y and \emph{heavy-tailed} noise. \cite{gaiffas2011hyper} propose a procedure which with probability at least $1-\alpha$ achieves
\begin{equation}\label{eq:agg-subexponential}
    \risk(\hagg) \leq \min_{m \in [M]} \risk(h_m) + \Cagg(B) \frac{(1+\log(1/\alpha)) \log M \log n}{ n}.
\end{equation}
We observe that the guarantee \eqref{eq:agg-subexponential} differs from the optimal rate via a $\log n$ factor and multiplicative dependence on $\log(1/\alpha)$. To the best of our knowledge, it remains an open question whether optimal high-probability guarantees can be established for MS aggregation for regression under random design. Importantly, the guarantee \eqref{eq:agg-subexponential} allows \emph{our result  \Cref{thm:main-result} to hold under the assumption of sub-exponential, in particular, sub-Gaussian  residuals}, i.e., if $\sup_{I} \| Y - \E[Y|X_I]\|_{\psi_1} \leq B$, with $\erroragg \coloneqq C_B \frac{(1+\log(1/\alpha))\log M \log n}{n}.$

\clearpage
\section{Deferred Proofs}\label{apx:deferred-proofs}
\subsection{Proof of \Cref{prop:optsetmin-membership}}\label{sec:proof-of-optsetmin-membership}
    Derivation:
    \begin{align*}
        R_{Q}(h_{I\cup I' \cup I''}) 
        &= \mathbb{E}_{Q}[(Y - \mathbb{E}_P[Y\mid X_{I\cup I' \cup I''}])^2]\\
        &= \mathbb{E}_{Q}\big[\mathbb{E}\big[(Y - \mathbb{E}_P[Y\mid X_{I\cup I' \cup I''}])^2\mid Y,X_{I\cup I'}\big]\big] && \text{Tower rule}\\
        &= \mathbb{E}_{Q}\big[\mathbb{E}_{Q}\big[\mathbb{E}[(Y - \mathbb{E}_P[Y\mid X_{I\cup I' \cup I''}])^2\mid X_{I''}]\mid Y,X_{I\cup I'}\big]\big] && \text{Tower rule}\\
        &= \mathbb{E}_{Q}\big[\mathbb{E}_{P}\big[\mathbb{E}[(Y - \mathbb{E}_P[Y\mid X_{I\cup I' \cup I''}])^2\mid X_{I''}]\mid Y,X_{I\cup I'}\big]\big] && \text{Due to condition (2)}\\
        &< \mathbb{E}_{Q}\big[\mathbb{E}_{P}\big[\mathbb{E}[(Y - \mathbb{E}_P[Y\mid X_{I\cup I'}])^2\mid X_{I''}]\mid Y,X_{I\cup I'}\big]\big] && \text{Worse under $P(X_{I''}\mid Y, X_{I\cup I'})$}\\
        &= \mathbb{E}_{Q}\big[\mathbb{E}[(Y - \mathbb{E}_P[Y\mid X_{I\cup I'}])^2\mid Y,X_{I\cup I'}\big]\big] && \text{Tower rule}\\
        &= \mathbb{E}_{Q}\big[\mathbb{E}_{Q}\big[\mathbb{E}[(Y - \mathbb{E}_P[Y\mid X_{I\cup I'}])^2\mid Y, X_{I}\big],do(X_{I'})\big]\big] && \text{Due to (1) and rule 3 of do-calc.}\\
        &\leq \mathbb{E}_{Q}\big[\mathbb{E}_{Q}\big[\mathbb{E}[(Y - \mathbb{E}_P[Y\mid X_{I}])^2\mid Y,X_{I} \big]\mid do(X_{I'}))\big]\big] && \text{Worse under $Q(Y,X_{I}\mid do(X_{I'}))$}\\
        &=\mathbb{E}_{Q}\big[\mathbb{E}_{Q}\big[\mathbb{E}[(Y - \mathbb{E}_P[Y\mid X_{I}])^2\mid Y,X_{I} \big]\mid X_{I'}\big]\big] && \text{Due to (1) and rule 3 of do-calc.}\\
        &=\mathbb{E}_{Q}\big[(Y - \mathbb{E}_P[Y\mid X_{I}])^2\big] && \text{Tower rule}\\
        &=\riskQ(h_{I,P})
    \end{align*}
    By (2), $X_{I''}$ is d-separated from $\Disc_{P,Q}$ given $X_{I\cup I'},Y$, and therefore, it is a stable variable under $P,Q$, thus, adding $X_{I''}$ to $X_{I\cup I'}$ does not hurt the prediction. By (3), $X_{I''}$ is not d-separated from $Y$ given $X_{I\cup I'}$, and therefore, the strict inequality holds under a stricter notion of faithfulness assumptions: the equality holds only if $h_{I\cup I',P} = h_{I\cup I' \cup I'',P}$, and in that case, we expect $X_{I''}$ to be d-separated from $Y$ given $X_{I,I'}$, which contradicts condition (3). Lastly, by (1), $Y$ is d-separated from $X_{I'}$ given $X_{I}$ in $\mathcal{G}_{\overline{X_{I'}}}$, which is the graph obtained from $\mathcal{G}$ by removing the edges going into the $X_{I'}$ nodes. Therefore, by rule 3 of do-calculus, $Q(Y\mid X_{I},X_{I'}) = Q(Y\mid X_{I},do(X_{I'}))$; in this regime, $X_{I'}$ is a sink node, unaffected by $\Disc_{P,Q}$, therefore adding it does not hurt the risk. For every admissible distribution $Q$, we showed $\riskQ(h_{I\cup I' \cup I'',P}) < \riskQ(h_{I,P})$, and therefore, we conclude $I\not\in \optsetmin$. 

\subsection{Proof of Proposition \ref{prop:impossibility-for-optimistic}}\label{sec:proof-of-impossibility-for-optimistic}
\textbf{Example 1.} For the first example, we consider the causal graph described by $\Ch(Y) = X$. More precisely, consider the linear Gaussian SCM
\begin{equation*}
    Y = U_Y; \quad X_j = Y + U_j \text{ for } j \in [d],
\end{equation*}
where $U_Y \sim \cN(0,1)$, $U_j \sim \cN(0,1)$ are all independent. Let $S(P,Q) = [d]$. For any non-empty subset $J \subseteq [d]$, consider $\mu \in \R^d$ with $\mu_j = 0$ if $j \in J$ and $\mu_j = M$ otherwise. 
Define the target distribution $Q$ by the mean shift $\mu$ by setting $X_j = Y + U_j + \mu_j$. The distribution of $Y$ remains invariant. 
Then, we compute 
\begin{equation*}
    \riskQ(h_{J',P}) = \frac{1}{1+|J'|} + \left( \frac{\sum_{j \in J'} \mu_j^2}{1+|J'|}  \right)^2. 
\end{equation*}
In other words, the risks of all models $h_{J',P}$  with $J' \neq J$ grows as $M^2$. The empty-set model has (invariant) risk $1$, in particular it achieves a worse target risk than $h_{J,P}$. It follows that there exists $M > 0$ large enough such that $h_{J,P}$ has overall minimal target risk. Thus, it has to hold that $J \in \optset$. 

\textbf{Example 2.} For the second example, we consider the causal graph described by the following: $\pa(Y) = (X_1,...,X_{d-1})$, $\pa(X_{d}) = (X_1,...,X_{d-1},Y)$. We set $\shiftset{P}{Q} = \{d\}$, i.e., only the causal mechanism of the child node can be shifted. Let $X_1,...,X_{d-1} \sim \cN(0,1)$, $U_Y \sim \cN(0,1)$, $U \sim \cN(0,1)$, all independent. Define $Y = \sum_{i=1}^{d-1} X_i + U_Y$, $X_d = \sum_{j < d} 0\cdot X_j +  Y + U$. First, we note that for a large enough mean shift on $X_d$, $[d-1] = \pa(Y)$ becomes the target risk minimizer among feature-subset models, thus $[d-1]$ is in $\optsetmin$. Further, $[d]$ is in $\optsetmin$ since $Q = P$ is admissible. Now, for any $J \subsetneq [d-1]$, let $k = |J|$ and write $S_J =\sum_{j \in J} X_j$. We compute the linear regression coefficient 
\begin{equation*}
    h_{J\cup\{d\},P} = \alpha_J X_J + \beta_J X_d,
\end{equation*}
where $\alpha_J = \frac{1}{d-k+1}$, $\beta_J = \frac{d-k}{d-k+1}$. We have $\beta_J > 0$ for every $J$. Under $Q$, we now define the new causal mechanism
\begin{equation*}
    X_{d} = \frac{d-k+1}{d-k} Y - \frac{1}{d-k} \sum_{j \in J} X_j.
\end{equation*}
It then follows $\riskQ(h_{J\cup\{d\},P}) = 0$. We note that any predictor that doesn't include $X_d$ has risk at least $\Var(U_Y) = 1$. Any other predictor using $X_{d}$ and $J' \neq J$ has coefficients strictly different from  $h_{J\cup\{d\},P}$ and thus incurs a non-zero risk. It follows that $J\cup\{d\}$ has to be contained in $\optset$. We note that the coefficients of $X_1,...,X_{d-1}$ can be made strictly positive in the source data-generating process: since the linear regression coefficients and the risk are continuous functions of the linear SCM coefficients, there exists a neighborhood around zero for which the same risk inequalities hold under the described causal mechanism shifts.

\subsection{Proof of \Cref{thm:main-result}}\label{sec:proof-of-main-theorem}
First, we list lemmas related to matrix concentration and convergence of linear regression. These lemmas are an integral part of the proof of our main results, and are of independent interest. 
\subsubsection{High-probability excess risk bounds for linear regression}\label{sec:high-prob-bounds-linear-regression}
Let $D$ be a joint distribution over $(X,Y)$. 
We assume that $\Sigmad \coloneqq \E_D[X X^\top]$ is positive definite with $\lambda_{\min}(\Sigmad) \geq \lambda_{\min} > 0$. 
The following lemma \cite{hanneke2025adaptivesampleaggregationtransfer} establishes covariance concentration for small sample sizes:
\begin{lemma}[Matrix concentration]\label{lem:matrix-concentration}
    Let $\Sigma_D = \E_D[XX^\top]$ and $\hat{\Sigma}_D = \E_{\hat D}[XX^\top]$.
    Assume that $\|X\|_2 \le B_X$ almost surely and that
    \[
    \lambda_{\min}(\Sigma_D) \ge \lambda_{\min} > 0.
    \]
    Then there exists a constant $C = C(B_X,\lambda_{\min},\tau)$ such that if $
    n \ge C(B_X,\lambda_{\min},\tau)\,\log(d/\tau)$,
    then with probability at least $1-\tau$,
    \begin{align*}
        \frac12 \Sigma_D \preceq \hat{\Sigma}_D \preceq \frac32 \Sigma_D.
    \end{align*}
    Equivalently, for every $v \in \R^d$,
    \begin{align*}
        \frac12 \|v\|_{\Sigma_D}^2 \le \|v\|_{\hat{\Sigma}_D}^2 \le \frac32 \|v\|_{\Sigma_D}^2.
    \end{align*}
\end{lemma}

\begin{proof}
    Let $
    Z \coloneqq \Sigma_D^{-1/2}X$.
    Then $
    \E[ZZ^\top] = I_d$. 
    Moreover, since $\lambda_{\min}(\Sigma_D)\ge \lambda_{\min}$ and $\|X\|_2\le B_X$ almost surely,
    \begin{align}
        \|Z\|_2^2
        \le \frac{B_X^2}{\lambda_{\min}}.
    \end{align}
    Hence $
    \sup_X \|\Sigma_D^{-1/2}X\|_2^2 \le \frac{B_X^2}{\lambda_{\min}}$.

    Now define
    \[
    \hat{\Sigma}_Z \coloneqq \frac1n \sum_{i=1}^n Z_i Z_i^\top
    = \Sigma_D^{-1/2}\hat{\Sigma}_D\Sigma_D^{-1/2}.
    \]
    By the matrix concentration result in \cite{tropp2015introduction}, there exists a constant
    $c = K \cdot B_X^2/\lambda_{\min}$  such that with probability at least $1-\tau$,
    \[
    \left\|\hat{\Sigma}_Z - I_d\right\|_{\mathrm{op}}
    \le
    \sqrt{\frac{c\log(d/\tau)}{n}} + \frac{c\log(d/\tau)}{n}.
    \]
    Therefore, if $n \ge C(B_X,\lambda_{\min},\tau)\log(d/\tau)$,
    the right-hand side is at most $1/2$. This implies
    \[
    \frac12 I_d \preceq \hat{\Sigma}_Z \preceq \frac32 I_d.
    \]
    Multiplying on the left and right by $\Sigma_D^{1/2}$ yields
    \[
    \frac12 \Sigma_D \preceq \hat{\Sigma}_D \preceq \frac32 \Sigma_D.
    \]
\end{proof}

For completeness, we also state the high-probability excess-risk bound for linear regression \citep{hsu2012random, hanneke2025adaptivesampleaggregationtransfer}, instantiated under \Cref{asm:tail-of-X-and-Y}: since $\|X\|_2 \leq B_X$ and $|Y| \leq \sigma_Y$ almost surely, $X$ is sub-Gaussian with norm $\mathcal{O}(B_X)$ and $Y$ is sub-Gaussian with variance proxy $\mathcal{O}(\sigma_Y^2)$, so the standard random-design linear-regression result applies in the bounded case. 
\begin{lemma}[Upper bounds for linear regression]\label{lem:convergence-linear-regression}
   Let $h^\star \in \arg\min_{h \in \hypolin} \risk_D(h)$. Let $\hemp \in \arg\min_{h \in \hypolin} \riskemp_D(h)$. Let $\alpha > 0$. There exist absolute constants $c_0, C_0 > 0$ such that if $n \geq c_0 \frac{B_X^4}{\lambdamin^2}(d+\log(1/\alpha))$, then with probability at least $1-\alpha$
   \begin{equation*}
       \risk_D(\hemp) - \risk_D(\hstar) \leq \Creg \frac{d + \log(1/\alpha)}{n},
   \end{equation*}
   where $\Creg \coloneqq C_0 \frac{B_X^2 \sigma_Y^2}{\lambdamin^2} $. 
\end{lemma}

\subsubsection{Effect of finite source data}\label{sec:effect-of-source-data}
In this section, we  prove statements which allow us to bound the population target risks of feature-subset linear models $\hempIP$ computed on the \emph{finite} source dataset $\dataP$, as well as gaps between these risks. 
For completeness, we restate the notation for population and empirical feature-subset models. For a feature subset $I \subset [d]$, recall the definitions 
\begin{align*}
    \hIP \coloneqq \arg\min_{h \in \hypoI} \risksource(h); \quad \hempIP \coloneqq \arg\min_{h \in \hypoI} \riskempP(h),
\end{align*}
where $\hypoI = \{h:\R^d \to \R,\, h(x) = w^\top x, \, \supp w \subseteq I \}$. 
For a collection $\set \subset 2^{[d]}$ of feature subsets, write
\begin{align*}
    \Istar \in \arg\min_{I \in \set} \riskQ(\hIP) \quad \text{and} \quad \widehat{\Istar} \in \arg\min_{I \in \set} \riskempP(\hempIP)
\end{align*}
for the indices of the best \emph{population} candidate model on the target distribution and of the best \emph{finite-sample} source model on the source distribution, respectively. This induces four candidate models: $\hbestP$ (the best population candidate), $\hempbestP$ (its finite-sample fit at the same index $\Istar$), $\hemppseudobestP$ (the empirically-selected candidate), and $\hpseudobestP$ (its population counterpart at the same index $\widehat{\Istar}$). Note that neither the population models $\hbestP$ and $\hpseudobestP$, nor the empirical models $\hempbestP$ and $\hemppseudobestP$ necessarily coincide unless $\nsource = \infty$.

We first state a "sandwich" bound for the population target risk of any candidate model $\hempIP$:
\begin{lemma}[Excess target risk of source models]\label{lem:target-source-risk-bound}
   Let $\lambdamax = \lambdamax(\Sigmasource^{-1}\Sigmatarget)$. Let $I \in \set$. Let $0 < \alpha < 1$. Then, with probability at least $1-\alpha$, 
   \begin{align}\label{eq:lemma-finite-np}
    \excessQ(\hIP) - \ErrnpI   \leq  \excessQ(\hempIP) \leq \excessQ(\hIP) + \ErrnpI,
   \end{align}
where $\ErrnpI \coloneqq  2 \sqrt{\excessQ(\hIP)}\sqrt{\ratenpI} + \ratenpI$, and  $\ratenpI \coloneqq \Creg \lambdamax \frac{|I| + \log(1/\alpha)}{\nsource}$ is the standard rate of convergence of linear regression weighted by the covariate shift ratio $\lambdamax$. 
\end{lemma}
\begin{proof}
    We compute 
\begin{align}\label{eq:target-by-source-upper-bound}
    \excessQ(\hempIP) &= \| \hempIP - \hQ \|_{\Sigmatarget}^2 \leq (  \| \hempIP - \hIP \|_{\Sigmatarget} +  \| \hIP - \hQ \|_{\Sigmatarget})^2 \\
    &\leq (\sqrt{\lambdamax} \| \hempIP - \hIP \|_{\Sigmasource} + \sqrt{\excessQ(\hIP)} )^2,
\end{align}
where we have used the finiteness of the covariate shift ratio $\lambdamax$. The term $\| \hempIP - \hIP \|_{\Sigmasource}$ is exactly the square root of the source excess risk of $\hempIP$ with respect to the hypothesis class $\hypoI$. We thus use \Cref{lem:convergence-linear-regression} with dimension $|I|$ to obtain with probability at least $1-\alpha$
\begin{align*}
     \| \hempIP - \hIP \|_{\Sigmasource}^2 \leq \Creg \frac{|I| + \log(1/\alpha)}{\nsource},
\end{align*}
where $\Creg = C(B_X, \lambdamin, \sigma_Y^2)$ as in \Cref{lem:convergence-linear-regression}.
Inserting this in \eqref{eq:target-by-source-upper-bound} yields the upper bound. The lower bound follows similarly by \Cref{lem:convergence-linear-regression} and reverse triangle inequality. 
\end{proof}
Looking at the sandwich bound \eqref{eq:lemma-finite-np}, we observe that $\ErrnpI \to 0$ as $\nsource \to \infty$. However, the speed of convergence depends on the value of $\excessQ(\hIP)$. If $\excessQ(\hIP) \lesssim \ratenpI$, then $\ErrnpI = \mathcal{O}(d/\nsource)$, resulting in fast convergence. However, when $\excessQ(\hIP) \gtrsim \ratenpI$, we have $\ErrnpI = \mathcal{O}(\sqrt{d/\nsource})$. Since both the linear regression bound in \Cref{lem:convergence-linear-regression} and the inequality $|a+b|^2 \leq (|a|+|b|)^2$ are tight, we remark that this slow-fast rate for the target risk of candidate linear models is unavoidable. 

The following bound for the \emph{gaps} between target risks of finite-sample source models follows immediately from \Cref{lem:target-source-risk-bound}:

\begin{lemma}\label{lem:from-gap-to-gap}
     Consider two models $\hIoneP$ and $\hItwoP$. Define the risk margins  $\Delta(\hIoneP, \hItwoP) = |\riskQ(\hItwoP) - \riskQ(\hIoneP)|$ and $\Delta(\hempIoneP, \hempItwoP) = |\riskQ(\hempItwoP) - \riskQ(\hempIoneP)|$. 
     Let $0 < \alpha < 1$. There exists a constant $C = C(B_X, \lambdamin, \sigma_Y^2)$ such that with probability at least $1-\alpha$, 
\begin{align*}
  \Delta(\hIoneP, \hItwoP) - \ErrnpIone - \ErrnpItwo &\leq  \Delta(\hempIoneP, \hempItwoP) \\ &\leq \Delta(\hIoneP, \hItwoP) + \ErrnpIone + \ErrnpItwo.
\end{align*}
\end{lemma}
Again, we observe that as $\nsource \to \infty$, both the lower and upper bound converge to the $P$-population risk margin $\Delta(\hIoneP, \hItwoP)$. However, now the speed of convergence is governed by $\max\{ \excessQ(\hIoneP),\excessQ(\hItwoP) \}$, and the fast rate for the gap is achieved whenever $\max\{ \excessQ(\hIoneP),\excessQ(\hItwoP) \} \lesssim \frac{d}{\nsource}$. 
\subsubsection{Proof of Theorem \ref{thm:main-result}}\label{sec:actual-proof-of-theorem}

From the auxiliary lemmas in \Cref{sec:high-prob-bounds-linear-regression} and \Cref{sec:effect-of-source-data}, we obtain the following facts that we will use in the main proof below:
\paragraph{1. Concentration of empirical covariance-weighted norm:} For $\ntarget \geq C(B_X, \lambda_{\min}, \alpha) \log(d/\alpha)$, with probability at least $1-\alpha$, it holds for any vector $v \in \R^d$:
\begin{align*}
        \frac12 \|v\|_{\Sigma_D}^2 \le \|v\|_{\hat{\Sigma}_D}^2 \le \frac32 \|v\|_{\Sigma_D}^2.
\end{align*}
\paragraph{2. Excess risk of the linear regression estimator:} If $n \geq c_0 \frac{B_X^4}{\lambdamin^2}(d+\log(1/\alpha))$, then with probability at least $1-\alpha$
   \begin{equation*}
       \risk_D(\hemp) - \risk_D(\hstar) \leq \Creg \frac{d + \log(1/\alpha)}{n},
   \end{equation*}
\paragraph{3. Sandwich identity for the target risk of the source candidate models:} With probability at least $1-\alpha$,
   \begin{align*}
    \excessQ(\hIP) - \ErrnpI   \leq  \excessQ(\hempIP) \leq \excessQ(\hIP) + \ErrnpI,
   \end{align*}
where $\ErrnpI \coloneqq  2 \sqrt{\excessQ(\hIP)}\sqrt{\ratenpI} + \ratenpI$, and  $\ratenpI \coloneqq \Creg \lambdamax \frac{|I| + \log(1/\alpha)}{\nsource}$ is the standard rate of convergence of linear regression weighted by the covariate shift ratio $\lambdamax$.

\paragraph{Excess risk guarantees for the adaptive procedure.} We now consider the output of \Cref{alg:iterative-step-agg} step-by-step and prove excess risk guarantees for the first iteration of the algorithm. We then show that additional iterations cannot increase the excess risk of the output. 
\paragraph{Step 1: Guard against negative transfer.} First, if $\hypoinit$ is empty (corresponding to an empty feature set collection $\set$), the target model is chosen, and the target rate term in the minimum is attained. Thus, without loss of generality, we assume that $\hypoinit$ is not empty. We condition on the event of 
\Cref{lem:matrix-concentration} on $\datatargetone$ (concentration of $\Sigmatargetemp$),the event of \Cref{lem:convergence-linear-regression} on $\datatargetone$  (convergence of linear regression) and the event of \Cref{lem:target-source-risk-bound} for all $I \in \set$ (target risk of empirical source models). 

First, consider the case that after completing Step 1, the accepted set $\hypoacczero$ is empty; that is, $ \| \hempIP - \hempQ \|^2_{\Sigmatargetemp} > C_1 \frac{d+\log(1/\alpha)}{\ntarget}$ and the output of the algorithm $\tilde{h} = \hempQ$. Further, we know by \Cref{lem:convergence-linear-regression} that $\excessQ(\hempQ) = \| \hempQ - \hQ \|^2_{\Sigmatarget} \leq \Creg$ where $\Creg = C(B_X, \lambdamin, \sigma_Y^2)$ from \Cref{lem:convergence-linear-regression}
Observe that for any model $\hempIP \in \hypoinit$ (including $I=I^\star$), we have by triangle inequality
\begin{equation}
\label{eq:arsenalchampions}
    \| \hempIP - \hempQ \|_{\Sigmatarget}^2 \leq 2 \| \hempIP - \hQ \|_{\Sigmatarget}^2 + 2 \| \hempQ - \hQ \|^2_{\Sigmatarget},
\end{equation}
and given that the accepted set is empty, we further have by \Cref{lem:matrix-concentration} that 
\begin{equation*}
    \| \hempIP - \hQ \|_{\Sigmatarget}^2 \geq \frac{1}{3} \| \hempIP - \hempQ \|^2_{\Sigmatargetemp} - \| \hQ - \hempQ \|^2_{\Sigmatarget} \geq (\frac{1}{3} C_1 - \Creg)   \frac{d+\log(1/\alpha)}{\ntarget}.
\end{equation*}
In particular $\excessQ(\hempbestP) \geq c  \frac{d+\log(1/\alpha)}{\ntarget}$ for some non-negative constant $c>0$. 
Recall that $\hempbestP$ be s.t. $ \Istar \in \argmin_{I \in \set} \excessQ(\hIP)$  (the best candidate model) and $\bestrisk = \excessQ(\hbestP)$ as well as $\bestemprisk = \excessQ(\hempbestP)$. Further, with $\maxmarginerror$ defined in \Cref{eq:rnp-definition}, observe that by \Cref{lem:target-source-risk-bound} we have $\bestemprisk \leq \bestrisk + \maxmarginerror$. 
Thus, the first term in the minimum is attained with the constant $\frac{1}{3} C_1 - \Creg > 0$ since $C_1 > 3 \Creg$.

Now, let us consider the case that after completing Step 1, $\hypoacczero$ is non-empty. Then, for any accepted model $\hempIP \in \hypoacczero$, it again holds by \Cref{lem:matrix-concentration} and triangle inequality that
\begin{equation*}
    \| \hempIP - \hQ \|_{\Sigmatarget}^2 \leq 4 \| \hempIP - \hempQ \|^2_{\Sigmatargetemp} + 2 \| \hQ - \hempQ \|^2_{\Sigmatarget} \leq (4 C_1 + 2 \Creg) \frac{d+\log(1/\alpha)}{\ntarget}.
\end{equation*}
Thus, the excess risk is upper bounded by the first term in the minimum with constant $C = 4 C_1 + 2 \Creg$. 

We now show that the excess risk of the output of our procedure is indeed smaller than the second term in the minimum. For this, first, we need to show that $I^\star \in \hypoacczero$.
One can verify, using again \Cref{eq:arsenalchampions} and \Cref{lem:matrix-concentration}, that if $\bestrisk + \maxmarginerror \leq (C_1/3 - \Creg) \frac{d+\log(1/\alpha)}{\ntarget} $, since $\bestemprisk \leq \bestrisk + \maxmarginerror$, the model $\hempbestP$ satisfies the criterion of Stage 1 and thus $\hempbestP \in \hypoacczero$. Otherwise we have $\bestrisk + \maxmarginerror \geq(C_1 - \frac{1}{3} \Creg)\frac{d+\log(1/\alpha)}{\ntarget}$, and it follows for any $\hempIP \in \hypoacczero$ that 
\begin{equation*}
    \excessQ(\hempIP) \leq (4 C_1 + 2 \Creg) \frac{d+\log(1/\alpha)}{\ntarget} \leq \frac{(4 C_1 + 2 \Creg)}{C_1 - \frac{1}{3} \Creg} (\bestrisk + \maxmarginerror),
\end{equation*}
from which the bound of \Cref{thm:main-result} follows with $c_0 = \frac{(4 C_1 + 2 \Creg)}{C_1 - \frac{1}{3} \Creg} $ for \emph{any} $\hempIP$, regardless of pruning in Stage 2. Thus, without loss of generality, we restrict ourselves to the case when the best model is accepted, i.e. $\hempbestP \in \hypoacczero$.

\paragraph{Step 2: Iterative aggregation and refinement.} Here, we prove guarantees for the first iteration of Step 2. We then show that consecutive iterations do not worsen the guarantee. 
The set $\hypoacczero$ contains $|\hypoacczero|$ models. We consider a model selection aggregation estimator $\hagg(\hypoacczero)$, or $\hagg$ for short (cf. \Cref{apx:model-aggregation}), over the dictionary of models $\hypoacczero$ satisfying, with probability $1-\alpha$,  the oracle inequality 
\begin{equation*}
    \excessQ(\hagg) \leq \excessQ(\hemppseudobestP) + \Cagg \frac{\log |\hypoacczero|+\log(1/\alpha)}{\ntarget} =: \empbestemprisk + \erroragg,
\end{equation*}
where $\hemppseudobestP = \argmin_{\hempIP \in \hypoacczero } \excessQ(\hempIP)$ is the empirical best model and $\empbestemprisk = \excessQ(\hempIP)$ is its excess target risk (cf. \Cref{sec:effect-of-source-data}). 
On $\datatargettwo$, we perform the update $$\hypoaccone \gets \left\{ \hempIP: \| \hagg - \hempIP \|_{\Sigmatargetemp}^2 \leq C_2 \frac{\log |\hypoacczero| + \log(1/\alpha)}{\ntarget} \right\}.$$ 

\textbf{Case 1: The set $\hypoaccone$ is empty.} First, consider the case that the updated set $\hypoaccone$ is empty. If it were to hold that $\excessQ(\hemppseudobestP) \leq C \frac{\log  |\hypoacczero| + \log(1/\alpha)}{\ntarget}$, it would follow that
\begin{align*}
    \| \hemppseudobestP - \hagg \|_{\Sigmatarget}^2 &\leq \| \hemppseudobestP - \hQ \|_{\Sigmatarget}^2 + 2  \| \hemppseudobestP - \hQ \|_{\Sigmatarget}  \| \hagg - \hQ \|_{\Sigmatarget} + \| \hagg - \hQ \|_{\Sigmatarget}^2 \\ 
    &\leq 2(C + 2 \sqrt{C}\sqrt{C + \Cagg}+ C + \Cagg) \frac{\log |\hypoacczero| + \log(1/\alpha)}{\ntarget} \leq C_2 \frac{\log |\hypoacczero| + \log(1/\alpha)}{\ntarget},
\end{align*}
where we have used \Cref{lem:matrix-concentration} and assumed $C_2 > 2\Cagg$ and $C \leq \frac{(C_2 -2 \Cagg)^2}{8C_2^2}$. Thus, the empirical-best model $\hemppseudobestP$ would be accepted. It follows that $\excessQ(\hemppseudobestP) \geq C \frac{\log |\hypoacczero| + \log(1/\alpha)}{\ntarget}$. According to \Cref{alg:iterative-step-agg}, we output a model $\hempIP \in \arg\min_{\hempIP}  \| \hagg - \hempIP \|_{\Sigmatargetemp}^2$ with the minimal distance to the aggregator. For such a model $\hempIP$, it holds
\begin{align*}
     \| \hagg - \hempIP \|_{\Sigmatargetemp}^2 &\leq \| \hagg - \hemppseudobestP \|_{\Sigmatargetemp}^2 \leq \frac{3}{2} \| \hagg - \hemppseudobestP \|_{\Sigmatarget}^2 \\
     &\leq 3 \| \hagg - \hQ  \|_{\Sigmatarget}^2 + 3 \| \hemppseudobestP - \hQ \|_{\Sigmatarget}^2 \leq 3 \erroragg + 6 \empbestemprisk.
\end{align*}
It follows that
\begin{align*}
    \excessQ(\hempIP)  = \| \hempIP - \hQ \|_{\Sigmatarget}^2 \leq 4 \| \hagg - \hempIP \|_{\Sigmatargetemp}^2 + 2  \| \hagg - \hQ  \|_{\Sigmatarget}^2 \leq 24 \empbestemprisk + 5 \erroragg. 
\end{align*}
Define the risk gap $\gapempIpseudobest \coloneqq \riskQ(\hempIP) - \riskQ(\hemppseudobestP)$. It thus follows for the output $\hadapt$ of \Cref{alg:iterative-step-agg} that
\begin{align*}
    \excessQ(\hadapt) \leq \empbestemprisk + \sup_{\gapempIpseudobest \leq 23 \empbestemprisk + 5  \erroragg} \gapempIpseudobest. 
\end{align*}

\textbf{Case 2: The set $\hypoaccone$ is non-empty.} Now, consider the case that the updated set $\hypoaccone$ is non-empty. For any accepted model $\hempIP$ it follows that 
\begin{align*}
    \| \hempIP - \hQ \|_{\Sigmatarget}^2 &\leq 2 \| \hempIP - \hagg \|_{\Sigmatarget}^2 + 2 \| \hagg - \hQ \|_{\Sigmatarget}^2  \\
    &\leq 4 \frac{C_2}{\Cagg} \erroragg + 2 \empbestemprisk + 2 \erroragg \leq 2 \empbestemprisk + ( 4 \frac{C_2}{\Cagg} + 2) \erroragg.
\end{align*}
In particular, for any accepted model $\hempIP \in \hypoaccone$ (and, by convexity, also its aggregate $\hagg$), it holds that 
\begin{align}\label{eq:empirical-guarantee-1}
    \excessQ(\hempIP) \leq \empbestemprisk + \sup_{\gapempIpseudobest \leq \empbestemprisk + C \erroragg} \gapempIpseudobest,
\end{align}
where $C = ( 4 \frac{C_2}{\Cagg} + 2)$. In particular, if there exists a constant $c_0$ such that $\empbestemprisk \leq c_0 \erroragg$, we conclude the guarantee for both Case 1 and Case 2:
\begin{align}\label{eq:empirical-guarantee-2}
    \excessQ(\hadapt) \leq \empbestemprisk + \sup_{\gapempIpseudobest \lesssim \erroragg} \gapempIpseudobest.
\end{align}
We remark that the guarantees \eqref{eq:empirical-guarantee-1} and \eqref{eq:empirical-guarantee-2} hold for any element of $\hypoaccone$, and thus, continue to hold after an arbitrary number of iterations, since a member of $\hypoaccone$ is returned in any case. 
\paragraph{From empirical to population margins.} The only remaining step now is to relate our guarantees based on empirical source models to their population counterparts to quantify the effect of source data. Recall the population-best source model
\begin{equation*}
    \hbestP \in \arg\min_{\hIP \in \hypoinit} \excessQ(\hIP)
\end{equation*}
and its finite-sample counterpart $\hempbestP$. By definition of the empirical best model $\hemppseudobestP$, it holds that $\empbestemprisk = \excessQ(\hemppseudobestP) \leq \excessQ(\hempbestP)$. Consider a model $\hempIP$ for which it holds that $\excessQ(\hempIP) \leq c_1 \excessQ(\hemppseudobestP) + c_2 \erroragg$ with $c_1, c_2 \geq 1$. Such a model is returned in Case 1 and Case 2. It follows that
\begin{equation*}
    \excessQ(\hempIP) \leq c_1 \excessQ(\hempbestP) + c_2 \erroragg,
\end{equation*}
and thus 
\begin{equation*}
    \sup_{\excessQ(\hempIP) \leq c_1 \excessQ(\hemppseudobestP) + c_2 \erroragg} \excessQ(\hempIP) \leq \sup_{\excessQ(\hempIP) \leq c_1 \excessQ(\hempbestP) + c_2 \erroragg} \excessQ(\hempIP).
\end{equation*}
Define the following risk margins:
\begin{align*}
    \gapIbest \coloneqq \riskQ(\hIP) - \riskQ(\hbestP); \quad \gapempIbest \coloneqq \riskQ(\hempIP) - \riskQ(\hempbestP); \\
    \gapIpseudobest \coloneqq \riskQ(\hIP) - \riskQ(\hpseudobestP); \quad \gapempIpseudobest \coloneqq \riskQ(\hempIP) - \riskQ(\hemppseudobestP). 
\end{align*}
We can decompose
\begin{align*}
    \excessQ(\hempIP) = \excessQ(\hempbestP) + \gapempIbest,
\end{align*}
and thus rewrite the upper bound as 
\begin{align*}
    \excessQ(\hempIP) \leq \excessQ(\hempbestP) + \sup_{\gapempIbest \leq (c_1-1) \excessQ(\hempbestP) + c_2 \erroragg} \gapempIbest.
\end{align*}
Finally, we use \Cref{lem:target-source-risk-bound} and \Cref{lem:from-gap-to-gap} (taking a union bound over all candidate models) to obtain
\begin{align*}
    \excessQ(\hempbestP) &\leq \excessQ(\hbestP) + \ErrnpIstar; \\
    \gapempIbest &\leq \gapIbest + \ErrnpIstar + \ErrnpI.
\end{align*}
Denoting $\bestrisk \coloneqq \excessQ(\hbestP)$, we can express upper bound in terms of population quantities and source-finite-sample error terms as follows:
\begin{equation*}
    \excessQ(\hadapt) \leq \bestrisk + \ErrnpIstar + \sup_{\gapIbest \lesssim \bestrisk + \ErrnpIstar + \ErrnpI + \erroragg} \gapIbest + \ErrnpIstar + \ErrnpI.
\end{equation*}
Finally, we take the supremum over $\ErrnpI$ which results in $\maxmarginerror$, define $\gapmax$ as in \Cref{eq:deltamax}, take $c_0$ to be $\max\{c_1,c_2\}$ over Case 1 and Case 2, and take a union bound over the conditioned events, where for \Cref{lem:convergence-linear-regression} for $\hIP$ we use confidence $\alpha/|\hypoinit|$. This yields the claim.
\clearpage
\subsection{Lower bound proofs}\label{apx:proof-of-lower-bound}
\subsubsection{Proof of Theorem \ref{thm:lower-bound1}.}\label{apx:proof-of-lower-bound1}


\begin{lemma}[Basic lemma]\label{lem:basic-lemma1}
Let $D_1,D_2$ be distributions and $Z\sim D_\omega^{\otimes n}$ with $\omega\sim\Unif(\{1,2\})$.
For any test $\hat\omega(Z)\in\{1,2\}$,
\[
\sup_{\omega\in\{1,2\}} \Pr_{D_\omega^{\otimes n}}(\hat\omega(Z)\neq \omega)
\;\ge\;
\frac12\Big(1-\TV(D_1^{\otimes n},D_2^{\otimes n})\Big).
\]
Moreover, by Pinsker and additivity of KL,
\[
\TV(D_1^{\otimes n},D_2^{\otimes n})
\le \sqrt{\frac12 \KL(D_1^{\otimes n}\|D_2^{\otimes n})}
= \sqrt{\frac{n}{2}\KL(D_1\|D_2)}.
\]
Hence,
\[
\sup_{\omega} \Pr(\hat\omega(Z)\neq \omega)
\;\ge\;
\frac12\Big(1-\sqrt{\frac{n}{2}\KL(D_1\|D_2)}\Big).
\]
\end{lemma}

\begin{proof}
Consider the product distributions $D_1^{\otimes n}$ and $D_2^{\otimes n}$. Any test
$\hat\omega$ can be identified with a measurable set $A\subseteq \mathcal Z^n$ via
$A:=\{z:\hat\omega(z)=1\}$.
Then, under the uniform prior on $\omega \in \{1,2\}$,
\begin{align*}
\Pr(\hat\omega(Z)\neq \omega)
&= \frac12 \Pr_{Z\sim D_1^{\otimes n}}(\hat\omega(Z)=2)+\frac12 \Pr_{Z\sim D_2^{\otimes n}}(\hat\omega(Z)=1)\\
&= \frac12 D_1^{\otimes n}(A^c) + \frac12 D_2^{\otimes n}(A)
= \frac12\big(1 - (D_1^{\otimes n}(A)-D_2^{\otimes n}(A))\big).
\end{align*}
Hence
\begin{equation}\label{eq:avg-error-tv}
\Pr(\hat\omega(Z)\neq \omega) \;\ge\; \frac12\big(1-\sup_{B}(D_1^{\otimes n}(B)-D_2^{\otimes n}(B))\big).
\end{equation}
By the definition of total variation distance,
\[
\TV(D_1^{\otimes n},D_2^{\otimes n}) := \sup_{B} |D_1^{\otimes n}(B)-D_2^{\otimes n}(B)| = \sup_{B} (D_1^{\otimes n}(B)-D_2^{\otimes n}(B)).
\]
Plugging this into \eqref{eq:avg-error-tv} yields
\[
\Pr(\hat\omega(Z)\neq \omega) \geq \frac12\big(1-\TV(D_1^{\otimes n},D_2^{\otimes n})\big).
\]
Finally, since $\sup_{\omega} \Pr_{D_\omega^{\otimes n}}(\hat\omega(Z)\neq \omega)$ is at least the average error
$\Pr(\hat\omega(Z)\neq \omega)$ under the uniform prior, we obtain
\[
\sup_{\omega\in\{1,2\}} \Pr_{D_\omega^{\otimes n}}(\hat\omega(Z)\neq \omega)
\;\ge\;
\frac12\Big(1-\TV(D_1^{\otimes n},D_2^{\otimes n})\Big).
\]

For the second part, by Pinsker's inequality, 
\[
\TV(P,Q) \le \sqrt{\frac12 \KL(P\|Q)}.
\]
By the product property of the KL divergence 
\[
\KL(D_1^{\otimes n}\|D_2^{\otimes n}) = n\,\KL(D_1\|D_2).
\]
In total,
\[
\TV(D_1^{\otimes n},D_2^{\otimes n})
\le \sqrt{\frac12 \KL(D_1^{\otimes n}\|D_2^{\otimes n})}
= \sqrt{\frac{n}{2}\KL(D_1\|D_2)}.
\]
Substituting into the previous inequality concludes the proof.
\end{proof}

\textbf{Construction of the source distribution. } Consider the feature sets $I_1 = \{1\}$, $I_2 = \{2\}$. Consider the source distribution with the following marginal: $(X_1, X_2) \in \{e_1, e_2\}$ with probability $1/2$ each, and $X_3 \sim \Unif(\{-1,+1\})$ with $(X_1,X_2) \perp X_3$.  
Set $Y = \sqrt{\Delta}$.
Then the source candidate models are $\honeP(X) = \sqrt{\Delta} X_1$ and $\htwoP(X) = \sqrt{\Delta} X_2$. 

\textbf{Construction of the target distribution.} We let $Q_X = P_X$. For $\omega \in \{1,2\}$, we define 
\begin{equation*}
    Y = \sqrt{\Delta} \indicator \{X_{\omega}=1 \} + \sqrt{\bestrisk} X_3 + \xi, \quad \xi \sim \cN(0,1). 
\end{equation*}
We compute the Bayes regression function $m_{\omega}(X) = \sqrt{\Delta} \indicator \{X = e_\omega \} + \sqrt{\bestrisk} X_3$. We compute the excess risks of the candidates under the target distribution: under $\omega = 1$, we have
\begin{equation*}
    \excessQ(\honeP) = \bestrisk; \quad \excessQ(\htwoP) = \Delta + \bestrisk.
\end{equation*}
Under $\omega = 2$, 
\begin{equation*}
    \excessQ(\honeP) = \bestrisk + \Delta; \quad \excessQ(\htwoP) = \bestrisk.
\end{equation*}
\textbf{Excess risk lower bound. } Let $\mathcal{A}$ be any mapping $\dataP \times \dataQ \to \{\honeP, \htwoP\}$. Let $\hat{h} = 1$ if $\mathcal{A}$ outputs $\honeP$ and $\hat{h} = 2$ otherwise. W.l.o.g., we let $\mathcal{A} = \mathcal{A}(\dataQ)$ since $\dataP$ carries no information about $\omega$. We have 
\begin{align*}
\sup_{\omega\in\{1,2\}} 
\Pr_{\dataQ\sim D_\omega^{\otimes \ntarget}}\!\Big(\excess_{Q^{(\omega)}}(\mathcal{A}(\dataQ)) \ge \bestrisk+\Delta\Big)
&\ge 
\E_{\omega\sim \Unif(\{1,2\})}
\Pr_{\dataQ\sim D_\omega^{\otimes \ntarget}}\!\Big(\excess_{Q^{(\omega)}}(\mathcal{A}(\dataQ)) \ge \bestrisk+\Delta\Big)\\
&=
\Pr_{\omega\sim\Unif(\{1,2\}),\,\dataQ\sim D_\omega^{\otimes \ntarget}}
\Big(\excess_{Q^{(\omega)}}(\mathcal{A}(\dataQ)) \ge \bestrisk+\Delta\Big)\\
&=
\Pr_{\omega\sim\Unif(\{1,2\}),\,\dataQ\sim D_\omega^{\otimes \ntarget}}(\hat h(\dataQ)\neq \omega).
\end{align*}
By \Cref{lem:basic-lemma1}, we have 
\[
\inf_{\hat h}\;
\Pr_{\omega\sim\Unif(\{1,2\}),\,\dataQ\sim D_\omega^{\otimes \ntarget}}(\hat h(\dataQ)\neq \omega)
\ge \frac12\Big(1-\sqrt{\frac{\ntarget}{2}\KL(Q_1\|Q_2)}\Big).
\]
It remains to bound $\KL(Q_1\|Q_2)$. Since $Q_1$ and $Q_2$ have the same marginal $Q_X$ and 
$Y\mid X$ is Gaussian with unit variance under both hypotheses, we have
\[
\KL(Q_1\|Q_2)=\E_{X\sim Q_X}\!\left[\KL\big(Q^{(1)}_{Y\mid X}\,\|\,Q^{(2)}_{Y\mid X}\big)\right]
=\frac{\Delta}{2},
\]
where we used $\KL(\mathcal N(\mu,1)\|\mathcal N(\nu,1))=\frac{(\mu-\nu)^2}{2}$.
Therefore,
\[
\inf_{\mathcal A}\sup_{\omega\in\{1,2\}}
\Pr_{Q^{(\omega)}}\!\Big(\excess_Q(\mathcal A)\ge \bestrisk+\Delta\Big)
\ge
\frac12\Big(1-\sqrt{\frac{\ntarget\Delta}{4}}\Big).
\]
In particular, if $\Delta \le \frac{1}{16\ntarget}$, then
\[
\inf_{\mathcal A}\sup_{\omega}
\Pr\!\Big(\excess_Q(\mathcal A)\ge \bestrisk+\Delta\Big)
\ge \frac{7}{16}=:c_1.
\]
This proves the claim with $c_0=\frac{1}{16}$.

\subsubsection{Proof of Theorem \ref{thm:lower-bound2}.}\label{apx:proof-of-lower-bound2}
As in the proof of \Cref{thm:lower-bound1}, we first outline our $(P,Q)$ construction including the two candidate models and then prove the lower bound via reduction to testing. 
\paragraph{Construction of $(P,Q^{(1)},Q^{(2)})$ and the candidates.}
Fix $\Delta\in(0,1/4]$ and set
\[
a := \sqrt{\Delta}\in\Big(0,\frac12\Big].
\]
\textbf{Source distribution and candidates.} 
Let $X\in\R^d$ take values in $\{e_1,e_2\}$ with
\[
P_X(e_1)=P_X(e_2)=\frac12.
\]
We consider $Y \in \{-1,+1\}$. Define the conditional distribution of $Y$ given $X$
\[
P(Y=+1\mid X=x)=\frac{1+a}{2},\qquad P(Y=-1\mid X=x)=\frac{1-a}{2}\qquad \forall x\in\{e_1,e_2\}.
\]
Hence, $\E_P[Y\mid X=x]=a$ for $x=e_1,e_2$.

Let $I_1=\{1\}$ and $I_2=\{2\}$. Consider the squared-loss population ERM predictors restricted to these feature sets.
The restricted predictors take the form
\[
h_{I_1,P}(x)=a\,X_1,\qquad h_{I_2,P} = a X_2.
\]

\textbf{Target distributions.} Define two target regression functions $m_1,m_2:\{e_1,e_2\}\to \R$ by
\[
m_1(e_1)=a,\quad m_1(e_2)=0,
\qquad
m_2(e_1)=0,\quad m_2(e_2)=a.
\]
Given $\omega\in\{1,2\}$, define $Q^{(\omega)}$ on $(X,Y)\in\{e_1,e_2\}\times\{-1,+1\}$ by the  marginal $Q_X = P_X$ and setting
\[
Q_{\omega}(Y=+1\mid X=x)=\frac{1+m_\omega(x)}{2},\qquad
Q_{\omega}(Y=-1\mid X=x)=\frac{1-m_\omega(x)}{2}.
\]
Then $\E_\omega[Y\mid X=x]=m_\omega(x)$. 

\paragraph{Target excess risks of the candidates.}
For squared loss and binary $Y\in\{-1,+1\}$, the Bayes regressor under $Q^{(\omega)}$ is the conditional mean $m_\omega(X)$, and for any predictor $f$ we have
\[
\excess_{Q^{(\omega)}}(f)
:=R_{Q^{(\omega)}}(f)-\inf_g R_{Q^{(\omega)}}(g)
=\E_{Q^{(\omega)}}\big[(f(X)-m_\omega(X))^2\big].
\]
We now compute these excess risks for $h_{I_1,P}$ and $h_{I_2,P}$.

\emph{Case $\omega=1$.} Recall $m_1(e_1)=a$ and $m_1(e_2)=0$.
Also, $h_{I_1,P}(e_1)=a$ and $h_{I_1,P}(e_2)=0$, hence $h_{I_1,P}(X)=m_1(X)$ almost surely and
\[
\excess_{Q^{(1)}}(h_{I_1,P})
=\E\big[(h_{I_1,P}(X)-m_1(X))^2\big]=0.
\]
On the other hand, $h_{I_2,P}(e_1)=0$ and $h_{I_2,P}(e_2)=a$, so
\[
(h_{I_2,P}(e_1)-m_1(e_1))^2=(0-a)^2=a^2,\qquad
(h_{I_2,P}(e_2)-m_1(e_2))^2=(a-0)^2=a^2.
\]
Since $Q_X(e_1)=Q_X(e_2)=1/2$, we obtain
\[
\excess_{Q^{(1)}}(h_{I_2,P})
=\frac12 a^2+\frac12 a^2
=a^2
=\Delta.
\]

\emph{Case $\omega=2$.} Under the swapped regression functions $m_2(e_1)=0$, $m_2(e_2)=a$,
\[
\excess_{Q^{(2)}}(h_{I_2,P})=0,\qquad
\excess_{Q^{(2)}}(h_{I_1,P})=\Delta.
\]
\paragraph{Separation between $Q^{(1)}$ and $Q^{(2)}$.}
We have
\[
\|m_1-m_2\|_{L^2(Q_X)}^2
=\E_{X\sim Q_X}\big[(m_1(X)-m_2(X))^2\big]
=\frac12 a^2+\frac12 a^2
=a^2
=\Delta.
\]

\paragraph{From any algorithm to a decision rule.}
Let $\mathcal A$ be any learning algorithm and let $\hat f=\mathcal A(\dataP,\dataQ)$ denote its output predictor
(possibly improper). Define the decision rule $\hat h(\dataQ)\in\{1,2,-1\}$ by
\[
\hat h(\dataQ)=
\begin{cases}
1,& \hat f(e_1) > a/2\ \text{and}\ \hat f(e_2) < a/2,\\
2,& \hat f(e_1) < a/2\ \text{and}\ \hat f(e_2) > a/2,\\
-1,& \text{otherwise}.
\end{cases}
\]
Let $A:=\{\dataQ:\hat h(\dataQ)\neq -1\}$.
In the following lemma, we connect the output of our decision rule to the excess risk of the algorithm:
\begin{lemma}\label{lem:ambig-excess}
For each $\omega\in\{1,2\}$, on the event $\{\hat h(\dataQ)=-1\}$ we have
\[
\excess_{Q^{(\omega)}}(\hat f)\ \ge\ \frac{\Delta}{8}.
\]
\end{lemma}
\begin{proof}
We prove the claim for $\omega=1$; the case $\omega=2$ is symmetric.
Write $u_1:=\hat f(e_1)$ and $u_2:=\hat f(e_2)$. Under $\omega=1$, we have
\[
\excess_{Q^{(1)}}(\hat f)=\frac12(u_1-a)^2+\frac12(u_2-0)^2.
\]
On the event $\{\hat h=-1\}$, it is not the case that $(u_1>a/2 \ \&\ u_2<a/2)$, so either $u_1\le a/2$ or $u_2\ge a/2$.
If $u_1\le a/2$ then $|u_1-a|\ge a/2$ and thus $(u_1-a)^2\ge a^2/4$.
If $u_2\ge a/2$ then $|u_2-0|\ge a/2$ and thus $u_2^2\ge a^2/4$.
Therefore $(u_1-a)^2+u_2^2\ge a^2/4$, and consequently
\[
\excess_{Q^{(1)}}(\hat f)\ge \frac12\cdot \frac{a^2}{4}=\frac{a^2}{8}=\frac{\Delta}{8}.
\]
\end{proof}
We observe that for any $\omega$,
\begin{equation}\label{eq:err-implies-excess}
\{\hat h(\dataQ)\neq \omega\}\subseteq \left\{\excess_{Q^{(\omega)}}(\hat f)\ge \frac{\Delta}{8}\right\}.
\end{equation}
Indeed, if $\hat h=-1$ this follows from Lemma~\ref{lem:ambig-excess}. Otherwise $\hat h\in\{1,2\}$ and
$\hat h\neq \omega$ implies that $\hat f$ places the ``high'' value (above $a/2$) on the wrong atom relative to $m_\omega$,
which forces at least one of $|\hat f(e_1)-m_\omega(e_1)|$ or $|\hat f(e_2)-m_\omega(e_2)|$ to be at least $a/2$,
and yields $\excess_{Q^{(\omega)}}(\hat f)\ge \Delta/8$ as in the proof of Lemma~\ref{lem:ambig-excess}.

Consequently, for all $\omega$,
\begin{equation}\label{eq:excess-lower-by-testerr}
\Pr_{Q^{(\omega)}}\!\left(\excess_{Q^{(\omega)}}(\hat f)\ge \frac{\Delta}{8}\right)
\ \ge\ 
\Pr_{Q^{(\omega)}}\big(\hat h(\dataQ)\neq \omega\big).
\end{equation}

\paragraph{Lower bounding the testing error via KL.}
Let $D_\omega$ denote the single-sample distribution of $(X,Y)$ under $Q^{(\omega)}$.
By \Cref{lem:basic-lemma1} in \Cref{apx:proof-of-lower-bound1},
\begin{equation}\label{eq:test-lb}
\sup_{\omega\in\{1,2\}}\Pr_{D_\omega^{\otimes n}}(\hat h\neq \omega)
\ \ge\ 
\frac12\Big(1-\sqrt{\frac{n}{2}\KL(D_1\|D_2)}\Big).
\end{equation}
It remains to upper bound $\KL(D_1\|D_2)$.

Under $\omega=1$, we compute
\[
p:=\Pr_{\omega=1}(Y=+1\mid X=e_1)=\frac{1+a}{2},
\]
while 
\[
\Pr_{\omega=1}(Y=+1\mid X=e_2)=\frac{1+0}{2}=\frac12.
\]
Under $\omega=2$, the conditionals are swapped.
Therefore,
\[
\KL(D_1\|D_2)
=\frac12\KL(\Ber(p)\|\Ber(1/2))+\frac12\KL(\Ber(1/2)\|\Ber(p)).
\]
We use the standard bound $\KL(\Ber(r)\|\Ber(s))\le \frac{(r-s)^2}{s(1-s)}$ for $r,s\in(0,1)$.
With $r=p=\frac{1+a}{2}$ and $s=\frac12$, we have 
\[
\KL(\Ber(p)\|\Ber(1/2))\le \frac{(a/2)^2}{1/4}=a^2,
\qquad
\KL(\Ber(1/2)\|\Ber(p))\le \frac{(a/2)^2}{p(1-p)}.
\]
Moreover, $p(1-p)=\frac{1-a^2}{4}\ge \frac{3}{16}$ for $a\le 1/2$, so
\[
\KL(\Ber(1/2)\|\Ber(p))\le \frac{(a/2)^2}{3/16}=\frac{4}{3}a^2.
\]
Combining the bounds gives
\begin{equation}\label{eq:kl-bernoulli}
\KL(D_1\|D_2)\le \frac12\Big(a^2+\frac{4}{3}a^2\Big)=\frac{7}{6}a^2=\frac{7}{6}\Delta.
\end{equation}
Plugging \eqref{eq:kl-bernoulli} into \eqref{eq:test-lb} yields
\[
\sup_{\omega}\Pr(\hat h\neq \omega)
\ge 
\frac12\Big(1-\sqrt{\frac{n}{2}\cdot \frac{7}{6}\Delta}\Big)
=\frac12\Big(1-\sqrt{\frac{7n\Delta}{12}}\Big).
\]

Combining with \eqref{eq:excess-lower-by-testerr}, we conclude that for any learning algorithm
$\mathcal A$,
\[
\sup_{\omega\in\{1,2\}}
\Pr_{Q^{(\omega)}}\!\left(\excess_{Q^{(\omega)}}(\mathcal A)\ge \frac{\Delta}{8}\right)
\ \ge\
\frac12\Big(1-\sqrt{\frac{7\ntarget\Delta}{12}}\Big).
\]
In particular, if $\Delta\le \frac{3}{28\ntarget}$, then $\sqrt{\frac{7\ntarget\Delta}{12}}\le \frac14$ and hence
\[
\inf_{\mathcal A}\sup_{\Sigma(0,\Delta)}
Q\!\left(\excess_Q(\mathcal A)\ge \frac{\Delta}{8}\right)
\ \ge\ \frac{3}{8}.
\]
This proves the claim with $c_0=\frac{3}{28}$ and $c_1=\frac{3}{8}$.
\clearpage

\subsection{Proof of Corollary \ref{cor:linear-scm} }\label{apx:proof-of-linear-scm}
\begin{proof}
    Since the graph is acyclic, $(\Id - B)$ is invertible. We define $W = (\Id - B)^{-1}$. We define the source covariance matrix $\Omega = W \Sigma W^\top$. It follows from the invertibility of $W$ and $\Sigma$ that $\Omega_{I,I}$ is invertible for every $I$. We thus have for the population source linear regression coefficients:
    \begin{equation*}
        \hIP = \Omega_{I,I}^{-1} \Omega_{I,Y}.
    \end{equation*}
We consider the residual $e_I = Y - \hIP^{\top} X_I$. Defining $c_I \in \R^{d+1}$ s.t. $(c_I)_Y = 1$, $(c_I)_k = - (\hIP)_k$ for $k \in I$ and $0$ otherwise, we obtain $e_I = c_I^\top V$. Finally, we define
\begin{equation*}
    \alpha_I = W^\top c_I.
\end{equation*}
Importantly, it follows for the target residuals:
\begin{equation*}
    e_{I,Q} = c_I^\top V_Q = (W^\top c_I^\top)(U + t \addshift_Q) = \alpha_I^\top U + t \alpha_I^\top \addshift_Q. 
\end{equation*}
We compute $\E_Q[e_I] = t \alpha_I^\top \mu$ and $\Var_Q(e_I) = \sigma_I^2 + t^2 \alpha_I^\top H \alpha_I$, where $\sigma_I^2 = \E_Q[(\alpha_I^\top U)^2] = \alpha_I^\top \Sigma \alpha_I$. It follows for the source and target risk of a model $\hIP$:
\begin{equation*}
    \riskP(\hIP) = \sigma_I^2 \quad \quad \riskQ(\hIP) = \sigma_I^2 + t^2((\alpha_I^\top \mu)^2 + \alpha_I^\top H \alpha_I).
\end{equation*}

\textbf{$d$-separation implies invariant risk.}  Let $I \in \set$. First, consider the case that $I$  $d$-separates $Y$ from $\squarenode{P}{Q}$. Since the SCM is Markovian, it follows that $Y \perp \squarenode{P}{Q} \mid X_I$. This again implies invariance of the regression function across the domains: $\E[Y \mid X_I, \squarenode{P}{Q} = 0] = \E[Y \mid X_I, \squarenode{P}{Q} = 1]$ and in particular $\hIP = \hIQ$. We now want to show that $\alpha_{I,\shiftset{P}{Q}} = 0$, i.e. $\alpha_I$ is $0$ on $\shiftset{P}{Q}$. Since the above invariance holds for any additive shift $\addshift_Q$ with support in $\shiftset{P}{Q}$, it holds in particularly for mean shifts of form $t e_j$ for any $t$ for $j \in \shiftset{P}{Q}$. Invariance of residuals then implies that $\alpha_{I,j} =0$, for any $j \in \shiftset{P}{Q}$. 

\textbf{Non-$d$-separation almost always implies quadratically growing risk.}
We now consider the case that $I$ does not $d$-separate $Y$ from $\squarenode{P}{Q}$. Define a map $g_I(B,\Sigma) =  \| \alpha_{I,\shiftset{P}{Q}}\|^2$ which maps the SCM parameters to the squared norm of the coefficients of $\alpha_I$ on $\shiftset{P}{Q}$. This map is real-analytic since it is rational. In particular, it is either zero or has a countable number of zeros. For any $I$ which is non-$d$-separating, there exists an active path from $\squarenode{P}{Q}$ given $X_I$, in particular, one can choose $B$ and $\sigma$ such that the active connection results in dependence of the residual on at least one shifted coordinate, i.e. $\alpha_{I, \shiftset{P}{Q}} \neq 0$. Thus, there exists at least one point $(B_0,\Sigma_0)$ s.t. $g_I(B_0,\Sigma_0) \neq 0$. Therefore, $g_I$ is not zero and the set of its zeros has Lebesgue measure zero. This implies $\alpha_{I,\shiftset{P}{Q}} \neq 0$ for almost all $(B,\Sigma)$ if $I$ is not $d$-separating. Finally, the set of vectors $\mu$ which satisfies $\mu^\top \alpha_I = 0$ for $\alpha_{I,\shiftset{P}{Q}} \neq 0$ forms a codimension-one  hyperplane in $\R^{d+1}$ and is thus a set of measure zero. Finally, a union of such hyperplanes over the finite set of non-$d$-separating $I$ is still a set of measure zero. Taking everything together, we obtain that for all combinations $(B, \Sigma, \mu)$ apart from a set of Lebesgue measure zero, we have $\alpha_I^\top \mu \neq 0$ for all non-$d$-separating sets $I$ and thus for all such sets,
\begin{equation*}
    \riskQ(\hIP) = \sigma_I^2 + t^2((\alpha_I^\top \mu)^2 + \alpha_I^\top H \alpha_I) = \sigma_I^2 + \Theta(t^2). 
\end{equation*}
Finally, we observe that for the risk margin between the $d$-separating and the non-$d$-separating $I$, it holds 
\begin{equation*}
    \inf_{I \in \set \setminus \setsep} \riskQ(\hIP) - \sup_{I \in \setsep} \riskQ(\hIP) = \Theta(t^2)
\end{equation*}
and invoke \Cref{cor:stable-unstable-models} to make our last claim. 
\end{proof}
\end{document}